\documentclass[letterpaper,11pt]{article}
\usepackage{fullpage}
\usepackage{amsmath,amsthm}

\usepackage{times}
\usepackage{soul}

\usepackage[colorlinks=true, citecolor=blue]{hyperref}
\usepackage{url}
\usepackage[utf8]{inputenc}
\usepackage[small]{caption}
\usepackage{graphicx}
\usepackage{amsmath}
\usepackage{amsthm}
\usepackage{booktabs}
\usepackage[maxnames=10,backref=true,style=alphabetic]{biblatex}
\newcommand{\citet}[1]{\citeauthor*{#1}~\cite{#1}}
\usepackage{color}
\usepackage{amsfonts,amssymb,bbm, bm}
\usepackage{empheq}
\usepackage{xcolor}
\usepackage{tcolorbox}
\usepackage{tikz}
\usetikzlibrary{shadows,arrows.meta,positioning,backgrounds,fit,chains,scopes}
\usepackage{caption}
\usepackage{subcaption}

\usepackage{etoolbox}

\usepackage{savesym,multirow,url,xspace,graphicx,hyperref,enumitem,color}
\usepackage[toc,page,header]{appendix}
\usepackage{minitoc}
\usepackage{dsfont}
\usepackage{multicol}
\usepackage{dsfont}

\usepackage{tikz}
\usetikzlibrary{automata, positioning, arrows}
\tikzset{
	->, 
	every state/.style={thick, fill=gray!10}, 
	initial text=$ $, 
}
\usetikzlibrary{arrows.meta}
\usepackage{pgfplots}
\pgfplotsset{width=10cm,compat=1.9}
\usepackage{diagbox}
\usepackage{caption}

\makeatletter
\newif\if@restonecol
\makeatother
\usepackage[lined, boxed, ruled, linesnumbered]{algorithm2e}
\usepackage{amsmath,amssymb,xfrac,amsthm}
\usepackage{mathtools}
\usepackage{wrapfig}
\setitemize{noitemsep,topsep=1pt,parsep=1pt,partopsep=1pt, leftmargin=12pt}
\makeatletter

\newcommand{\Rmnum}[1]{\expandafter\@slowromancap\romannumeral #1@}
\makeatother
\usepackage{caption}

\usepackage{enumitem}
\usepackage{wasysym}

\DeclareMathOperator*{\argmin}{arg\,min}
\DeclareMathOperator*{\argmax}{arg\,max}

\usepackage{xcolor}

\SetCommentSty{mycommfont}
\SetKwInput{KwInput}{Input}
\SetKwInput{KwOutput}{Output}  
\SetAlFnt{\small}
\SetAlCapFnt{\small}
\SetAlCapNameFnt{\small}
\SetAlCapHSkip{0pt}
\IncMargin{-\parindent}
\usepackage{bm}
\usepackage{bbm}

\let\mathbbm\mathds
\newcommand{\norm}[1]{\left\lVert#1\right\rVert}

\newcommand{\abs}[1]{\left|#1\right|}

\newcommand{\Pro}{\mathbb{P}}

\newcommand{\one}{\mathbbm{1}}
\newcommand{\set}[1]{\left\{ #1 \right\}}
\renewcommand{\bar}{\overline}
\renewcommand{\tilde}{\widetilde}
\newcommand{\E}{\mathbb{E}}
\newcommand{\calL}{\mathcal{L}}

\newcommand{\calS}{\mathcal{S}}
\newcommand{\calR}{\mathcal{R}}

\newcommand{\calD}{\mathcal{D}}
\newcommand{\calH}{\mathcal{H}}
\newcommand{\calC}{\mathcal{C}}

\newcommand{\calT}{\mathcal{T}}

\newcommand{\calM}{\mathcal{M}}

\newcommand{\calW}{\mathcal{W}}
\newcommand{\calA}{\mathcal{A}}
\newcommand{\calN}{\mathcal{N}}
\newcommand{\expp}{\mathbb{E}}
\newcommand{\mx}{\mathrm{x}}
\newcommand{\R}{\mathbb{R}}

\newcommand{\calF}{\mathcal{F}}

\newcommand{\reg}{\textrm{Reg}}

\newcommand{\eps}{\varepsilon}

\newcommand{\Identity}{\mathrm{Id}}

\usepackage{xargs}                      
\PassOptionsToPackage{pdftex,dvipsnames}{xcolor}  
\newcommandx{\task}[2][1=]{\todo[linecolor=red,backgroundcolor=red!25,bordercolor=red,#1]{#2}}
\newcommandx{\change}[2][1=]{\todo[linecolor=blue,backgroundcolor=blue!25,bordercolor=blue,#1]{#2}}
\newcommandx{\info}[2][1=]{\todo[linecolor=green,backgroundcolor=green!25,bordercolor=green,#1]{#2}}

\newcommand{\htheta}{\widehat{\theta}}
\newcommand{\tref}{\textrm{ref}}
\newcommand{\diff}{\textrm{diff}}
\newcommand{\avg}{\textrm{avg}}

\usepackage[toc,page,header]{appendix}
\usepackage{minitoc}

\newtheorem{theorem}{Theorem}[section]
\newtheorem{proposition}[theorem]{Proposition}
\newtheorem{lemma}[theorem]{Lemma}

\theoremstyle{definition}
\newtheorem{definition}[theorem]{Definition}
\newtheorem{assumption}[theorem]{Assumption}
\theoremstyle{remark}

\newtheorem*{statement*}{Statement}
\usepackage{cleveref}

\title{Corruption Robust Offline Reinforcement Learning \\ with Human Feedback}

\begin{document}

\author{Debmalya Mandal\\Dept. of Computer Science\\University of Warwick, UK\\ \texttt{Debmalya.Mandal@warwick.ac.uk}
\and Andi Nika\\Max-Planck Institute for \\Software Systems, Germany\\ \texttt{andinika@mpi-sws.org}
\and Parameswaran Kamalaruban\\Independent Researcher\\ London, UK\\ \texttt{pkamalaruban@gmail.com}
\and Adish Singla\\Max-Planck Institute for \\Software Systems, Germany\\ \texttt{adishs@mpi-sws.org}
\and Goran Radanovi\'c\\Max-Planck Institute for \\Software Systems, Germany\\ \texttt{gradanovic@mpi-sws.org}
}
\maketitle
\begin{abstract}
     We study data corruption robustness for reinforcement learning with human feedback (RLHF) in an offline setting. Given an offline dataset of pairs of trajectories along with feedback about human preferences, an $\varepsilon$-fraction of the pairs is corrupted (e.g., feedback flipped or trajectory features manipulated), capturing an adversarial attack or noisy human preferences. We aim to design algorithms that identify a near-optimal policy from the corrupted data, with provable guarantees. Existing theoretical works have separately studied the settings of corruption robust RL (learning from scalar rewards directly under corruption) and offline RLHF (learning from human feedback without corruption); however, they are inapplicable to our problem of dealing with corrupted data in offline RLHF setting. To this end, we design novel corruption robust offline RLHF methods under various assumptions on the coverage of the data-generating distributions. At a high level, our methodology robustifies an offline RLHF framework by first learning a reward model along with confidence sets and then learning a pessimistic optimal policy over the confidence set. Our key insight is that learning optimal policy can be done by leveraging an offline corruption-robust RL oracle in different ways (e.g., zero-order oracle or first-order oracle), depending on the data coverage assumptions. To our knowledge, ours is the first work that provides provable corruption robust offline RLHF methods.
\end{abstract}

\doparttoc
\faketableofcontents

\section{Introduction}\label{sec:intro}

Reinforcement Learning from Human Feedback (RLHF) has emerged as a powerful paradigm for addressing complex tasks across diverse domains, ranging from large language models (LLMs) to robotics and game-playing~\cite{christiano2017deep,ziegler2019fine,stiennon2020learning,ouyang2022training,bai2022training,shin2023benchmarks}. At the core of RLHF is its unique ability to model reward functions solely from preference data, making it particularly well-suited for scenarios where explicit reward signals are challenging to define. Following reward model estimation, traditional RLHF approaches employ online reinforcement learning algorithms for subsequent policy optimization. However, the integration of offline RL within the RLHF pipeline holds promise for alleviating limitations inherent to online RL, notably in terms of sample efficiency and safety concerns~\cite{levine2020offline,kidambi2020morel}. By incorporating offline RL algorithms, RLHF becomes more adaptable to scenarios where online data collection proves prohibitive, facilitating the reuse of valuable pre-existing datasets~\cite{shin2023benchmarks}.

The real-world deployment of RLHF faces substantial challenges rooted in the reliability of the preference data, which is integral to its effectiveness. These challenges primarily arise from two sources: adversarial corruption and inherent noise~\cite{casper2023open,xue2023reinforcement,chhan2024crowd}. Adversarial entities, acting with malicious intent, may deliberately manipulate feedback labels or trajectory features, introducing potential biases in the reward model. Simultaneously, inherent human subjectivity within crowd-sourced preference data can contribute substantial noise, impeding accurate reward estimation. In light of these challenges, a pivotal research question emerges: \emph{Can we devise a robust variant of RLHF that efficiently learns from adversarially corrupted or noisy preference data, exhibiting graceful scalability amidst increasing corruption levels?}


\begin{table*}[!t]
\centering
\begin{tabular}{m{8em}|m{10em}|m{11em}|c}\hline\hline
Type of Coverage & Suboptimality Gap & Robust RL Oracle & \# Oracle Calls\\\hline\hline
Uniform ($\xi$) & $O\left( \frac{H^3 + \sqrt{Hd}}{\xi} \varepsilon^{1 - o(1)}\right)$ & \texttt{R-LSVI}~\parencite{zhang2022corruption}, zero-order access & 1\\ \hline 
Relative Condition Number ($\alpha$) & $\tilde{O}\left(H^2 d \kappa \sqrt{\alpha \varepsilon}\right)$  $+\ \tilde{O}\left(H^{5/4} d^{3/4} (\alpha \varepsilon)^{1/4}\right)$ & \texttt{R-LSVI}~\parencite{zhang2022corruption}, zero-order access & $\tilde{O} \left( \frac{H^{3/2}d^5   }{\varepsilon^3}\right)$\\ \hline 
Generalized  Coverage Ratio ($\nu$) & ${O}\left( \nu \kappa \sqrt{\varepsilon} H^2 d^{3/2}\right)$ & \Cref{alg:offline-primal-dual} (\textcolor{blue}{Our method}), first-order access& $O\left( \frac{1}{\varepsilon \nu}\right)$\\ \hline 
\end{tabular}
\caption{We design \emph{corruption robust RLHF} through reduction to \emph{corruption robust offline RL} problem. Here $H$ is the horizon length, $d$ is the dimension of the features, and $\kappa$ and $\alpha$ are constants . Under uniform coverage and low relative condition number, we use \texttt{R-LSVI} as an oracle, and obtain suboptimality gap of $O(\varepsilon^{1-o(1)})$ and $O(\varepsilon^{1/4})$ respectively, in terms of $\varepsilon$ (fraction of corrupted data). Calls to \texttt{R-LSVI} are zero-order i.e. we only obtain a robust policy and an estimate of the value function. Under bounded generalized coverage ratio, we design a new robust offline RL method (algorithm~\eqref{alg:offline-primal-dual}) that also returns an estimate of the sub-gradient (first order access). Using algorithm~\eqref{alg:offline-primal-dual}, we can improve the dependence on $\varepsilon$ to $O(\sqrt{\varepsilon})$ and also significantly reduce the number of oracle calls.}
\end{table*}

\raggedbottom

In this paper, we initiate the study of \emph{corruption-robust offline reinforcement learning from human feedback}. Although there are several works on corruption robust offline reinforcement learning~\cite{zhang2022corruption, YYGZ23}, and provable preference based reinforcement learning~\cite{zhan2023provable, ZJJ23}, ours is the first work to combine these two threads and provide provable corruption robust offline RLHF methods. We design corruption robust offline RLHF methods through reduction to corruption robust offline RL methods. In particular, we modify the existing RLHF framework through three steps -- (1) Robustly learn a reward model by solving a robust logistic regression problem, (2) Construct a confidence set around the learned model, and (3) learn a pessimistic optimal policy over the confidence set through reduction to offline RL. We instantiate this general framework for datasets with various types of coverage assumptions, and as is often the case in offline RL, different coverage assumptions require different algorithms. 

In particular, we consider the standard \emph{Huber contamination} model where $\varepsilon$-fraction of the data (human feedback, features of the trajectories or both) are corrupted. Moreover, we consider a linear Markov decision process~\cite{jin2020provably} with horizon length $H$, and feature dimension $d$. Then, we prove the following set of results.
\begin{enumerate}[itemsep=0.05cm]
    \item When the offline data has \emph{uniform coverage},  we show that it is possible to learn a policy with sub-optimality gap at most $O\big ( H^3 \sqrt{d} \varepsilon^{1-o(1)}\big )$.\footnote{As $\varepsilon \rightarrow 0$, $\varepsilon^{1-o(1)}$ approaches $\varepsilon$. We actually show a dependence of $\varepsilon \cdot \exp(\sqrt{\log(1/\varepsilon)})$ which is $\varepsilon^{1-o(1)}$.}
    \item When the offline data satisfies the condition of \emph{low relative condition number}, a condition substantially weaker than the uniform coverage, we bound the sub-optimality gap by  $\tilde{O}\left( H^2 d \sqrt{\varepsilon} + H^{5/4} d^{3/4}\varepsilon^{1/4}\right)$. For $\varepsilon$ small (i.e. $< 1/d$) it can be checked that the upper bound is $O(H^2 d^{3/4} \varepsilon^{1/4})$. In order to achieve this bound, we reduce our problem to corruption robust offline RL, by using an existing corruption-robust method~\parencite{zhang2022corruption} as a biased,  zero-order oracle, and using the technique of \emph{Gaussian approximation}~\cite{NS17} to construct an approximate sub-gradient. We also develop a method of convex optimization with biased zero-order oracle that might be of independent interest.
    \item Finally, we show that we can improve the sub-optimality gap to $\tilde{O}(H^2 d^{3/2} \sqrt{\varepsilon})$ if the offline data satisfies the assumption of \emph{bounded generalized coverage ratio}, an assumption recently considered by \citet{GNOP23} (see also \cite{JYW21} for a similar coverage ratio). In this case, we construct a new corruption robust offline RL that is \emph{first-order} i.e. not only returns an approximately optimal policy but also an approximate sub-gradient of the optimal value function.
\end{enumerate}

\subsection{Related Work}

\textbf{Preference-based RL}: Our work is related to preference-based reinforcement learning (PbRL)~\parencite{wirth2017survey, lee2021b}. Although the field of PbRL is not new, there have been significant recent interests in designing provably optimal RL methods from preferences~\parencite{zhan2023provable,ZJJ23, wang2023rlhf}. In particular, \citet{ZJJ23} proposed a pessimistic maximum likelihood estimation for provable PbRL under clean data. Our algorithm, in particular the reward confidence set construction, is related to the method proposed by \citet{zhan2023provable}. However, unlike \cite{zhan2023provable} we don't build a confidence set around the probability transpition function, but rather use reduction to offline RL. Finally, there are several works on PbRL in online setting~\parencite{pacchiano2021dueling, chatterji2021theory, chen2022human} which are complementary to the  offline setting. 

\textbf{Corruption robust RL}: Our work is closely aligned with the research on corruption robust RL, where the challenge lies in designing agents that can effectively learn in the presence of adversarial corruption on both rewards and transitions~\cite{rakhsha2020policy}. 
\citet{zhang2022corruption} has considered linear MDP, and have designed corruption robust offline RL by robustifying the least squares value iteration method. On the other hand, \cite{YYGZ23} has considered corruption robustness in general MDPs by adopting uncertainty weighting to nonlinear function approximation~\cite{ye2023corruption}. In the online RL setting, \citet{lykouris2021corruption,chen2021improved}  proposed robust RL methods capable of accommodating up to $\epsilon \leq O(1/\sqrt{T})$ fraction of corruptions. \citet{zhang2021robust} developed an online policy gradient method that is resilient against a constant fraction of adaptive corruption.
Furthermore, there are other approaches for robustness in offline RL, including model selection~\cite{wei2022model}, hybrid RL~\cite{panaganti2022robust}, and others~\cite{wu2022copa, yang2023towards}. 


\section{Preliminaries}\label{sec:preliminaries}

\textbf{Markov Decision Process}: Let $\calM = \left( \calS,\calA,P^\star,r^\star,H,\rho\right)$ be an episodic Markov Decision Process (MDP) where $\calS$ denotes the state space and $\calA$ denotes the action space. The initial state is sampled from the distribution $\rho$. $P^\star=(P^\star_1,\ldots,P^\star_{H})$ denote the transition kernels where  for each $h\in[H]$, $P^\star_h(\cdot |s,a)\in \Delta(\calS)$ denotes the distribution over states given that the system is in state $s$ at step $h$ and action $a$ is taken. Let $r^\star:\calS \times \calA$ be the reward function where $r^\star_h(s,a)$ is the reward obtained from taking action $a$ from state $s$ at time-step $h$. We can also  extend the reward function to reward over trajectories by taking the sum of the rewards over the $H$ steps. 
Specifically, given a trajectory $\tau = (s_1,a_1, s_2,\ldots,s_{H+1})$, we define $r^\star(\tau) = \sum_{h=1}^{H}r^\star_h(s_h,a_h)$. 

\textbf{Policy}: Policies denote mappings from histories of traversed state-action pairs to distributions over actions. Formally, a non-stationary history-dependent policy $\pi=(\pi_1,\ldots,\pi_H)$ is a sequence of mappings where, for each $h\in[H]$, $\pi_h : \calH_h \rightarrow \Delta(\calA)$, with $\calH_h = \calS\times (\calS\times\calA)^{H-1}\times \calS$ denoting the history space up to time-step $h$. The space of such policies is denoted by $\Pi_{\textrm{his}}$. We further denote by $q^\pi_h(s,a)=\mathbb{P}(s_h=s,a_h=a|\pi,P^\star)$ the state-action occupancy measure for every time-step $h\in[H]$.
The expected performance of a given policy $\pi$ with respect to the true transitions $P^\star$ and true reward $r^\star$ is denoted by 
$$\textstyle
    V^\pi(P^\star,r^\star) = \expp \Big[ \sum_{h=1}^H r^\star(s_h,a_h) \Big| s_h\sim P^\star_h,a_h\sim\pi_h \ \forall h\Big]~.
$$
\subsection{Offline RLHF}

We have an offline dataset $\calD = \left\{ (\tau^{n,0},\tau^{n,1},o^n\right\}^N_{n=1}$ of $N$ pairs of trajectories, where each pair $(\tau^{n,0},\tau^{n,1})$ is associated with feedback $o^n\in \{ +1,-1\}$ representing the human preference, coming from a latent model assumed to satisfy the following assumption. 

\begin{assumption}[Preference-based model]\label{asm:preference_model}
    Given a pair of trajectories $(\tau^0,\tau^1)$, and a  preference $o\in \{ +1,-1\}$, the probability that $\tau^1$ is preferred over $\tau^0$ satisfies
    \begin{align*}
        \mathbb{P}\left( o=1| \tau^0,\tau^1\right) = \sigma\left( r^\star(\tau^1) - r^\star(\tau^0) \right)~,
    \end{align*}
    where $\sigma:\mathbb{R}\rightarrow [0,1]$ is a monotonically increasing link function.
\end{assumption}
In this paper, we will utilize the sigmoid link function $\sigma (x) = 1/(1+\exp(-x))$, commonly used in the literature on RLHF~\cite{christiano2017deep}. For our setting, the rewards are bounded and the range of the function is bounded away from $0$ and $1$. This also implies that there exists a constant $\kappa$ such that $\sup_{p \in [0,1]} \abs{\frac{d\sigma^{-1}(p)}{dp}} \le \kappa$.
The performance of a given policy $\pi$ is measured by the notion of suboptimality gap with respect to a target policy $\pi^\star$. Formally, we want to minimize 
\begin{align*}
    \textrm{SubOpt}(\pi, \pi^\star) = V^{\pi^\star}(r^\star,P^\star) - V^\pi(r^\star,P^\star)~.
\end{align*}
We will assume that the pairs of trajectories $(\tau^0,\tau^1)$ in our offline dataset $\calD$ are drawn from pairs of behaviour policies $(\mu_0,\mu_1)$, and we will denote it as $\tau^0\sim \mu_0, \tau^1 \sim \mu_1$. Additionally, we will write $\Sigma^\diff_{\mu_0, \mu_1}$ to denote the difference feature covariance matrix, which is defined as 
\begin{equation}\label{eq:diff-cov-matrix}
    \Sigma^\diff_{\mu_0, \mu_1} = \mathop{\E}_{{\tau^0 \sim \mu_0,\tau^1 \sim \mu_1}}\left[\left(\phi(\tau^0) - \phi(\tau^1)\right) \left( \phi(\tau^0) - \phi(\tau^1)\right)^\top\right].
\end{equation}
Similarly, we will write $\Sigma^\avg_{\mu_0, \mu_1}$ to denote the average feature covariance matrix, which is defined as
\begin{equation}\label{eq:avg-cov-matrix}
    \Sigma^\avg_{\mu_0, \mu_1} = \mathop{\E}_{{\tau^0 \sim \mu_0,\tau^1 \sim \mu_1}}\left[\left(\phi(\tau^0) + \phi(\tau^1)\right) \left( \phi(\tau^0) + \phi(\tau^1)\right)^\top\right].
\end{equation}

\subsection{Contamination Model}

In this paper, we study the problem of corruption robustness in offline RLHF. We assume that the collected data contains an $\epsilon$-fraction of contaminated samples, i.e. an attacker, who is given access to the data beforehand, is allowed to arbitrarily modify up to an $\epsilon$-fraction of the data samples (both the the trajectory features, and the human feedback). Formally,  \textit{Huber contamination model of human preferences} is defined as follows.

\begin{assumption}[$\varepsilon$-corruption in Offline RLHF]\label{asm:corruption_model}
    Let $\varepsilon\in[0,1]$ denote the contamination parameter and $\widetilde{\calD}=\{ (\widetilde{\tau^{n,0}},\widetilde{\tau^{n,1}},\widetilde{o}^n\}^N_{n=1}$ be a dataset of $N$ pairs of trajectories and human preferences. An attacker inspects $\widetilde{\calD}$ and arbitrarily modifies up to $\epsilon N$ tuples from $\widetilde{\calD}$. We denote the corrupted dataset by $\calD=\{ (\tau^{n,0},\tau^{n,1},o^n)\}^N_{n=1}$. In other words, there are at most $\epsilon N$ indices $n$, for which we have $\widetilde{o}^n \neq o^n$, or $\widetilde{\tau^{n,1}} \neq \tau^{n,1}$, or $\widetilde{\tau^{n,0}} \neq \tau^{n,0}$.
\end{assumption}

\subsection{Parametric Markov Decision Processes}

It is generally impossible to design provable offline RL algorithms without making any parametric assumptions on the underlying MDP. Therefore, throughout this paper, we will assume that the MDP is linear i.e. the reward and the transition are linear functions of by $d$-dimensional features. 

\begin{definition}[Linear MDP~\cite{jin2020provably}]
    We assume access to known feature map $\phi: \calS \times \calA \rightarrow \R^d$, and that there exist $\{\theta_h\}_{h \in [H]}$  and signed measures $\bm{\mu}_h = (\mu^1_h,\ldots,\mu^d_h)$ over the state space so that 
    $$
    r_h(s,a) = \phi(s,a)^\top \theta_h \textrm{ and } P_h(s' \mid s,a) = \phi(s,a)^\top \bm{\mu}_h(s').
    $$
    We also assume $\norm{\phi(s,a)}_2 \le 1$ for any $s,a$, $\max \set{\norm{\theta_h}_2, \norm{\bm{\mu}_h(\calS)}_2} \le \sqrt{d}$ for any $h \in [H]$.
\end{definition}
Given a trajectory $\tau = (s_1,a_1,s_2,\ldots,s_{H+1})$ we will  write $\phi(\tau) = [\phi(s_1,a_1); \ldots; \phi(s_H,a_H)]$ to denote the feature of the trajectory $\tau$. Note that $\phi(\tau) \in \R^{Hd}$ and $\norm{\phi(\tau)}_2  \le \sqrt{H}$.


    $$
    $$

\section{Robust RLHF with Uniform Coverage}\label{sec:rlhf-uniform-coverage}

We now provide our first algorithm for corruption robust reinforcement learning from human feedback (RLHF). Standard RLHF framework estimates the reward parameter by solving a maximum likelihood estimation problem. We essentially robustify this step and replace it with a robust version of logistic regression. Let $\mathbb{P}_\theta(o \mid \phi(\tau^1) - \phi(\tau^0))$ be the probability of observing feedback $o$ from a comparison of trajectory $\tau^1$ and $\tau^0$. Then \cref{alg:robust_freehand_uniform_coverage} solves a trimmed maximum likelihood estimation problem, 
\begin{equation}
    \label{defn:trimmed-mle}
    \widehat{\theta} \leftarrow \textrm{argmax}_{\theta} \max_{S \subseteq \calD: \abs{S}=(1-\varepsilon)N} \sum_{n\in S}\log \mathbb{P}_\theta(o^n \mid \mathrm{x}^n)
\end{equation}
where $\mathrm{x}^n = \phi(\tau^{n,1}) - \phi(\tau^{n,0})$. Therefore, the estimate $\widehat{\theta}$ is chosen to maximize the likelihood over best subset containing $(1-\varepsilon)$-fraction of the points. With an estimate $\widehat{\theta}$ of the reward parameter, \cref{alg:robust_freehand_uniform_coverage} uses a robust offline RL method (input \texttt{RobRL}) to compute an approximately optimal policy $\tilde{\pi}$. For this step, \cref{alg:robust_freehand_uniform_coverage} uses the features from the dataset $\widehat{\calD}$ but the reward is defined according to the estimated model $\widehat{\theta}$ i.e. $r_h(s,a) = \widehat{\theta}_h^\top \phi(s,a)$. 
Finally, note that the dataset $\calD$ is split uniformly at random into two datasets of equal sizes at the beginning, and one partition is used for reward estimation, and the other for  policy optimization.

\begin{algorithm}[!t]
\caption{Robust RLHF (with Uniform Coverage)}\label{alg:robust_freehand_uniform_coverage}
\KwIn{
(a) Dataset $\calD$, (b) corruption parameter $\epsilon$,  (c) corruption robust offline RL algorithm \texttt{RobRL}.}
Partition dataset $\calD$ uniformly at random into three datasets $\calD_1$, $\calD_2$, and $\calD_3$ of equal sizes.\\
Use $\calD_1$ to build a robust estimate $\widehat{\Sigma}$ of $\E[(\phi(\tau^1) - \phi(\tau^0)) (\phi(\tau^1) - \phi(\tau^0))^\top]$ (e.g. \cite{diakonikolas2025sos}).\\
\tcc{Whiten covariates using $\widehat{\Sigma}^{-1/2}$}
Let $\mathrm{x}^n = \phi(\tau^{n,1}) - \phi(\tau^{n,0})$, and $S = \set{\widehat{\Sigma}^{-1/2} \mathrm{x}^n : \mathrm{x}^n \in \calD_2}$.\\
Let $\widetilde{\calD}_2$ be the output of running filtering ~\cite{diakonikolas2020outlier} algorithm on set $S$.\\
\tcc{Estimate reward parameter of linear MDP $\htheta = (\widehat{\theta}_1,\ldots,\widehat{\theta}_H$).}
Using \cref{alg:alternating_optimization}, find an approximate stationary point of the following trimmed maximum likelihood estimation problem.
\begin{align}
    \widehat{\theta} \leftarrow \argmax_{\theta} \max_{\stackrel{\widehat{S} \subseteq \widetilde{\calD}_2}{ \abs{\widehat{S}} = (1-\varepsilon) \abs{\widetilde{\calD}_2}}} \sum_{\mathrm{z}^n \in \widetilde{\calD}_2} \log \mathbb{P}_{\theta}(o^n \mid \widehat{\Sigma}^{1/2} \mathrm{z}^n)
\end{align}

Let $\tilde{\pi}$ be the policy returned by  \texttt{RobRL} with reward function $r_h(s,a) = \phi(s,a)^\top \widehat{\theta}_h$ and dataset $\calD_3$.\\
\Return $\tilde{\pi}$.
\end{algorithm}

%
 

\textbf{Solving Trimmed MLE}: Before providing the performance guarantees of \cref{alg:robust_freehand_uniform_coverage},  we discuss how to solve the trimmed maximum likelihood estimation (MLE) problem~\eqref{defn:trimmed-mle}. In general, the optimization problem is hard to solve, but it is known that for the setting of generalized linear models, alternating optimization methods converge to a stationary point under certain assumptions on the link function~\cite{ADKS22}. Therefore, \cref{alg:robust_freehand_uniform_coverage} calls an alternating optimization method (\cref{alg:alternating_optimization}) to obtain an approximate stationary point of the trimmed MLE problem. However, alternating optimization succeeds only when the preferences are corrupted. Since the features in the trajectories can be corrupted, we first whiten covariates using a robust estimate of the covariance matrix $\Sigma$ and then run the alternating optimization method on the whitened features. We will write $\widetilde{D}_2$ to denote the set of whitened features. 

Let us now define the notion of an approximate stationary point. Given an estimate $\hat{\theta}$, let the set $\widehat{S}$ contain $(1-\varepsilon)N$ data-points with the largest log-likelihood under $\hat{\theta}$. Then we say  $\hat{\theta}$ is a $\gamma$-stationary point if the following condition is satisfied.
\begin{equation}\label{eq:apx-stationarity}  
    \frac{1}{N} \sum_{n \in \widehat{S}} \nabla_\theta \log \mathbb{P}_{\hat{\theta}}(o^n \mid \mathrm{x}_n)^\top \frac{(\theta^\star  - \hat{\theta})}{\lVert\theta^\star  - \hat{\theta}\rVert_2} \le \gamma 
\end{equation}

\begin{algorithm}[!h]
\caption{Alternating Optimization}
\label{alg:alternating_optimization}
\KwIn{ Whitened covariates  $\widetilde{\calD}_2$, corruption parameter $\epsilon$, slackness parameter $\eta$, transition model $P^\star$. }
 Let $\mathrm{x}^n = \phi(\tau^{n,1}) - \phi(\tau^{n,0})$.\\
 Set $\widehat{\theta}_1 = 0$.\\
\For{$t=1,2,\ldots$}{
$\widehat{S}_t \leftarrow \argmax_{\stackrel{S \subseteq [N]}{ \abs{S} = (1-\varepsilon) N} } \sum_{\mathrm{z}^n \in {S}} \log \mathbb{P}_{\widehat{\theta}_t} (o^n \mid \widehat{\Sigma}^{1/2}\mathrm{z}^n)$.\\
$\widehat{\theta}_{t+1} \leftarrow \argmax_{\theta: \norm{\theta}_2 \le \sqrt{Hd}}  \sum_{\mathrm{z}^n \in \widehat{S}_t} \log \mathbb{P}_{{\theta}} (o^n \mid \widehat{\Sigma}^{1/2} \mathrm{z}^n)$.\\
\tcc{Check Progress}
\If{$\sum_{\mathrm{z}^n \in \widehat{S}_t} \log \mathbb{P}_{\widehat{\theta}_{t+1}} (o^n | \mx^n) \le \sum_{\mathrm{z}^n \in \widehat{S}_t} \log \mathbb{P}_{\widehat{\theta}_t} (o^n | \widehat{\Sigma}^{1/2} \mathrm{z}^n) + \eta$}{
Return $\widehat{\theta}_t$.
}
}
\end{algorithm}
We propose an alternating optimization based method (\Cref{alg:alternating_optimization}) to obtain such an approximate stationary point of the trimmed MLE problem. At each iteration $t$, the alternating optimization updates $\widehat{S}$ and $\widehat{\theta}$ as follows.
\begin{enumerate}
    \item $\widehat{S}_t \leftarrow \argmax_{\stackrel{S \subseteq \widetilde{\calD}_2}{ \abs{S} = (1-\varepsilon) \abs{\widetilde{\calD}_2}} } \sum_{\mathrm{z}^n \in S} \log \mathbb{P}_{\widehat{\theta}_t} (o^n \mid \widehat{\Sigma}^{1/2}\mathrm{z}^n)$.
    \item $\widehat{\theta}_{t+1} \leftarrow \argmax_{\theta: \norm{\theta}_2 \le \sqrt{Hd}}  \sum_{\mathrm{z}^n \in \widehat{S}_t} \log \mathbb{P}_{{\theta}} (o^n \mid \widehat{\Sigma}^{1/2} \mathrm{z}^n)$,
\end{enumerate}

where $\mathrm{z}^n$ denotes the whitened feature vector for the $n$-th sample. The method stops when the improvement in the likelihood is less than a threshold $\eta$. The next lemma shows (proof in Appendix \ref{sec:first-sec-app}) that algorithm \ref{alg:alternating_optimization} obtains a $\gamma$-stationary point of the trimmed maximum likelihood estimation problem~\eqref{defn:trimmed-mle} for $\gamma = \max \set{2L \varepsilon, \frac{\varepsilon^2}{\norm{\theta^\star - \widehat{\theta}}_2} }$. 


\begin{lemma}\label{lem:apx-stationarity}
Suppose $\norm{\mathrm{x}_n}_2 \le L$ for all $n$ and we set $\eta = \varepsilon^2$. Then algorithm~\ref{alg:alternating_optimization} computes a $\max \set{O(L \varepsilon), \frac{\varepsilon^2}{\norm{\theta^\star - \widehat{\theta}}_2} }$-stationary point of the trimmed maximum likelihood estimation problem~\eqref{defn:trimmed-mle}.
\end{lemma}

Note that algorithm \ref{alg:alternating_optimization} stops when the improvement in the objective falls below $\varepsilon^2$. Furthermore, the value of the log-likelihood can be bounded by $O(L)$ which implies that \cref{alg:alternating_optimization} runs for at most $O(L/\varepsilon^2)$ iterations. We next show that an approximate sationary solution of trimmed MLE is enough to provide a bound on the sub-optimality of the policy $\tilde{\pi}$ returned by \cref{alg:robust_freehand_uniform_coverage}. We will make the following assumption regarding the pair of policies $\mu_0, \mu_1$ that generate the offline data.



\begin{assumption}[Uniform Coverage]\label{asn:uniform_coverage}
Suppose $\norm{\phi(\tau)}_2 \le L$ for any trajectory $\tau$.    Then there exists a constant $\xi > 0$ such that $$\Sigma^\diff_{\mu_0, \mu_1} \succcurlyeq \xi L \cdot \Identity_d.$$
\end{assumption}

\begin{theorem}\label{thm:reduction_to_offline_rl_uniform}
    Suppose 
    \texttt{RobRL} returns a $f(\varepsilon)$-robust estimate of the optimal value function, and \cref{asn:uniform_coverage} holds with $\xi \ge 5 \varepsilon$ and $N \ge \Omega\left(\frac{H^{3/2}}{\varepsilon^2}\left( d + \log(1/\delta) \right) \right)$. Then with probability at least $1-\delta$, the policy $\tilde{\pi}$ returned by \cref{alg:robust_freehand_uniform_coverage} satisfies,
    $$
    V^\star(\theta^\star) - V^{\tilde{\pi}}(\theta^\star) \le  f(\varepsilon) + 2  \sqrt{Hd} C_1 \frac{\varepsilon}{\xi} \cdot \exp(2L +{ \sqrt{\log\left( {1}/{2\delta \varepsilon}\right)} })\footnote{For most of our results, the sub-optimality can be shown to be of the form $g(\varepsilon) + c/\sqrt{N}$. We assume that the number of samples $N$ is large so that the term involving $\varepsilon$ is the dominating term. This is to simplify the results and follows existing literature~\cite{DK19}.}
    $$
\end{theorem}

The proof of the theorem is provided in the Appendix, but the main idea is to  show that the estimate $\widehat{\theta}$ obtained from solving trimmed MLE satisfies $\norm{\htheta - \theta^\star}_2 = O(\varepsilon^{1-o(1)})$.
We now instantiate \Cref{thm:reduction_to_offline_rl_uniform} for the setting of linear MDP. For corruption robust offline RL, we use algorithm \texttt{R-LSVI} from \cite{zhang2022corruption} as an oracle. It requires a coverage assumption similar to assumption~\eqref{asn:uniform_coverage}. 

\begin{assumption}[Uniform Coverage: V2]\label{asn:uniform_coverage-v2}
Suppose, $\norm{\phi(\tau)}_2 \le L$ for any trajectory $\tau$.    Then there exists a constant $\xi > 0$ such that $$\Sigma^\avg_{\mu_0, \mu_1} \succcurlyeq \xi L \cdot \Identity_d.$$
\end{assumption}

Under assumption~\eqref{asn:uniform_coverage-v2} \texttt{R-LSVI} returns a policy $\tilde{\pi}$ so that 
$
V^{\tilde{\pi}}(s_0) \ge V^\star(s_0) - f(\varepsilon)$ where $f(\varepsilon) = \tilde{O}\left( \frac{H^{5/2}}{\xi \sqrt{N}}\textrm{poly}(d) + \frac{H^3}{\xi}\varepsilon \right).
$
Note that if $N \ge \tilde{\Omega}\left( \frac{\textrm{poly}(d)}{\varepsilon^2}\right)$, we have $f(\varepsilon) = \frac{H^3}{\xi}\varepsilon$. Substituting this value of $f(\varepsilon)$ in the bound of theorem~\eqref{thm:reduction_to_offline_rl_uniform} gives us the following bound on the suboptimality gap.

\begin{proposition}
    Suppose assumptions \ref{asn:uniform_coverage} and \ref{asn:uniform_coverage-v2} hold. Then for the setting of linear MDP and $N \ge \tilde{\Omega}\left( \frac{\textrm{poly}(d)  \log\left( 1/\delta \right)}{\varepsilon^2}\right)$, \Cref{alg:robust_freehand_uniform_coverage} returns a policy $\tilde{\pi}$ so that with probability at least $1-\delta$,
    $$\textstyle
    V^\star(\theta^\star) - V^{\tilde{\pi}}(\theta^\star) \le O\left(\frac{H^3 + \sqrt{Hd}}{\xi} \varepsilon \cdot \exp( 2L + \sqrt{\log\left( {1}/{2\delta\varepsilon}\right)} ) \right).
    $$
  
\end{proposition}

\section{Low Relative Condition Number}\label{sec:low-relative-condition-number}
Although the assumption of uniform coverage allows us to design $O(\varepsilon^{1-o(1)})$-optimal policy, it is a strong assumption since the features generated by the offline policy might not cover the entire $d$-dimensional space. In this section, we relax this assumption to a new assumption named \emph{Low Relative Condition number}, which is significantly weaker. 
\begin{assumption}[Relative Condition Number]\label{asn:rel_cond_number}
    Let $\Sigma^\diff_{\mu_0, \mu_1}$ (resp. $\Sigma^\diff_{\pi_0,\pi_1}$) be the difference feature covariance matrix under pair of policies $\mu_0$ and $\mu_1$ (resp. $\pi_0$ and $\pi_1$), as defined in \cref{eq:diff-cov-matrix}.
   Then there exists a constant $\alpha > 0$ such that the following holds.
    $$
    \sup_{w} \frac{w^\top \Sigma^\diff_{\pi_0, \pi_1} w}{w^\top \Sigma^\diff_{\mu_0, \mu_1} w} = \alpha < \infty
    $$
\end{assumption}
Although the above assumption is stated for a pair of policies $\pi_0, \pi_1$, given a target policy $\pi^\star$ one can choose $\pi_0 = \pi^\star$ and $\pi_1 = \mu_1$, and the assumption needs to hold only for this pair of policies.

\Cref{alg:robust_freehand_unknown_P} provides our new corruption robust RLHF method under the assumption of low relative condition number. The algorithm begins similarly to \cref{alg:robust_freehand_uniform_coverage} by solving the trimmed maximum likelihood estimation to obtain a robust estimate $\htheta$ of the reward parameter $\theta^\star$. However, in the absence of uniform coverage, $\htheta$ might not be close to $\theta^\star$ in terms of $L_2$ distance. So the following lemma provides a bound in terms of the likelihood at $\htheta$ and $\theta^\star$.

\begin{lemma}\label{lem:diff-log-likelihood}
Let $\mathbb{P}_\theta(y \mid x) =  {1}/{(1 + e^{-y \cdot \theta^\top x})}$ and for any $n \in [N]$ define $\mathrm{x}_n = \phi(\tau^{1,n}) - \phi(\tau^{0,n})$. Then   with probability at least $1-\delta$, we have
    $$
    \frac{1}{N} \sum_{n=1}^N \log \left( \frac{\mathbb{P}_{ \htheta }(o^n \mid \mx_n)}{\mathbb{P}_{\theta^\star}(o^n \mid \mx_n)}\right) \le 6\varepsilon H\sqrt{d} + c \cdot \frac{d}{N} \log\left( \frac{HN}{\delta}\right)
    $$
\end{lemma}

The above result is a generalization of Lemma 1 from \cite{zhan2023provable}, and allows us to build a confidence set around the estimate $\htheta$ even when a $\varepsilon$-fraction of the data has been corrupted (line 3 of algorithm~\eqref{alg:robust_freehand_unknown_P}). Now we leverage two important observations.

First, the above approach requires the set $\Theta(\calD_1)$ is a convex set. It can be easily checked this is true if the function $\log \mathbb{P}_\theta (\cdot)$ is concave. Moreover, for the case of sigmoid link function   $\nabla^2_\theta \log \sigma(\theta^\top x) \preccurlyeq 0$ i.e. $\Theta(\calD_1)$ is a convex set. 
Second, for the setting of linear MDP, the optimal value function $V^\star(\theta) = \max_\pi V^\pi(\theta)$ is a convex function in the reward parameter $\theta$. This follows from the  occupancy measure based representation of MDP. Indeed, $V^\star(\theta) = \max_{q^1,\ldots,q^H \in \calC} \sum_{h=1}^H q_h^\top \Phi \theta$, where $\calC$ is the set of valid occupancy measures, and $\Phi$ is the feature matrix. Since for any $\theta$, $V^\star(\theta)$ is a maximum over linear functions, $V^\star(\cdot)$ is convex.

\begin{algorithm}[!ht]
\caption{Robust RLHF with Condition Number}\label{alg:robust_freehand_unknown_P}
\KwIn{
(a) Corrupted dataset $\calD$, (b) corruption parameter $\epsilon$,  (c) corruption robust offline RL algorithm \texttt{RobRL}, (d) reference distribution $\mu_{\textrm{ref}}$.}
Partition dataset $\calD$ uniformly at random into $\calD_1$ and $\calD_2$ of equal size. \\
\tcc{Build an estimate $\htheta$ of the reward parameter, as in Algorithm~\eqref{alg:robust_freehand_uniform_coverage}.} 
Set $\zeta = 6\varepsilon H \sqrt{d} + 2 \frac{d}{N}\log\left( \frac{HN}{\delta}\right)$ and   
$ 
    \Theta ({\calD}_1) = \bigg\{ \theta : \norm{\theta}_2 \le \sqrt{Hd} \wedge  \frac{2}{N}\sum^{N/2}_{n=1} \log   \frac{\sigma \left( \theta^\top \mathrm{x}^n \right)}{\sigma \left( \widehat{\theta}^\top \mathrm{x}^n \right)} \ge - \zeta \bigg\}
$ 
\tcc{Run Projected Sub-gradient Descent with Biased Oracle} 
Initialize $\theta_0 \in \Theta(\calD_1)$.\\
\For{$t=0,\ldots,T-1$}{
\tcc{Sub-Gradient Construction}
Generate $u_1,\ldots,u_K$ uniformly at random from the standard normal distribution.\\
Let $g_t = \frac{1}{K}\sum_{k=1}^K \frac{\widehat{V}(\theta_t + \mu u_k)-\widehat{V}(\theta) - \mu \cdot \E_{\tau \sim \mu_{\textrm{ref}}}\left[ \phi(\tau)^\top u_k \right]}{\mu} u_k$ be the approximate sub-gradient, where $\widehat{V}(\theta)$ is the value estimate returned by \texttt{RobRL} with reward function $r_h(s,a) = \phi(s,a)^\top \theta_{t,h}$ and dataset $\calD_2$.\\
$\theta_{t+1} = \textrm{Proj}_{\Theta(\calD_1)}\left( \theta_t - \eta g_t\right)$
}
Set $\Bar{\theta} = \frac{1}{T}\sum_{k=1}^T \theta_k$ and let $\tilde{\pi}$ be the policy returned by running \texttt{RobRL} with reward function $r_h(s,a) = \phi(s,a)^\top \Bar{\theta}_h$.\\
\Return $\tilde{\pi}$.
\end{algorithm}

Therefore, we run a projected subgradient descent over the set $\Theta(\calD_1)$. At each iteration $t$, algorithm~\eqref{alg:robust_freehand_unknown_P} selects a reward parameter $\theta_t$. Although the corruption robust offline RL method \texttt{RobRL} can return an approximately optimal policy with reward parameter $\theta$, we need a subgradient i.e. $g_t \in \delta_\theta V^\star(\theta_t) = \set{v : V^\star(\theta') \ge V^\star(\theta_t) + v^\top(\theta' -\theta_t)\ \forall \theta'}$. So we treat \texttt{RobRL} as a biased, zero-order oracle and explicitly build an estimator of a subgradient (lines 8-9)~\cite{NS17, DJWW15, flaxman2004online}. In particular, we use the \emph{gaussian approximation} technique introduced by \cite{NS17}. Given a convex function $f: E \rightarrow \R^d$, let $f_\mu$ be its smoothed Gaussian approximation, defined as
$$
f_\mu(\theta) = \frac{1}{\kappa} \int_E f(\theta + \mu \cdot u) e^{-1/2 \norm{u}_2^2} du,
$$
where $\kappa = \int_E e^{-1/2 \norm{u}_2^2} du$. The Gaussian approximation method performs a gradient descent of the smoothed function $f_\mu$, with the gradient
$$
\nabla f_\mu(\theta) = \frac{1}{\kappa} \int_E \frac{f(\theta + \mu \cdot u) - f(\theta)}{\mu} e^{-1/2 \norm{u}_2^2} du.
$$
 \Cref{alg:robust_freehand_unknown_P} constructs an  estimator of $\nabla f_\mu(\theta) $ for $f(\theta) = V^\star(\theta) - \E_{\tau \sim \mu_\tref}\left[ \phi(\tau)^\top \theta\right]$\footnote{We subtract rewards according to a reference policy $\mu_\tref$ since we only have preference data over rewards.}. 
 The algorithm finally computes the average reward parameter $\bar{\theta} = 1/T\cdot \sum_{k=1}^T \theta_k$ and returns a robust policy $\tilde{\pi}$ with respect to the parameter $\bar{\theta}$. 
\Cref{alg:robust_freehand_unknown_P} provides our full implementation of the reduction to corruption robust RL. The next theorem provides a bound on the sub-optimality gap of \cref{alg:robust_freehand_unknown_P}, assuming access to a $f(\varepsilon)$-robust offline RL method.

\begin{theorem}\label{thm:relative-cond-number-bound}
    Suppose \cref{asn:rel_cond_number} holds,  $\sup_{p \in [0,1]} \abs{\frac{d\sigma^{-1}(p)}{dp}} \le \kappa$, and \texttt{RobRL} returns a $f(\varepsilon)$-robust estimate of the optimal value function. If $N \ge \tilde{\Omega}\left( \frac{H^{3/2}d^5   }{\varepsilon^3}\right)$, then for a target policy $\pi^\dagger$, the policy $\tilde{\pi}$ output by \cref{alg:robust_freehand_unknown_P} satisfies the following w.p. at least $1-\delta$.
\begin{align*}
V^{\pi^\dagger}(\theta^\star) &- V^{\tilde{\pi}}(\theta^\star) \le f(\varepsilon) + 8 \sqrt{f(\varepsilon)}(Hd)^{1/4} \\
&+ c \kappa \sqrt{\alpha} \left( \sqrt{\varepsilon H} d^{1/4}  + \sqrt{{d}/{N} \cdot \log \left( {HdN}/{\delta}\right)}\right) 
\end{align*}
\end{theorem}

We now instantiate \Cref{thm:relative-cond-number-bound} for the setting of linear MDP. For corruption robust offline RL, we use algorithm (\texttt{R-LSVI} from \cite{zhang2022corruption}) as an oracle, which requires a coverage assumption.
\begin{assumption}[Relative Condition Number: V2]\label{asn:rel_cond_number_v2}
  Let $\Sigma^\avg_{\mu_0, \mu_1}$ be the average feature covariance matrix under pair of distributions $\mu_0$ and $\mu_1$, as defined in \cref{eq:avg-cov-matrix}.
    Then there exists a constant $\alpha > 0$ such that  the following condition holds.
    $$
    \sup_{w} \frac{w^\top \Sigma_{\pi^\star} w}{w^\top \Sigma^\avg_{\mu_0, \mu_1} w} = \alpha < \infty
    $$
\end{assumption}

Under \cref{asn:rel_cond_number_v2}, \texttt{R-LSVI} returns a policy $\tilde{\pi}$ so that 
$V^{\tilde{\pi}}(s_0)\ge V^\star(s_0) - f(\varepsilon)$ where $f(\varepsilon) = \tilde{O}\left( \frac{H^{5/2}}{\sqrt{N}} \sqrt{\alpha} \textrm{poly}(d) + H^2 d \sqrt{\alpha \varepsilon}\right).$


\begin{proposition}\label{prop:bound-finite-relative-cond-number}
    Suppose assumptions \eqref{asn:rel_cond_number} and \eqref{asn:rel_cond_number_v2} hold. Moreover, suppose $\sup_{p \in [0,1]} \abs{\frac{d \sigma^{-1}(p)}{dp}} \le \kappa$. Then for the setting of linear MDP and $N \ge \tilde{\Omega}\left(\frac{H^{3/2}}{\varepsilon^3} \cdot \textrm{poly}(d,1/\delta) \right)$, \cref{alg:robust_freehand_unknown_P} returns a policy $\tilde{\pi}$ so that with probability at least $1-\delta$,
    $$
    V^\star(\theta^\star) - V^{\tilde{\pi}}(\theta^\star)  \le \tilde{O}(H^2 d \kappa \sqrt{\alpha \varepsilon}) + \tilde{O}\left( H^{5/4} d^{3/4} (\alpha \varepsilon)^{1/4}\right)
    $$
    
\end{proposition}

\Cref{prop:bound-finite-relative-cond-number} provides an upper bound of $O(\varepsilon^{1/4})$ when other parameters are constant. The reason we obtain sub-optimal dependence on $\varepsilon$ is because we assume a zero-order access to the offline robust RL oracle. We now show that we can improve the dependence on $\varepsilon$ with access to a first-order oracle. 
\section{ Bounded Generalized Coverage Ratio}\label{sec:bounded-coverage-ratio}
\begin{algorithm}[!h]
\caption{Robust FreeHand with First-Order Oracle}\label{alg:robust_freehand_first_order}
\KwIn{
(a) Corrupted dataset $\calD$, (b) corruption parameter $\epsilon$,  (c) corruption robust offline RL algorithm \texttt{RobRL}, (d) reference distribution $\mu_{\textrm{ref}}$.}
\tcc{Estimate $\widehat{\theta}$ and build confidence interval $\Theta(\calD_1)$ as in algorithm~\eqref{alg:robust_freehand_unknown_P}.} 
Initialize $\theta_0 \in \Theta(\calD_1)$.\\
\For{$t=0,\ldots,T-1$}{
Let $g_t$ be the sub-gradient returned by running \texttt{RobRL} with reward parameter $r_h(s,a) = \phi(s,a)^\top \theta_{t,h}$ and dataset $\calD_2$.\\
$\theta_{t+1} = \textrm{Proj}_{\Theta(\calD_1)}\left( \theta_t - \eta \left( g_t + \E_{\tau \sim \mu_{\textrm{ref}}}[ \phi(\tau)]\right)\right)$
}
Set $\Bar{\theta} = \frac{1}{T}\sum_{k=1}^T \theta_k$ and let $\tilde{\pi}$ be the policy returned by running \texttt{RobRL} with reward function $r_h(s,a) = \phi(s,a)^\top \Bar{\theta}_h$ and dataset $\calD_2$.\\
 \Return $\tilde{\pi}$.
\end{algorithm}

\Cref{alg:robust_freehand_first_order} assumes access to a robust offline RL oracle \texttt{RobRL}, that given any reward parameter $\theta$, returns an approximate sub-gradient of the optimal value function $V^\star(\theta) = \max_{\pi}V^\pi(\theta)$. Given such a first order oracle, it essentially performs a projected subgradient descent to determine an approximately optimal reward parameter $\bar{\theta}$, and the corresponding policy $\tilde{\pi}$.

\begin{theorem}\label{thm:bounded-coverage-feehand}
    Suppose \cref{asn:rel_cond_number} holds,  $\sup_{p \in [0,1]} \abs{\frac{d\sigma^{-1}(p)}{dp}} \le \kappa$, and \texttt{RobRL} returns a $f(\varepsilon)$-robust estimate of the optimal value function, and $f(\varepsilon)$-approximate subgradient with norm at most $G$. If $N \ge {\Omega}\left(\frac{H^{3/2}dG}{f(\varepsilon)^2} \right)$, then with probability at least $1-\delta$,   the following holds for any  policy $\pi^\dagger$.
\begin{align*}
V^{\pi^\dagger}(\theta^\star) - V^{\tilde{\pi}}(\theta^\star) &\le 2f(\varepsilon)  + c \kappa \sqrt{\alpha} \left( \sqrt{\varepsilon H} d^{1/4}  + \sqrt{{d}/{N} \cdot \log \left( {HdN}/{\delta}\right)} \right) 
\end{align*}
\end{theorem}

We now construct a corruption robust sub-gradient estimator of the function $V^\star(\theta) = \max_\pi V^\pi(\theta)$. 
Given a reward parameter $\theta = (\theta_1,\ldots,\theta_H)$, the optimal value function can be expressed as follows.
$$ 
V^\star(\theta) = \max_{q = (q_1,\ldots,q_H) \in \calC} \sum_{h=1}^H q_h^\top \Phi \theta_h.
$$
Here $q_h$ is the state, action occupancy measure at time step $h$, and the constraint set $\calC$ ensures the Bellman flow constraints. Now from the definition of sub-gradient of a convex function which is expressed as a maximum of affine function (~\cite{Nesterov18}, chapter 3) we can write down the following expression of the sub-differential.
\begin{align*}
    \delta_\theta V^\star(\theta) &= \textrm{co} \bigg\{(\Phi^\top q_1,\ldots,\Phi^\top q_H):  (q_1,\ldots,q_H) \in \argmax_{q = (q_1,\ldots,q_H) \in \calC} \sum_{h=1}^H q_h^\top \Phi \theta_h \bigg \}
\end{align*}
Here $\textrm{co}(S)$ is the convex-hull of a set $S$. Since $q_h$ is the state, action occupancy measure at time-step $h$, $\Phi^\top q_h$ is the average feature observed at time-step $h$, and the result states that the subdifferential set is the convex hull of reward-maximizing average features. Therefore, we will construct a corruption robust offline RL method, that not only returns an approximately optimal policy but also the average feature under that policy. We  make the following assumption.

\begin{assumption}[Bounded Generalized Coverage Ratio]\label{asn:bounded_coverage}
    For a target policy $\pi^\star$, there exists $\nu > 0$ so that
    $$
    \E_{\tau \sim \pi^\star}[\phi(\tau)]^\top \left(\Sigma^\avg_{\mu_0, \mu_1}\right)^{-2} \E_{\tau \sim \pi^\star}[\phi(\tau)]< \nu
    $$
\end{assumption}
We have stated the above assumption assuming $\Sigma^\avg_{\mu_0,\mu_1}$ is invertible, but this is only for simplicity and consistency with prior literature. An alternate way to state this assumption would be that there exists a vector $y \in \R^d$ such that $\E_{(s,a) \sim \pi^\star}[\phi(s,a)] = \Sigma^\avg_{\mu_0, \mu_1} y$ and $\norm{y}_2^2 < \nu$. 

Our method is based on the primal-dual framework of linear MDP and builds upon the recent work by \cite{GNOP23}, who considered a similar assumption for discounted MDP. The standard linear program for a finite horizon linear MDP is the following optimization problem.
\begin{align*}
    \max_q\ &\sum_{h=1}^H q_h^\top \Phi \theta_h\\
    \textrm{s.t.}\ &\sum_a q_1(s,a) = \rho(s)\ \forall s\\
    &E q_{h+1} = \bm{\mu}_h \Phi^\top q_h\ \forall h \in \set{1,2,\ldots,H-1}\\
    &q_h \ge 0 \ \forall h \in [H]
\end{align*}
Here $\Phi\in \R^{SA\times d}$ is the feature matrix and the matrix $E \in \R^{S\times SA}$ is defined as $E[s,(s',a')] = \one \set{s = s'}$. The constraints specify Bellman-flow constraints at each time step $h$. We now substitute  $\lambda_h = \Phi^\top q_h$ to the above LP formulation, with the interpretation that $\lambda_h$ denotes the expected feature at time step $h$. 
\begin{align}\label{eq:primal-LP-2} 
\begin{split}
    \max_{\stackrel{\{q_h: q_h \ge 0\}_{h=1}^H,}{\{\lambda_h\}_{h=1}^H }} &\sum_{h=1}^H \lambda_h^\top  \theta_h\\
    \textrm{s.t.}\ &Eq_1 = \rho ,\
    E q_{h+1} = \bm{\mu}_h \lambda_h\ \forall h \in [H-1]\\
    &\lambda_h = \Phi^\top q_h \ \forall h \in [H]
    \end{split}
\end{align}

 Note that this substitution doesn't change the optimal value of the LP and we aim to solve the optimization problem~\ref{eq:primal-LP-2} instead of the original LP. The dual problem of \cref{eq:primal-LP-2} is given as follows.
\begin{align}
    \label{eq:dual-LP-2}
    \begin{split}
    \min_{\{v_h\}_{h=1}^H, \{w_h\}_{h=1}^H}\ &\rho^\top v_1\\
    \textrm{s.t.}\
    &E^\top v_h \ge \Phi w_{h}\ \forall h \in [H]\\
    &w_h \ge \theta_h + \bm{\mu}^\top_h v_{h+1} \ \forall h \in [H-1]\\
    &w_H \ge \theta_H
    \end{split}
\end{align}

Suppose $\calL(\bm{q},\bm{\lambda}; \bm{v}, \bm{w}) $ is the Lagrangian corresponding to the optimization problem above.
Then the main idea is to solve a saddle point of the Lagrangian i.e. $\max_{\bm{q},\bm{\lambda}} \min_{\bm{v}, \bm{w}} \calL(\bm{q},\bm{\lambda}; \bm{v}, \bm{w})$ through a gradient descent-ascent based algorithm. However, $\bm{q}$ and $\bm{v}$ are infinite dimensional parameters. So we represent them symbolically in terms of $\bm{\lambda}$ and $\bm{w}$, and perform gradient descent-ascent steps over the $H \cdot d$ dimensional parameters $\bm{\lambda}$ and $\bm{w}$. 

Note that, we don't exactly know the Lagrangian, and hence can only estimate the gradients through samples collected from the offline behavioral policy. However, recall that a $\varepsilon$-fraction of the data is corrupted, and hence we use robust mean to estimate the gradient from corrupted data. Additionally, as noted by \cite{GNOP23}, computing estimates of the gradients require explicit knowledge of the feature covariance matrix $\Lambda_h = \E_{(s,a) \sim \mu^h_\tref}\left[ \phi(s,a) \phi(s,a)^\top \right]$. It turns out that a substitution $\lambda_h = \Lambda_h \beta_h$ lets us compute an estimate of the gradient without knowledge of the covariance matrix $\Lambda_h$. Hence we compute the saddle point of the  Lagrangian $\calL_R(\bm{q},\bm{\beta}; \bm{v}, \bm{w}) = \calL(\bm{q},\bm{\lambda}; \bm{v}, \bm{w}) \mid_{\{\lambda_h = \Lambda_h \beta_h\}_{h\in [H]}}$ through \emph{robust} gradient descent-ascent method.

Once we obtain a solution $(\bar{\bm{\beta}}, \bar{\bm{w}})$, we choose policy $\bar{\pi}_h(a \mid s) \propto \exp\left( \phi(s,a)^\top \bar{w}_h\right)$ and set the primal solution $\lambda_h$ as $\widehat{\Lambda}_h \bar{\beta}_h$. Here, $\widehat{\Lambda}_h$ is an estimate of the feature covariance matrix at step $h$. Since $\varepsilon$-fraction of our data is corrupted, we use robust covariance estimation to build  $\widehat{\Lambda}_h$, and thereby obtain an approximate average features under $\bar{\pi}$. The full details of the algorithm is provided in the appendix (\cref{alg:offline-primal-dual}), and the next theorem provides the guarantees. 
\begin{theorem}\label{thm:robust-rl-oracle-first-order}
    Suppose assumption ~\eqref{asn:bounded_coverage}  holds, and $N \ge {\Omega}\left( \frac{H^2 d^4 \nu^4}{\varepsilon^2} (\log^2 d + \log^2 A)\right)$. Then there is an algorithm that runs in time $\textrm{poly}(H,d)$ and returns policy $\bar{\pi}$ and a vector $\widehat{v} = (\widehat{v}_1,\ldots, \widehat{v}_H)$ s.t.
    \[ 
    \max_\pi V^{\pi}(\theta) - \E\left[ V^{\bar{\pi}}(\theta) \right] \le O\left(\nu \sqrt{\varepsilon } H^2 d^{3/2}\right), \textrm{and} \ 
\]
    $$ 
    V^\star(\theta') \ge V^\star(\theta) + \sum_{h=1}^H \left \langle \widehat{v}_h, \theta_h\right \rangle -  O\left(\nu \sqrt{\varepsilon } H^2 d^{3/2}\right) \ \forall \theta'.
    $$
\end{theorem}
With such a first-order oracle, the next result states the improved guarantees given by \cref{alg:robust_freehand_first_order}.
\begin{proposition}\label{prop:bounded-coverage-ratio}
    Suppose assumptions \eqref{asn:rel_cond_number} and \eqref{asn:bounded_coverage} hold, and  $\sup_{p \in [0,1]} \abs{\frac{d \sigma^{-1}(p)}{dp}} \le \kappa$. If $N \ge \tilde{\Omega}\left(\frac{H^2 d^4 \nu^4}{\varepsilon^2}  \right)$, \cref{alg:robust_freehand_first_order} returns a policy $\tilde{\pi}$ so that the following holds.
    $$
    V^\star(\theta^\star) - V^{\tilde{\pi}}(\theta^\star)  \le {O}\left( \nu \sqrt{\varepsilon} H^2 d^{3/2}\right)$$
\end{proposition}
\section{Conclusion}
\label{sec:conclusion}
We have designed corruption robust offline RLHF algorithms under different types of coverage assumptions. When uniform coverage holds, we can recover almost optimal dependency on the parameter $\varepsilon$. It is also possible to obtain an upper bound of $O(\sqrt{\varepsilon})$ under a substantially weaker assumption of bounded general coverage ratio. In the standard offline RL, the assumption of a low relative condition number is sufficient to obtain a dependence of $O(\sqrt{\varepsilon})$. As pointed out by \cite{GNOP23}, these two assumptions are not directly comparable, and there is scope to explore the design of robust RLHF further.
%
%
%

In terms of future work, we have considered linear MDP in this work, and it would be interesting to consider non-convex reward functions or RLHF with general function approximation. However, such an extension is quite challenging.  \Cref{alg:robust_freehand_uniform_coverage} can be generalized by utilizing recent corruption robust RL under general function approximation~\cite{YYGZ23}, but we are not aware of similar results with weaker coverage assumptions. Furthermore, algorithms \ref{alg:robust_freehand_unknown_P}, and \ref{alg:robust_freehand_first_order} crucially depend on the fact that $V^\star(\theta) = \max_\pi V^\pi(\theta)$ is convex in $\theta$ for linear MDPs, and in the presence of non-convex reward functions, we will require new proof techniques for gradient based methods. 

Another interesting direction is to consider trajectory based rewards~\cite{zhan2023provable}, which requires non-Markovian RL policies. In this case, the computation of optimal policy itself is a hard problem, and the design of corruption robust RLHF will require different approaches. Finally, we have provided preliminary simulation results considering a large grid-world and linear parametrization. It would be quite interesting but challenging to see the effects of data corruption in practical RLHF setting e.g. fine-tuning large language models.


\printbibliography

\parttoc
\appendix
\addcontentsline{toc}{section}{Appendix}


\section{\texorpdfstring{Missing Proofs from Section~\ref{sec:rlhf-uniform-coverage}}{Missing Proofs from Section~(3)} }\label{sec:first-sec-app}
\subsection{Convergence of Alternating Optimization}

\begin{proof}
Let us write $H = \frac{1}{N\norm{\theta^\star - \hat{\theta}}_2^2}(\theta^\star - \hat{\theta})^\top \sum_{n \in \widehat{S}} \nabla^2_\theta \log \mathbb{P}_\theta(o^n \mid \mx^n) (\theta^\star - \hat{\theta})$ be the second order derivative in the direction of $ \theta^\star - \hat{\theta}$.
\begin{align*}
    \abs{H} &\le   \frac{1}{\abs{\widehat{S}} \norm{\theta^\star - \hat{\theta}}_2^2}  \left|\sum_{n \in \widehat{S}} (\theta^\star - \hat{\theta})^\top \frac{\exp(-o^n \cdot \theta^\top \mx_n)}{\left(1 + \exp(-o^n \cdot \theta^\top \mx_n)\right)^2 }  \mx_n \mx_n^\top (\theta^\star - \hat{\theta}) \right|\\
    &\le \frac{1}{\abs{\widehat{S}}} \norm{\sum_{n \in \widehat{S}} \frac{\exp(-o^n \cdot \theta^\top \mx_n)}{\left(1 + \exp(-o^n \cdot \theta^\top \mx_n)\right)^2 }  \mx_n \mx_n^\top}_2\\
    &\le \norm{\frac{1}{\abs{\widehat{S}}} \sum_{n \in \widehat{S}}   \mx_n \mx_n^\top}_2\\
    &\le \sup_{v: \norm{v}_2 = 1} \frac{1}{\abs{\widehat{S}}} \sum_{n \in \widehat{S}} (v^\top \mx_n)^2\\
     &\le \sup_{v: \norm{v}_2 = 1} \frac{1}{\abs{\widehat{S}}} \sum_{n \in \widehat{S}} (v^\top \widehat{\Sigma}^{1/2}(z_n- \widehat{\Sigma}^{-1/2}\mu) + v^\top \mu)^2\quad \textrm{[Since } \mx_n = \widehat{\Sigma}^{1/2} z_n]\\
     &\le \sup_{v: \norm{v}_2 = 1} \frac{1}{\abs{\widehat{S}}} \sum_{n \in \widehat{S}} 2(v^\top \widehat{\Sigma}^{1/2}(z_n- \widehat{\Sigma}^{-1/2}\mu) )^2 + 2(v^\top \mu)^2\\
     &\le \norm{\widehat{\Sigma}^{1/2}}_2^2 \cdot \sup_{v: \norm{v}_2 = 1} \frac{1}{\abs{\widehat{S}}} \sum_{n \in \widehat{S}} 2(v^\top (z_n- \widehat{\Sigma}^{-1/2}\mu) )^2 + 2 \norm{\mu}_2^2\\
\end{align*}
Now observe that the set $\widetilde{D}_2$ is a stable set as it is returned by the filtering algorithm. Additionally, by \cref{lem:stability-whitened-covariates} the set $\widehat{D}_2$ satisfies $\left(O(\varepsilon), O(\varepsilon L \norm{\widehat{\Sigma}^{-1/2}}_2 \sqrt{\log(1/\varepsilon)}) \right)$-stable with respect to the vector $\widehat{\Sigma}^{-1/2}\mu$. Since $\widehat{S}$ is a $(1-\varepsilon)$-dense subset of $\widetilde{D}_2$, we are guaranteed
$$
\sup_{v:\norm{v}_2 = 1} \abs{\frac{1}{\abs{\widehat{S}}} \sum_{n \in \widehat{S}} (v^\top (z_n- \widehat{\Sigma}^{-1/2}\mu) )^2 - 1} \le O\left(\varepsilon L^2 \norm{\widehat{\Sigma}^{-1/2}}_2^2 \log(1/\varepsilon)\right).
$$
This upper bound gives us the following bound on $H$.
$$
\abs{H} \le 2 \norm{\widehat{\Sigma}^{1/2}}_2^2 \left( 1+  O\left(\varepsilon L^2 \norm{\widehat{\Sigma}^{-1/2}}_2^2 \log(1/\varepsilon)\right) \right) + 2\norm{\mu}_2^2 
$$
Note that, $\norm{x}_2 \le L$ for any uncorrupted samples, and hence $\Sigma = \E[\mx \mx^\top] \preccurlyeq L^2 \cdot \Identity$.
Since $\widehat{\Sigma} \preccurlyeq (1+O(\varepsilon \log(1/\varepsilon))\Sigma$, $\norm{\widehat{\Sigma}^{1/2}}_2 \le \sqrt{1 + O(\varepsilon \log(1/\varepsilon))} L = O(L)$. Similarly, $\widehat{\Sigma} \succcurlyeq (1-O(\varepsilon \log(1/\varepsilon)) \Sigma$ gives us $\norm{\widehat{\Sigma}^{-1/2}}_2 \le (1-O(\varepsilon \log(1/\varepsilon))^{-1/2} \norm{\Sigma^{-1/2}}_2 =O(\xi^{-1/2} L)$. Furthermore, $\norm{\mu}_2 \le L$. Therefore, 
$$\abs{H} \le O\left(L^2\left(1 + \frac{\varepsilon \log(1/\varepsilon)  }{\xi}\right)\right) = O(L^2)$$
as $\xi \ge \varepsilon \log(1/\varepsilon)$. The rest of the proof is very similar to the proof of Lemma A.12 of \cite{ADKS22}. Let $\Delta = \frac{1}{N} \sum_{n \in \widehat{S}} \nabla_\theta \log \mathbb{P}_{\hat{\theta}}(o^n \mid \mx_n)^\top \frac{(\theta^\star - \hat{\theta})}{\norm{\theta^\star - \hat{\theta}}_2}$. Writing $F(\theta) = \frac{1}{N} \sum_{n \in \widehat{S}} \log \mathbb{P}_{\theta}(o^n \mid \mx_n)$, we get that there exists $\theta'$ such that
$$
F(\theta') \le F(\hat{\theta}) - \frac{\Delta^2}{2cL^2}
$$
for some constant $c > 0$.
Suppose $\norm{\theta'}$ is feasible. Then it must be that $\eta \ge \frac{\Delta^2}{2cL^2}$ as it is impossible to make improvement more than $\eta$. This implies that $\Delta \le cL \sqrt{\eta}$. On the other hand, if $\theta'$ is not a feasible solution, then we use the fact that $F(\cdot)$ is a concave function and obtain the following bound.
$$
F(\theta^\star) \le F(\hat{\theta}) + \nabla_\theta F(\hat{\theta})^\top (\theta^\star - \hat{\theta}) = F(\hat{\theta}) + \Delta \norm{\theta^\star - \hat{\theta}}_2
$$
Then it must be that $\eta \ge \Delta \norm{\theta^\star - \hat{\theta}}_2$ or $\Delta \le \frac{\eta}{\norm{\theta^\star - \hat{\theta}}_2}$. Combining the two results and after substituting $\eta = \varepsilon^2$ we get $\Delta \le \max \set{O (L \varepsilon), \frac{\varepsilon^2}{\norm{\theta^\star - \hat{\theta}}_2}}$. 
\end{proof}


\subsection{\texorpdfstring{Proof of ~\Cref{thm:reduction_to_offline_rl_uniform}}{Proof of Theorem~(3.3)}}
\begin{proof}
    By \Cref{lem:diff-norm-uniform-coverage} the reward estimate $\widehat{\theta}$ is $C_1 \frac{\varepsilon}{\zeta}e^{2L + \sqrt{\log\left( \frac{1}{2\delta \varepsilon}\right)}}$ close to the true parameter $\theta^\star$. 
Since algorithm \texttt{RobRL} returns at least $f(\varepsilon)$ optimal policy in terms of value function we are guaranteed that $V^\star(\widehat{\theta}) \ge {V}^{\tilde{\pi}}(\widehat{\theta}) \ge V^\star(\widehat{\theta}) - f(\varepsilon)$ for any $\theta$. Using this result we can lower bound $V^{\tilde{\pi}}(\theta^\star)$.
\begin{align*}
   &V^\star(\theta^\star) -  V^{\Tilde{\pi}}(\theta^\star) = V^\star(\theta^\star) - V^{\Tilde{\pi}}(\bar{\theta}) + V^{\Tilde{\pi}}(\bar{\theta}) - V^{\Tilde{\pi}}(\theta^\star) \\
   &\le f(\varepsilon) + V^\star(\theta^\star) - V^{\star}(\widehat{\theta}) + V^{\Tilde{\pi}}(\bar{\theta}) - V^{\Tilde{\pi}}(\theta^\star) 
\end{align*}
For the first difference, we use the fact that the optimal value function $V^\star(\cdot)$ is $\sqrt{Hd}$-Lipschitz in the reward parameter (lemma~\eqref{lem:value-function-lipshitzness}) and obtain the following bound.
$$
V^\star(\theta^\star) - V^{\star}(\widehat{\theta}) \le \sqrt{Hd} \norm{\theta^\star - \widehat{\theta}}_2 \le  \sqrt{Hd} C_1 \frac{\varepsilon}{\zeta} \exp \left( 2L + \sqrt{\log\left( \frac{1}{2\delta \varepsilon}\right)}\right)
$$
 Using lemma~\eqref{lem:diff-norm-uniform-coverage} the second difference can be bounded as follows.
\begin{align*}
&V^{\Tilde{\pi}}(\bar{\theta}) - V^{\Tilde{\pi}}(\theta^\star) = \sum_{h=1}^H \sum_{s,a} \Pro_{\tilde{\pi}}(s_h = s, a_h = a) \phi(s,a)^\top \left( \bar{\theta}_h - \theta^\star_h\right)\\
&\le \sum_{h=1}^H \sum_{s,a} \Pro_{\tilde{\pi}}(s_h = s, a_h = a) \norm{\phi(s,a)}_2 \norm{\bar{\theta}_h - \theta^\star_h}_2\\
&\le  \sum_{h=1}^H \norm{\bar{\theta}_h - \theta^\star_h}_2\\
&\le  \sqrt{H}\sqrt{\sum_{h=1}^H \norm{\bar{\theta}_h - \theta^\star_h}_2^2}\\
&=  \sqrt{H} \norm{\bar{\theta} - \theta^\star}_2^2\\
&\le  \sqrt{H} C_1 \frac{\varepsilon}{\xi} \cdot \exp \left( 2L + \sqrt{\log\left( \frac{1}{2\delta \varepsilon}\right)}\right)
\end{align*}

\end{proof}
\begin{lemma}\label{lem:diff-norm-uniform-coverage}
Suppose assumption~\eqref{asn:uniform_coverage} holds with $\xi \ge 5 \varepsilon$ and $N \ge \Omega\left(\frac{H^{3/2}}{\varepsilon^2}\left( d + \log(1/\delta) \right) \right)$. Then algorithm~\eqref{alg:alternating_optimization} returns $\htheta$, so that with probability at least $1-\delta$, we have
$$
\lVert\htheta - \theta^\star\rVert_2 \le C_1 \frac{\varepsilon}{\xi} \exp\left(2L + \sqrt{\log\left({1}/{2\delta \varepsilon} \right)}\right)
$$
\end{lemma}
\begin{proof}
    From \Cref{lem:apx-stationarity} we know that algorithm~\eqref{alg:alternating_optimization} computes a $\gamma = \max \set{cL\varepsilon, \frac{\varepsilon^2}{\norm{\theta^\star - \widehat{\theta}}_2}}$ stationary point for some constant $c > 0$. We can assume that $cL\varepsilon \ge \frac{\varepsilon^2}{\norm{\theta^\star - \widehat{\theta}}_2}$. Otherwise, $\norm{\theta^\star - \widehat{\theta}}_2 \le \varepsilon / (c\cdot L)$ and we are done.

    Let $T$ be the set of uncorrupted samples and $E$ be the set of corrupted samples. Then we can write down the stationarity condition~\eqref{eq:apx-stationarity} as follows.
\begin{align}
    \frac{1}{N} \sum_{n \in \widehat{S} \cap E} \nabla_\theta \log \mathbb{P}_{\widehat{\theta}}(o^n \mid \mathrm{x}_n)^\top \left( \htheta - \theta^\star \right) \le c L \varepsilon \cdot \norm{\htheta - \theta^\star}_2 - \frac{1}{N} \sum_{n \in \widehat{S} \cap T} \nabla_\theta \log \mathbb{P}_{\widehat{\theta}}(o^n \mid \mathrm{x}_n)^\top \left( \htheta - \theta^\star \right) \label{eq:lhs-rhs-sepration}
\end{align}
We first upper bound the term on the right.
\begin{align}
    -\frac{1}{N} \sum_{n \in \widehat{S} \cap T} \nabla_\theta \log \mathbb{P}_{\widehat{\theta}}(o^n \mid \mathrm{x}_n)^\top \left( \htheta - \theta^\star \right) &=  \underbrace{-\frac{1}{N} \sum_{n \in \widehat{S} \cap T} \nabla_\theta \log \mathbb{P}_{\theta^\star}(o^n \mid \mathrm{x}_n)^\top \left( \htheta - \theta^\star \right)}_{:=T_1}\nonumber \\
    &+ \underbrace{\frac{1}{N} \sum_{n \in \widehat{S} \cap T} \left(\nabla_\theta \log \mathbb{P}_{{\theta}^\star }(o^n \mid \mathrm{x}_n) - \nabla_\theta \log \mathbb{P}_{\widehat{\theta} }(o^n \mid \mathrm{x}_n) \right)^\top \left( \htheta - \theta^\star \right)}_{:= T_2    } \label{defn:terms-T1-T2}
\end{align}
Using the functional form of sigmoid link function i.e. $\Pro_{\theta}(o \mid \mathrm{x}) = \frac{1}{1 + \exp(-o \cdot \theta^\top \mathrm{x})}$, we get the following  expression for the term $T_1$.
\begin{align*}
    T_1 &= -\frac{1}{N} \sum_{n \in \widehat{S} \cap T} \frac{o^n}{1 + \exp(o^n \cdot \left \langle \theta^\star, \mathrm{x}_n \right \rangle)}\mathrm{x}_n^\top (\htheta - \theta^\star)
\end{align*}
In order to provide a high probability bound on $T_1$, we first provide a bound on the $k$-th moment of the random vector $X = \frac{o}{1 + \exp(o \cdot \left \langle \theta^\star, \mathrm{x} \right \rangle)}\mathrm{x}$. For any unit vector $v \in \R^d$ with $\norm{v}_2 = 1$ we have,
\begin{align*}
    \E\left[ \left( \frac{o}{1 + \exp(o \cdot \left \langle \theta^\star, \mathrm{x} \right \rangle)}\right)^k \left( \mathrm{x}^\top v\right)^k\right] &\le \sqrt{\E\left[  \frac{o^{2k}}{\left(1 + \exp(o \cdot \left \langle \theta^\star, \mathrm{x} \right \rangle)\right)^{2k}} \right] } \sqrt{\E\left[ (\mathrm{x}^\top v)^{2k}\right]}\\
    &\le\sqrt{\E\left[  \frac{1}{\left(1 + \exp(o \cdot \left \langle \theta^\star, \mathrm{x} \right \rangle)\right)^{2k}} \right] }  L^{k} \le L^{k}
\end{align*}
The second inequality uses the fact that $o \in \set{-1,1}$ and $\norm{\mathrm{x}}_2 \le L$. Since $\widehat{S} \cap T$ contains uncorrupted samples, and $\abs{\widehat{S} \cap T} \ge (1-2\varepsilon) N$ we can use Corollary G.1 from \cite{ZJS22} to obtain the following result with probability at least $1-\delta$.
\begin{align*}
    \norm{\E\left[ \frac{o}{1 + \exp(o \cdot \left \langle \theta^\star, \mathrm{x} \right \rangle)} \mathrm{x}\right] - \frac{1}{\abs{\widehat{S} \cap T}} \sum_{i \in \widehat{S} \cap T} \frac{o^n}{1 + \exp(o^n \cdot \left \langle \theta^\star, \mathrm{x}_n \right \rangle)}\mathrm{x}_n}_2 \le \frac{CkL}{1-2\varepsilon}\left(\frac{(2\varepsilon)^{1-1/k}}{\delta^{1/k}} + \frac{1}{\delta} \sqrt{\frac{d}{N}}\right)
\end{align*}
Now substituting $k = \sqrt{\log(\frac{1}{2\delta \varepsilon})}$ and assuming $N \ge d/\varepsilon^2$ we obtain the following result.
\begin{align*}
    \frac{1}{\abs{\widehat{S} \cap T}} \sum_{n \in \widehat{S} \cap T} \frac{o^n}{1 + \exp(o^n \cdot \left \langle \theta^\star, \mathrm{x}_n \right \rangle)}\mathrm{x}_n = \E\left[ \frac{o}{1 + \exp(o \cdot \left \langle \theta^\star, \mathrm{x} \right \rangle)} \mathrm{x}\right] + \Delta
\end{align*}
where
\begin{align*}
    \norm{\Delta}_2 \le \frac{ 4 \varepsilon C L}{1-2\varepsilon} \sqrt{\log\left(\frac{1}{2\delta \varepsilon}\right)} \left( \frac{1}{2\delta\varepsilon}\right)^{1/\sqrt{\log\left(\frac{1}{2\delta \varepsilon}\right)}} \le \frac{C_1 \varepsilon L}{1-2\varepsilon} e^{\sqrt{\log\left(\frac{1}{2\delta \varepsilon}\right)}}
\end{align*}
for some constant $C_1 > 0$. This lets us derive the following upper bound on $T_1$.
\begin{align*}
    T_1 &= -\frac{\abs{\widehat{S}\cap T}}{N}\left(\E\left[ \frac{o}{1 + \exp(o \cdot \left \langle \theta^\star, \mathrm{x} \right \rangle)} \mathrm{x}\right] + \Delta \right)^\top \left(\widehat{\theta} - \theta^\star \right)\\
    &= -\frac{\abs{\widehat{S}\cap T}}{N}\left(\E_{\mathrm{x},o}\left[ \nabla_\theta \log \Pro_{\theta^\star}(o\mid \mathrm{x})\right] + \Delta \right)^\top \left(\widehat{\theta} - \theta^\star \right)\\
    &= -\frac{\abs{\widehat{S}\cap T}}{N} \Delta^\top  \left(\widehat{\theta} - \theta^\star \right) \\
    &\le \frac{\abs{\widehat{S}\cap T}}{N} \norm{\Delta}_2 \norm{\widehat{\theta}- \theta^\star}_2\\
    &\le C_1 \varepsilon L \exp\left( \sqrt{\log\left( \frac{1}{2\delta \varepsilon}\right)}\right) \norm{\htheta - \theta^\star}_2
\end{align*}
The second equality uses that the fact $\theta^\star$ optimizes the population logistic loss and hence the derivative is zero. The last inequality uses that $\abs{\widehat{S} \cap T} \ge (1-2\varepsilon) N$.

We now bound the term $T_2$ defined in \cref{defn:terms-T1-T2}. We use assumption~\eqref{asn:uniform_coverage} to show that the function $\frac{1}{N} \sum_{n \in \widehat{S} \cap T} \nabla_\theta \log \Pro_{\theta}(o^n \mid \mathrm{x}_n)$ is strongly concave in $\theta$. Indeed from the definition of $\Pro_{\theta}(o \mid \mathrm{x})$ we have the following result.
\begin{align*}
  \frac{1}{N} \sum_{n \in \widehat{S} \cap T}   \nabla_\theta^2 \log \Pro_{\theta}(o^n \mid \mathrm{x}_n) &= \frac{1}{N} \sum_{n \in \widehat{S} \cap T} - \frac{\exp(o^n \left \langle \theta, \mathrm{x} \right \rangle)}{(1 + \exp(o^n \left \langle \theta, \mathrm{x} \right \rangle))^2} \mathrm{x}_n \mathrm{x}_n^\top \\
  &= - \frac{1}{N} \sum_{n \in \widehat{S} \cap T} \frac{1}{\left( \exp(-o^n \left \langle \theta, \mathrm{x} \right \rangle / 2) + \exp(o^n \left \langle \theta, \mathrm{x} \right \rangle / 2)\right)^2} \mathrm{x}_n \mathrm{x}_n^\top\\
  &\preccurlyeq - \frac{\exp(-2L)}{4N} \sum_{n \in \widehat{S} \cap T}  \mathrm{x}_n \mathrm{x}_n^\top\\
  &= - \frac{e^{-2L}}{4N} \left(\sum_{n = 1}^N \widetilde{\mathrm{x}}_n \widetilde{\mathrm{x}}_n^\top - \sum_{n \in \widetilde{D}_2\setminus ( \widehat{S} \cap T)} \widetilde{\mathrm{x}}_n \widetilde{\mathrm{x}}_n^\top \right)\\
  &\preccurlyeq -\frac{e^{-2L}}{4} \E\left[ \widetilde{\mathrm{x}} \widetilde{\mathrm{x}}^\top \right] + c_1 e^{-2L} L^2 \sqrt{\frac{d + \log(1/\delta)}{N}} \cdot \Identity_d + \frac{e^{-2L}}{4N} \sum_{n \in \widetilde{D}_2 \setminus ( \widehat{S} \cap T)} \widetilde{\mathrm{x}}_n \widetilde{\mathrm{x}}_n^\top
\end{align*}

The first inequality follows from the observation that $\abs{\left \langle \theta, \mx \right \rangle }\le L$ and $e^u + e^{-u} \le 2 \exp(L)$. The last inequality uses the concentration bound of a sample covariance matrix (lemma \ref{lem:concentration-of-covariance}). For the third term in the last upper bound, note that $\abs{\widehat{S} \cap T} \ge 1-2\varepsilon N$ and the $L_2$-norm of an original uncorrupted feature (i.e. $\widetilde{\mathrm{x}}_n$) is bounded by $L$. This implies that the last term is at most $ \varepsilon L \exp(-2L)/2$. Now using assumption~\eqref{asn:uniform_coverage} and choosing $N \ge \frac{4 c_1^2 L^3}{\varepsilon^2}\left( d + \log(1/\delta)\right)$ we obtain the following upper bound.
\begin{align*}
    \frac{1}{N} \sum_{n \in \widehat{S} \cap T}   \nabla_\theta^2 \log \Pro_{\theta}(o^n \mid \mathrm{x}_n) \preccurlyeq - \left( \frac{\xi}{4} - \varepsilon \right) L \exp(-2L) \cdot \Identity
\end{align*}
Therefore, we get the following upper bound.
\begin{align*}
    T_2 := \frac{1}{N} \sum_{n \in \widehat{S} \cap T} \left(\nabla_\theta \log \mathbb{P}_{{\theta}^\star }(o^n \mid \mathrm{x}_n) - \nabla_\theta \log \mathbb{P}_{\widehat{\theta} }(o^n \mid \mathrm{x}_n) \right)^\top \left( \htheta - \theta^\star \right) \le - \left( \frac{\xi}{4} - \varepsilon \right) L e^{-2L} \norm{\widehat{\theta} - \theta^\star}_2^2
\end{align*}
This gives us the following upper bound on the right hand side of \cref{eq:lhs-rhs-sepration}.
\begin{align}\label{eq:ubd}
     - \left( \frac{\xi}{4} - \varepsilon \right) L e^{-2L} \norm{\widehat{\theta} - \theta^\star}_2^2 + \left(2 +  C_1 \exp\left( \sqrt{\log\left(\frac{1}{2\delta \varepsilon} \right)}\right)\right) \varepsilon L \norm{\widehat{\theta}- \theta^\star}_2
\end{align}

We now provide a lower bound on the left hand side of \cref{eq:lhs-rhs-sepration}. From the definition of $\Pro_\theta(o \mid \mathrm{x})$ we obtain the following identity.
\begin{align}
&\frac{1}{N} \sum_{n \in \widehat{S} \cap E} \nabla_\theta \log \Pro_{\widehat{\theta}}(o^n \mid \mathrm{x}_n)^\top \left( \widehat{\theta} - \theta^\star \right) = \frac{1}{N} \sum_{n \in \widehat{S} \cap E} \frac{o^n}{1 + \exp\left( o^n \cdot \left \langle \htheta, \mathrm{x}_n\right \rangle\right)} \mathrm{x}_n^\top \left( \widehat{\theta} - \theta^\star \right)\nonumber \\
&= \frac{1}{N} \sum_{n \in \widehat{S} \cap E} \left( 1 - \frac{1}{1 + \exp\left( -o^n \cdot \left \langle \htheta, \mathrm{x}_n\right \rangle\right)} \right) o^n \cdot  \mathrm{x}_n^\top \left( \widehat{\theta} - \theta^\star \right) \nonumber \\
&= \frac{1}{N} \sum_{n \in \widehat{S} \cap E} \left( 1 - \Pro_{\htheta}(o^n \mid \mathrm{x}_n) \right) o^n \cdot  \mathrm{x}_n^\top \left( \widehat{\theta} - \theta^\star \right)\nonumber \\
&\ge -\frac{1}{N} \sqrt{\sum_{n \in \widehat{S} \cap E} (1-\Pro_{\widehat{\theta}}(o^n|\mx_n) o^n} \sqrt{\sum_{n \in \widehat{S} \cap E} (\mx_n^\top (\widehat{\theta} - \theta^\star))^2}\nonumber \\
&\ge -\sqrt{\varepsilon}\norm{\widehat{\theta} - \theta^\star}_2 \sqrt{\norm{\frac{1}{N}\sum_{n \in \widehat{S} \cap E} \mx_n \mx_n^\top}_2} \label{eq:last-lower-bound}
\end{align}

The last inequality uses $\abs{\widehat{S} \cap E} \le \varepsilon N$. We now use resilience property to bound the norm of the matrix $1/N \cdot \sum_{n \in \widehat{S} \cap E} \mx_n \mx_n^\top$. Recall that $\mx_n = \widehat{\Sigma}^{1/2} z_n$ where the set $\widetilde{\calD}_2 = \set{z_1,\ldots,z_N}$ is a stable set (\Cref{lem:stability-whitened-covariates}). This implies, 
$$\norm{\frac{1}{\abs{\widehat{S}}} \sum_{n \in \widehat{S}} z_n z_n^\top - \Identity}_2 \le O(\sigma^2 \varepsilon \norm{\widehat{\Sigma}^{-1/2}}_2^2\log(1/\varepsilon)) .$$
This also implies that the set $\widehat{S}$\footnote{With slight abuse of notation we write $\widehat{S} = \set{z_1,\ldots,z_{\abs{\widehat{S}}}}$ whereas it should be $\widehat{S} = \set{ \widehat{\Sigma}^{1/2}z_1,\ldots,\widehat{\Sigma}^{1/2}z_{\abs{\widehat{S}}}}.$ } satisfies the necessary conditions for Corollary D.3 of \cite{ADKS22}. Therefore,
\begin{align*}
&\norm{\frac{1}{N} \sum_{n \in \widehat{S} \cap E} \mx_n \mx_n^\top}_2 = \norm{\frac{1}{N} \sum_{n \in \widehat{S} \cap E} \widehat{\Sigma}^{1/2} z_n z_n^\top \widehat{\Sigma}^{1/2} }_2 \\
\le &\norm{\widehat{\Sigma}^{1/2}}_2^2 \norm{\frac{1}{N} \sum_{n \in \widehat{S} \cap E} z_n z_n^\top}_2 \le O\left(\sigma^2 \norm{\widehat{\Sigma}^{-1/2}}_2^2 \norm{\widehat{\Sigma}^{1/2}}_2^2 \varepsilon \log(1/\varepsilon) \right)
\end{align*}
Substituting this bound in \cref{eq:last-lower-bound} we obtain the following lower bound.
$$
\frac{1}{N} \sum_{n \in \widehat{S} \cap E} \nabla_\theta \log \Pro_{\widehat{\theta}}(o^n \mid \mathrm{x}_n)^\top \left( \widehat{\theta} - \theta^\star \right) \ge - \frac{L}{\xi} \varepsilon \sqrt{\log(1/\varepsilon)} \norm{\theta^\star - \widehat{\theta}}_2
$$
Now combining this lower bound with the upper bound established in \cref{eq:ubd} we can obtain a bound on $\norm{\htheta - \theta^\star}_2$.
\begin{align*}
    &-\frac{\varepsilon L}{\xi} \sqrt{\log(1/\varepsilon)} \norm{\htheta - \theta^\star}_2 \le - \left( \frac{\xi}{4} - \varepsilon \right) L e^{-2L} \norm{\widehat{\theta} - \theta^\star}_2^2 + \left(2 +  C_1 \exp\left( \sqrt{\log\left(\frac{1}{2\delta \varepsilon} \right)}\right)\right) \varepsilon L \norm{\widehat{\theta}- \theta^\star}_2\\
    &\Rightarrow \norm{\htheta - \theta^\star}_2 \le \frac{3/\xi + C_1 \exp\left( \sqrt{\log\left(\frac{1}{2\delta \varepsilon} \right)}\right)}{\xi/4 - \varepsilon} \cdot \varepsilon \cdot e^{2L}
\end{align*}
\end{proof}

\subsection{Stability Analysis}

\begin{theorem}
    [Theorem 5.5 of \cite{diakonikolas2025sos}]\label{thm:robust_covariance_estimation} 
    There exist absolute constants $c,c'<1/2$ such that the following holds. Let $\varepsilon\in(0,c)$ and let $P$ be an $s$-hypercontractive (see Definition 1.7 of \cite{diakonikolas2025sos}) sub-Gaussian distribution over $\mathbb{R}^d$ with mean $\mu$ and covariance $\Sigma$. Let $S$ be an $\varepsilon$-corrupted set of samples from $P$ with $|S|=m$. Fix any $t\in\mathbb{N}$ such that $t=2^j$, for some $j\in\mathbb{N}$, and let $st\varepsilon^{1-2/t}< c'$. If $m\geq \Omega(\textnormal{poly}(d^t,1/\varepsilon)$, there is an algorithm that (i) takes as input $S,t,\varepsilon,s$, (ii) runs in $(nd)^{\textnormal{poly}(t)}$ time, and (iii) outputs $\widehat{\Sigma}$ such that
    \begin{align*}
        (1-\delta)\Sigma \preceq \widehat{\Sigma} \preceq (1+\delta) \Sigma, \; \;\textnormal{for} \;\; \delta \leq st\varepsilon^{1-\frac{2}{t}}~,
    \end{align*}
    where $A\preceq B$ implies that $B-A$ is positive semi-definite.
\end{theorem}
Letting $t=\log(1/\varepsilon)$, the above result implies:
\begin{align}
    \norm{\widehat{\Sigma}-\Sigma}_2 \leq O\left( \varepsilon \log\left(\frac{1}{\varepsilon}\right)\right),
\end{align}
where $\norm{\cdot}_2$ denotes the operator norm for matrices. Let us denote by $\texttt{RobCovEst}$ the oracle which yields the result above. 

Let us recall the definition of stability from \cite{diakonikolas2023algorithmic} (Definition 2.1). 

\begin{definition}
    Fix $0<\varepsilon<1/2$ and $\delta \ge \varepsilon$. A finite set $S \subset \R^d$ is $(\varepsilon,\delta)$-stable with respect to a vector $\mu$ if for every unit vector $v \in \R^d$ and every $S' \subseteq S$ with $\abs{S'} \ge (1-\varepsilon) \abs{S}$, the following conditions hold:
    \begin{enumerate}
        \item $\abs{\frac{1}{\abs{S'}} \sum_{x \in S'} v^\top (x-\mu)} \le \delta$.
        \item $\abs{\frac{1}{\abs{S'}} \sum_{x \in S'} (v^\top (x-\mu))^2 - 1} \le \delta^2/\varepsilon$.
    \end{enumerate}
\end{definition}
Suppose $x_1,\ldots,x_N$ are drawn from a $\sigma$-sub-Gaussian distribution with mean $\mu$ and covariance $\Sigma$ with smallest eigenvalue at least  $\xi \ge \Omega(\varepsilon \log(1/\varepsilon))$. Let $\widehat{\Sigma}$ be an estimate of $\Sigma$ that satisfies the following bound.
\begin{equation}\label{eq:robust-covariance-estimation}
(1-c\cdot \varepsilon \log(1/\varepsilon)) \Sigma \preccurlyeq \widehat{\Sigma} \preccurlyeq (1+ c \cdot \varepsilon \log(1/\varepsilon))
\end{equation}
for some constant $c>0$. Then $\widehat{\Sigma}$ is invertible.
We construct the set of whitened covariates $S = \set{\widehat{\Sigma}^{-1/2}x_1,\ldots, \widehat{\Sigma}^{-1/2}x_N}$. Then we claim that as long as $N\ge O(\textrm{poly}(d)/\varepsilon^2)$, the set $S$ is $(\varepsilon, O(\varepsilon \sqrt{\log(1/\varepsilon)}))$-stable with respect to the vector $\widehat{\Sigma}^{-1/2} \mu$. The proof is similar to the proof of proposition 2.3 of \cite{diakonikolas2023algorithmic}. We provide the proof here for completeness.

\begin{lemma}\label{lem:stability-whitened-covariates}
    Suppose $x_1,\ldots,x_N$ are drawn from a $\sigma$-sub-Gaussian distribution with mean $\mu$ and covariance $\Sigma$. Let $\widehat{\Sigma}$ be an estimate of $\Sigma$ satisfying
    $$
    (1-c \cdot \varepsilon \log(1/\varepsilon)) \Sigma \preccurlyeq \widehat{\Sigma} \preccurlyeq (1 + c \cdot \varepsilon \log(1/\varepsilon)) \Sigma.
    $$
    If $\Sigma \succcurlyeq \xi \Identity \succcurlyeq O(\varepsilon \log(1/\varepsilon)) \Identity$ and $N \ge O\left( \frac{d^3 \xi}{\varepsilon^2} \right)$ then the set $S = \set{\widehat{\Sigma}^{-1/2} x_1, \ldots, \widehat{\Sigma}^{-1/2} x_N}$ is $(\varepsilon, \varepsilon \sigma  \norm{\widehat{\Sigma}^{-1/2}}_2 \sqrt{\log(1/\varepsilon)  })$-stable with respect to the vector $\widehat{\Sigma}^{-1/2} \mu$.
\end{lemma}
\begin{proof}
First, observe that $\E[\widehat{\Sigma}^{-1/2}x] = \widehat{\Sigma}^{-1/2} \mu$. $\text{Cov}(\widehat{\Sigma}^{-1/2}x) = \widehat{\Sigma}^{-1/2} \text{Cov}(X) \widehat{\Sigma}^{-1/2} = \widehat{\Sigma}^{-1/2} \Sigma \widehat{\Sigma}^{-1/2}$. This implies
$$
\frac{1}{1+c\cdot \varepsilon \log(1/\varepsilon)} \Identity \preccurlyeq \text{Cov}(\widehat{\Sigma}^{-1/2}X) \preccurlyeq \frac{1}{1-c\cdot \varepsilon \log(1/\varepsilon)} \Identity
$$
and for $\varepsilon < 1/2$ we have, $\norm{\text{Cov}(\widehat{\Sigma}^{-1/2}X) - \Identity}_2 \le c\cdot \varepsilon \log(1/\varepsilon)$.

We will write $x\sim S$ to denote the uniformly at random sampling from the set $S$. Note that,
\begin{equation}\label{eq:integral-mean}
\frac{1}{\abs{S}} \sum_{x \in S} v ^\top (x - \widehat\Sigma^{-1/2}\mu) = \int_{0}^\infty \Pro_{x \sim S}\left(v ^\top (x - \widehat\Sigma^{-1/2}\mu) > t\right) dt - \int_{-\infty}^{0} \Pro_{x \sim S}\left(v ^\top (x - \widehat\Sigma^{-1/2}\mu) < t\right) dt
\end{equation}
Each $x \in S$ is drawn from a sub-Gaussian distribution with parameter $\sigma \cdot \norm{\widehat{\Sigma}^{-1/2}}_2 $.
Therefore, with high probability $\norm{x_i-\widehat{\Sigma}^{-1/2}\mu}_2 \le \underbrace{O\left( {\sigma} \norm{\widehat{\Sigma}^{-1/2}}_2 \sqrt{d \log(dN)}\right)}_{:=\mathfrak{U}}$ for each $i \in [N]$.
This lets us bound the range of the variable $t$ in \cref{eq:integral-mean}.
\begin{equation}\label{eq:bdd-integral-mean}
\frac{1}{\abs{S}} \sum_{x \in S} v ^\top (x - \widehat\Sigma^{-1/2}\mu) = \int_{0}^{\mathfrak{U}} \Pro_{x \sim S}\left(v ^\top (x - \widehat\Sigma^{-1/2}\mu) > t\right) dt - \int_{-\mathfrak{U}}^{0} \Pro_{x \sim S}\left(v ^\top (x - \widehat\Sigma^{-1/2}\mu) < t\right) dt
\end{equation}
Now, we use two facts. First, for any subset $S' \subseteq S$ with $\abs{S'} \ge (1-\varepsilon) \abs{S}$ the following holds for any unit vector $v \in \R^d$ and $t \in \R$.
$$
\abs{\Pro_{x \sim S}(v^\top x > t) - \Pro_{x \sim S'}(v^\top x > t)} \le \min \set{\Pro_{x \sim S}(v^\top x > t), O(\varepsilon)}
$$
Second, by a standard application of VC-inequality, with high probability, for every unit vector $v \in \R^d$ and $t \in \R$, we have
$$
\abs{\Pro_{x \sim S}(v^\top x > t) - \Pro_{x \sim D}(v^\top x > t)} \le \eta
$$
as long as $N \ge O(d/\eta^2)$.  
\begin{align*}
   &\abs{\frac{1}{\abs{S}} \sum_{x \in S} v^\top (x - \widehat\Sigma^{-1/2}\mu) - \frac{1}{\abs{S'}} \sum_{x \in S'} v^\top (x - \widehat\Sigma^{-1/2}\mu)}\\
    \le &\int_{-\mathfrak{U}}^{\mathfrak{U}} \min\left\{ \Pro_{x \sim S}\left(v ^\top (x - \widehat\Sigma^{-1/2}\mu) > t\right), O(\varepsilon) \right\} dt\\
    \le &\int_{-\mathfrak{U}}^{\mathfrak{U}} \min\left\{ \Pro_{x \sim D}\left(v ^\top (x - \widehat\Sigma^{-1/2}\mu) > t\right) + O(\eta), O(\varepsilon) \right\} dt
\end{align*}
Now $v^\top X$ is $\sigma \norm{\widehat{\Sigma}^{-1/2}}_2$-subgaussian with mean $v^\top \widehat{\Sigma}^{-1/2}\mu$. Therefore, $\Pro_{x \sim D}\left( v^\top(x - \widehat{\Sigma}^{-1/2}\mu) > t\right) \le O(\exp(-t^2/(\sigma^2\norm{\widehat{\Sigma}^{-1/2}}_2^2))$. Substituting this bound above, we obtain the following upper bound.
\begin{align*}
    &\abs{\frac{1}{\abs{S}} \sum_{x \in S} v^\top (x - \widehat\Sigma^{-1/2}\mu) - \frac{1}{\abs{S'}} \sum_{x \in S'} v^\top (x - \widehat\Sigma^{-1/2}\mu)}\\
    \le &\int_{-\mathfrak{U}}^{\mathfrak{U}} \min\left\{ O\left(\exp\left(-\frac{t^2}{\sigma^2\norm{\widehat{\Sigma}^{-1/2}}_2^2}\right)\right) + O(\eta), O(\varepsilon) \right\} dt\\
    \le &O(\eta \cdot \mathfrak{U}) + O\left(\varepsilon \cdot \sigma \norm{\widehat{\Sigma}^{-1/2}}_2 \sqrt{ {\log(1/\varepsilon) } }\right) + \int_{\abs{t}\ge \sigma \norm{\widehat{\Sigma}^{-1/2}}_2 \sqrt{{\log(1/\varepsilon) } }} O\left(\exp\left(-\frac{t^2}{\sigma^2\norm{\widehat{\Sigma}^{-1/2}}_2^2}\right)\right) dt\\
    \le &O\left( \eta \cdot  {\sigma} \norm{\widehat{\Sigma}^{-1/2}}_2\sqrt{d \log(dN)}\right) + O\left(\varepsilon \cdot \sigma \norm{\widehat{\Sigma}^{-1/2}}_2 \sqrt{ {\log(1/\varepsilon) } }\right) 
\end{align*}
It can be verified that if $\eta < \varepsilon / d$ i.e. the number of samples $N \ge O(d^3/\varepsilon^2)$ the first term above is dominated by the second term, and we obtain the following upper bound.
\begin{align*}
    &\abs{  \frac{1}{\abs{S'}} \sum_{x \in S'} v^\top (x - \widehat\Sigma^{-1/2}\mu)} \le \abs{\frac{1}{\abs{S}} \sum_{x \in S} v^\top (x - \widehat\Sigma^{-1/2}\mu)} + O\left(\varepsilon \cdot \sigma \norm{\widehat{\Sigma}^{-1/2}}_2 \sqrt{ {\log(1/\varepsilon) } }\right) 
\end{align*}
For a $\sigma$-SubGaussian distribution, with high probability
$$
\abs{\frac{1}{\abs{S}} \sum_{x \in S} v^\top (x - \widehat\Sigma^{-1/2}\mu)} \le \norm{\frac{1}{\abs{S}} \sum_{x \in S} x - \widehat{\Sigma}^{-1/2} \mu }_2 \le O\left(\sigma \norm{\widehat{\Sigma}^{-1/2}}_2 \sqrt{\frac{\log d}{N}}\right)
$$
Therefore, as long as $N \ge O\left( \max\set{d^3/\varepsilon^2, \log d/\log(1/\varepsilon)}\right)$ we obtain,
\begin{align*}
    \abs{\frac{1}{\abs{S'}} \sum_{x \in S'} v^\top (x - \widehat\Sigma^{-1/2}\mu)} &\le O\left(\varepsilon \cdot \sigma \norm{\widehat{\Sigma}^{-1/2}}_2 \sqrt{ {\log(1/\varepsilon) } }\right).
\end{align*}

We now turn to proving the second inequality for stability. Note that,
$$
\frac{1}{\abs{S}}\sum_{x \in S} (v^\top (x - \widehat{\Sigma}^{-1/2}\mu))^2 = \int_0^\infty 2 t \cdot \Pro_{x \sim S}\left( \abs{v^\top(x - \widehat{\Sigma}^{-1/2}\mu)} > t\right) dt
$$
Now following an argument same as earlier, we can establish the following bound for a $(1-\varepsilon)$-dense subset $S'$ of $S$.
\begin{align*}
   &\abs{\frac{1}{\abs{S}} \sum_{x \in S} \left( v^\top (x - \widehat\Sigma^{-1/2}\mu) \right)^2 - \frac{1}{\abs{S'}} \sum_{x \in S'} \left( v^\top (x - \widehat\Sigma^{-1/2}\mu)\right)^2 }\\
    \le &\int_{0}^{\mathfrak{U}} 2t \cdot \min\left\{ \Pro_{x \sim S}\left(\abs{v ^\top (x - \widehat\Sigma^{-1/2}\mu)} > t\right), O(\varepsilon) \right\} dt\\
    \le &\int_{0}^{\mathfrak{U}} 2t \cdot \min\left\{ \Pro_{x \sim D}\left(\abs{v ^\top (x - \widehat\Sigma^{-1/2}\mu)} > t\right) + O(\eta), O(\varepsilon) \right\} dt\\
    \le &O(\eta \cdot \mathfrak{U}^2) + \int_{0}^{\mathfrak{U}} 2t \cdot \min\left\{ O\left(\exp\left(-\frac{t^2}{\sigma^2\norm{\widehat{\Sigma}^{-1/2}}_2^2}\right)\right), O(\varepsilon) \right\} dt\\
    \le &O(\eta \cdot \mathfrak{U}^2) + \int_0^{\sigma  \norm{\widehat{\Sigma}^{-1/2}}_2  \sqrt{ {\log(1/\varepsilon)} }} 2t\cdot O(\varepsilon) dt + \int_{\sigma \norm{\widehat{\Sigma}^{-1/2}}_2 \sqrt{ {\log(1/\varepsilon)} }}^{\mathfrak{U}} 2t \cdot O(\exp(-t^2  /\sigma^2 \norm{\widehat{\Sigma}^{-1/2}}_2^2 )) dt\\
    \le &O(\eta \cdot \mathfrak{U}^2) + O\left( \varepsilon \cdot \norm{\widehat{\Sigma}^{-1/2}}_2^2 \cdot {\sigma^2 \log(1/\varepsilon)} \right) \\
\end{align*}
For the choice of $\eta < \varepsilon/(d\log d)$ it can be checked that the upper bound is at most $O(\varepsilon \sigma^2 \norm{\widehat{\Sigma}^{-1/2}}_2^2\cdot \log(1/\varepsilon)  )$. Therefore, we have established that,
$$
\abs{\frac{1}{\abs{S'}} \sum_{x \in S'} \left( v^\top (x - \widehat\Sigma^{-1/2}\mu)\right)^2 - 1} \le \abs{\frac{1}{\abs{S}} \sum_{x \in S} \left( v^\top (x - \widehat\Sigma^{-1/2}\mu) \right)^2 - 1} + O\left( \varepsilon \sigma^2 \norm{\widehat{\Sigma}^{-1/2}}_2^2\cdot \log(1/\varepsilon) \right).
$$
We now bound the second term above. Since $v$ is a unit vector $v^\top v = 1$.
\begin{align*}
    \abs{\frac{1}{\abs{S}} \sum_{x \in S} \left( v^\top (x - \widehat\Sigma^{-1/2}\mu) \right)^2 - 1} = &\abs{v^\top \left( \frac{1}{\abs{S}} \sum_{x \in S} (x - \widehat{\Sigma}^{-1/2}\mu) (x-\widehat{\Sigma}^{-1/2}\mu)^\top - \Identity\right) v}\\
    \le &\norm{\frac{1}{\abs{S}} \sum_{x \in S} (x - \widehat{\Sigma}^{-1/2}\mu) (x-\widehat{\Sigma}^{-1/2}\mu)^\top - \Identity}_2\\
\end{align*}
By Theorem 4.7.1 of \cite{Vershynin18}, the error in estimation of covariance matrix can be bounded with high probability, i.e.
$$
\norm{\frac{1}{\abs{S}} \sum_{x \in S} (x - \widehat{\Sigma}^{-1/2}\mu) (x-\widehat{\Sigma}^{-1/2}\mu)^\top - \textrm{Cov}(\widehat{\Sigma}^{-1/2}X)}_2 \le O\left( \sqrt{\frac{d}{N}} \norm{\textrm{Cov}(\widehat{\Sigma}^{-1/2}X)}_2\right)
$$
holds with high probability. As $\text{Cov}(\widehat{\Sigma}^{-1/2}X) = \widehat{\Sigma}^{-1/2} \Sigma \widehat{\Sigma}^{-1/2} \succcurlyeq (1+c_1 \cdot \varepsilon \log(1/\varepsilon))^{-1} \Identity$, as long as $N \ge d /\varepsilon^2$ we have,
$$
\norm{\frac{1}{\abs{S}} \sum_{x \in S} (x - \widehat{\Sigma}^{-1/2}\mu) (x-\widehat{\Sigma}^{-1/2}\mu)^\top - \widehat{\Sigma}^{-1/2} \Sigma \widehat{\Sigma}^{-1/2}}_2 \le O(\varepsilon \log(1/\varepsilon)).
$$
Substituting this upper bound gives us,
\begin{align*}
    \abs{\frac{1}{\abs{S'}} \sum_{x \in S'} \left( v^\top (x - \widehat\Sigma^{-1/2}\mu)\right)^2 - 1} \le &O(\sigma^2 \varepsilon \log(1/\varepsilon)/\xi) + \norm{\Identity - \widehat{\Sigma}^{-1/2} \Sigma \widehat{\Sigma}^{-1/2}}_2 \\
    \le &O\left(\sigma^2 \varepsilon \norm{\widehat{\Sigma}^{-1/2}}_2^2\log(1/\varepsilon) \right).
\end{align*}
\end{proof}

\section{\texorpdfstring{Missing Proofs from Section~\ref{sec:low-relative-condition-number}}{Missing Proofs from Section~(4)}}
Here we state a more general version of \Cref{lem:diff-log-likelihood}. Let us write $\mathbb{P}_\theta(o \mid \mx) = \frac{1}{1 + \exp(-o\cdot \mx^\top \theta)}$. We will also write ${\theta}_N^\star$ to denote the parameter that maximizes empirical log-likelihood i.e.
$$
{\theta}_N^\star \in \argmax_{\theta: \norm{\theta}_2 \le 1} \frac{1}{N}\sum_{n} \log \mathbb{P}_{\theta}(o^n \mid x_n)
$$

\begin{lemma}\label{lem:diff-log-likelihood-new}
Suppose that  $\norm{\theta}_2 \le B$ for any $\theta  \in \Theta_B$, $\norm{\phi(\tau)}_2 \le L$ for any trajectory $\tau \in \calT$, and $\log \mathbb{P}_\theta(\cdot)$ is a concave function of $\theta$. Then   with probability at least $1-\delta$, we have
    $$
    \frac{1}{N} \sum_{n=1}^N \log \left( \frac{\mathbb{P}_{\tilde{\theta}}(o^n \mid \mx_n)}{\mathbb{P}_{\theta^\star}(o^n \mid \mx_n)}\right) \le 6\varepsilon LB + c \cdot \frac{d}{N} \log\left( \frac{LN}{\delta}\right)
    $$
    for $\tilde{\theta} = \widehat{\theta}$ or $\theta^\star_N$. Here $c > 0$ is a universal constant.
\end{lemma}
\begin{proof}
First note that we can express the difference in log-likelihood as follows.
\begin{align}
    &\frac{1}{N} \sum_{n=1}^N \log  {\mathbb{P}_{\hat{\theta}}(o^n \mid \mx_n)} - \log {\mathbb{P}_{\theta^\star}(o^n \mid \mx_n)} \nonumber \\
    &= \frac{1}{N} \sum_{n=1}^N \log \left( \frac{\mathbb{P}_{\hat{\theta}}(o^n \mid \mx_n)}{\mathbb{P}_{\theta^\star_N}(o^n \mid \mx_n)}\right) + \frac{1}{N} \sum_{n=1}^N \log \left( \frac{\mathbb{P}_{{\theta}^\star_N}(o^n \mid \mx_n)}{\mathbb{P}_{\theta^\star}(o^n \mid \mx_n)}\right) \label{eq:likelihood-decomposition}
\end{align}
For linear reward functions, we can use Lemma 1 of \cite{zhan2023provable} to bound the second term. Let $T \subseteq [N]$ be the set of corrupted data points. Then we have,
\begin{align*}
    &\frac{1}{N} \sum_{n=1}^N \log \left( \frac{\mathbb{P}_{{\theta}^\star_N}(o^n \mid \mx_n)}{\mathbb{P}_{\theta^\star}(o^n \mid \mx_n)}\right)\\
    &= \frac{1}{N} \sum_{n \in T} \log \left( \frac{\mathbb{P}_{{\theta}^\star_N}(o^n \mid \mx_n)}{\mathbb{P}_{\theta^\star}(o^n \mid \mx_n)}\right) + \frac{1}{N} \sum_{n \notin T} \log \left( \frac{\mathbb{P}_{{\theta}^\star_N}(o^n \mid \mx_n)}{\mathbb{P}_{\theta^\star}(o^n \mid \mx_n)}\right)\\
    &\le \varepsilon \cdot \log \left( \frac{1 + e^{LB}}{1 - e^{-LB}}\right) + O\left( \frac{d}{(1-\varepsilon) N} \log\left( \frac{LN}{\delta}\right)\right)\\
    &\le 2\varepsilon LB +  O\left( \frac{d}{ N} \log\left( \frac{LN}{\delta}\right)\right)
\end{align*}
The first inequality uses Lemma~1 of \cite{zhan2023provable} and $\abs{T} \le \varepsilon N$. Now, we consider bounding the first term in \cref{eq:likelihood-decomposition}.
\begin{align}
&\frac{1}{N} \sum_{n=1}^N \log \left( \frac{\mathbb{P}_{\hat{\theta}}(o^n \mid \mx_n)}{\mathbb{P}_{\theta^\star_N}(o^n \mid \mx_n)}\right) \\&= \frac{1}{N} \sum_{n\notin \widehat{S}} \log \left( \frac{\mathbb{P}_{\hat{\theta}}(o^n \mid \mx_n)}{\mathbb{P}_{\theta^\star_N}(o^n \mid \mx_n)}\right)  + \frac{1}{N} \sum_{n \in \widehat{S}} \log \left( \frac{\mathbb{P}_{\hat{\theta}}(o^n \mid \mx_n)}{\mathbb{P}_{\theta^\star_N}(o^n \mid \mx_n)}\right) \nonumber \\
&\le \varepsilon \cdot \log \left( \frac{1 + e^{LB}}{1 - e^{-LB}}\right) + \frac{1}{N} \sum_{n\in \widehat{S}} \log \left( \frac{\mathbb{P}_{\hat{\theta}}(o^n \mid \mx_n)}{\mathbb{P}_{\theta^\star_N}(o^n \mid \mx_n)}\right)\nonumber\\
&\le 2 \varepsilon LB + \frac{1}{N} \sum_{n\in \widehat{S}} \log \left( \frac{\mathbb{P}_{\hat{\theta}}(o^n \mid x_n)}{\mathbb{P}_{\theta^\star_N}(o^n \mid \mx_n)}\right)\nonumber\\
&\le 2 \varepsilon LB + \frac{1}{N} \sum_{n \in \widehat{S}} \nabla_\theta \log \mathbb{P}_{\hat{\theta}}(o^n \mid \mx_n)^\top {(\theta^\star - \hat{\theta})}\nonumber\\
&\le 2 \varepsilon LB + \gamma \norm{\theta^\star_N - \widehat{\theta}}_2 \label{eq:upper-bound-likelihood-ratio}
\end{align}
The first inequality uses that the size of $\widehat{S}$ is $(1-\varepsilon)N$ and the inner product between the parameter and the feature is bounded by $LB$. The second inequality uses that $\log \mathbb{P}_\theta(o^n \mid \mx_n)$ is a concave function in $\theta$.

Lemma~\ref{lem:apx-stationarity} shows that $\gamma \le \max \set{2L \varepsilon, \frac{\varepsilon^2}{\norm{\theta^\star_N - \hat{\theta}}_2}}$. Substituting this upper bound in \cref{eq:upper-bound-likelihood-ratio} and using $\norm{\theta^\star_N}_2, \norm{\hat{\theta}}_2 \le B$ we get the following result:   $\frac{1}{N} \sum_{n=1}^N \log \left( \frac{\mathbb{P}_{\hat{\theta}}(o^n \mid \mx_n)}{\mathbb{P}_{\theta^\star_N}(o^n \mid \mx_n)}\right)  \le \max \set{4\varepsilon LB, 2\varepsilon LB + \varepsilon^2} \le 4\varepsilon LB$.
\end{proof}

\subsection{\texorpdfstring{Proof of \Cref{thm:relative-cond-number-bound}}{Proof of Theorem~(4.3)}}
\begin{proof}
    Given a reward parameter $\theta$ let $V^\star(\theta) = \max_\pi V^\pi(\theta)$ be the optimal value function with reward parameter $\theta$. We claim that $V^\star(\cdot)$ is a convex function. In order to see this, given a policy $\pi$ let $d$ be the corresponding occupancy measure i.e. $d_h(s,a) = \Pro_\pi(s_h = s, a_h = a)$. Then we can write the value function as $V^\pi(\theta) = \sum_{h,s,a} = d_h(s,a) \phi(s,a)^\top \theta = d^\top \Phi \theta$. This observation implies the following inequality.
    \begin{equation}\label{eq:ineq-le-1}
     \max_\pi V^\pi(\theta) \le  \max_d d^\top \Phi \theta
    \end{equation}
    On the other hand, given an occupancy measure $d$ one can consider the following policy.
    $$
    \pi^d_h(s,a) = \left\{ \begin{array}{cc}
        \frac{d_h(s,a)}{\sum_b d_h(s,b)} &   \textrm{ if } \sum_b d_h(s,b) > 0\\
        \frac{1}{A} & \textrm{o.w.} 
    \end{array}\right.
    $$
    Moreover, it is known that occupancy measure induced by $\pi^d = (\pi^d_1,\ldots,\pi^d_H)$ is $d$. This implies the following inequality.
    \begin{equation}\label{eq:ineq-ge-1}
     \max_\pi V^\pi(\theta) \ge  \max_d d^\top \Phi \theta
    \end{equation}
    Therefore, from equations \eqref{eq:ineq-ge-1} and \eqref{eq:ineq-le-1} we conclude that
    $$
    V^\star(\theta) =  \max_\pi V^\pi(\theta) =  \max_d d^\top \Phi \theta
    $$
    Since $V^\star(\cdot)$ is a maximum of linear functions, it is a convex function. Moreover, by lemma~\eqref{lem:value-function-lipshitzness} $V^\star(\cdot)$ is $\sqrt{Hd}$-Lipschitz. By a similar argument the function $\calR(\theta) = \E_{\tau \sim \mu_{\textrm{ref}}}\left[ \phi(\tau)^\top \theta\right]$ is $\sqrt{Hd}$-Lipshitz in $\theta$. Therefore, $V^\star(\cdot) - \calR(\cdot)$ is $2 \sqrt{Hd}$-Lipschitz function.
    
    Now observe that, algorithm~\eqref{alg:robust_freehand_unknown_P} performs a projected sub-gradient descent of the function $V^\star(\cdot) - \calR(\cdot)$ with biased zero oracle calls. In particular, since \texttt{RobRL} returns a $f(\varepsilon)$-robust estimate of the optimal value function, we are guaranteed that $\abs{\widehat{V}(\theta) - V^\star(\theta)} \le f(\varepsilon)$. Therefore, we can apply the result of \Cref{thm:proj-subgradient-zero-order} to obtain the following bound.
    $$
    V^\star(\bar{\theta}) - \calR(\bar{\theta}) - \min_\theta \left(V^\star(\theta) - \calR(\theta)\right) \le 5 \sqrt{2 f(\varepsilon)} (Hd)^{1/4}
    $$
    Note that in order to apply theorem \Cref{thm:proj-subgradient-zero-order}, we need a lower bound on the number of iterations ($T$) and the number of calls to zero-order oracle ($K$) per iteration. For linear MDP we have the maximum norm of the parameter, $D \le \sqrt{H d}$ and maximum value of the function $M \le H\sqrt{d}$. This implies the following lower bound on the number of samples.
    $$
    N \ge T \cdot K \ge \tilde{\Omega} \left( \frac{MD}{\varepsilon} \frac{M^2 d^3}{\varepsilon^2}\right) = \tilde{\Omega}\left( \frac{H^{3/2}d^5   }{\varepsilon^3}\right)
    $$
    Since $\tilde{\pi}$ is $f(\varepsilon)$-approximately optimal with respect to the reward parameter $\Bar{\theta}$ we are guaranteed that,
    \begin{equation}
    V^\star(\bar{\theta}) - \calR(\bar{\theta}) - f(\varepsilon) \le V^{\tilde{\pi}}(\bar{\theta}) - \calR(\bar{\theta}) \le V^\star(\bar{\theta}) - \calR(\bar{\theta})  + 8 \sqrt{f(\varepsilon)} (Hd)^{1/4}.
    \end{equation}
    Now using lemma~\eqref{lem:min-max-to-max-min} (i.e. $\min_\theta \max_\pi V^\pi(\theta) - \calR(\theta) = \max_\pi \min_\theta V^\pi(\theta)- \calR(\theta) $ for linear reward models) we obtain the following inequality.
    \begin{align}
    \max_\pi \min_\theta \left( V^\pi(\theta) - \calR(\theta) \right) - f(\varepsilon) &= \min_\theta \max_\pi \left( V^\pi(\theta) - \calR(\theta) \right) \le V^{\tilde{\pi}}(\bar{\theta}) \le V^\star(\bar{\theta}) - \calR(\bar{\theta}) - f(\varepsilon) \nonumber \\
    &\le V^{\tilde{\pi}}(\bar{\theta}) - \calR(\bar{\theta}) - f(\varepsilon) \nonumber \\
    &\le \max_\pi \min_\theta \left( V^\pi(\theta) - \calR(\theta)  \right) + 8 \sqrt{f(\varepsilon)} (Hd)^{1/4} \label{eq:tilde-pi-approximation}
    \end{align}
    We claim that this implies that $\tilde{\pi}$ approximately optimizes the objective $\max_\pi \min_\theta V^\pi(\theta) - \calR(\theta)$ i.e.
    \begin{equation}
        \min_{\theta} \left(V^{\tilde{\pi}}(\theta) - \calR(\theta) \right) \ge \max_\pi \min_\theta \left( V^\pi(\theta) - \calR(\theta) \right) - f(\varepsilon) - 8 \sqrt{f(\varepsilon)}(Hd)^{1/4}
    \end{equation}
    Let $(\pi^\star, \theta^\star)$ be an optimal solution of the optimization problem $\max_\pi \min_\theta V^\pi(\theta) - \calR(\theta)$. Then the observation above follows from the following set of inequalities.
    \begin{align*}
        &\min_{\theta}\left( V^{\tilde{\pi}}(\theta) - \calR(\theta) \right) - \max_{\pi} \min_{\theta} \left( V^\pi(\theta) - \calR(\theta) \right) \\
        &=  \min_{\theta}\left( V^{\tilde{\pi}}(\theta) - \calR(\theta) \right) - \min_{\theta} \left( V^{\pi^\star}(\theta) - \calR(\theta) \right)\\
        &\ge -\min_{\theta}\abs{V^{\tilde{\pi}}(\theta) - V^{\pi^\star}(\theta) }\\
        &= -\min_{\theta}\abs{\left(V^{\tilde{\pi}}(\theta) -\calR(\theta) \right) - \left( V^{\pi^\star}(\theta) - \calR(\theta) \right) }\\
        &\ge -\underbrace{\min_{\theta}\abs{\left( V^{\tilde{\pi}}(\theta) - \calR(\theta) \right) - \left(V^{\tilde{\pi}}(\bar{\theta}) - \calR(\bar{\theta}) \right) } }_{:= T_1} - \underbrace{\min_{\theta}\abs{\left( V^{\tilde{\pi}}(\bar{\theta}) - \calR(\bar{\theta}) \right) - \left( V^{\pi^\star}(\theta) - \calR(\theta) \right) } }_{:= T_2} \\
        &\ge -\abs{\left( V^{\tilde{\pi}}(\bar{\theta}) - \calR(\bar{\theta}) \right) - \left( V^{\pi^\star}(\theta^\star) - \calR(\theta^\star) \right) }\\
        &\ge - f(\varepsilon) - 8 \sqrt{f(\varepsilon)}(Hd)^{1/4}
    \end{align*}
    The first inequality follows since $\min_{\theta} V^{\pi^\star}(\theta) \le \min_\theta \abs{V^{\pi^\star}(\theta) - V^{\tilde{\pi}}(\theta) } + V^{\tilde{\pi}}(\theta) \le \min_\theta \abs{V^{\pi^\star}(\theta) - V^{\tilde{\pi}}(\theta) } + \min_\theta V^{\tilde{\pi}}(\theta)$. The third inequality follows by substituting $\theta = \bar{\theta}$ in the term $T_1$ and $\theta = \bar{\theta}$ in the term $T_2$. Finally, the last inequality uses \cref{eq:tilde-pi-approximation}. Now we can apply lemma~\eqref{lem:diff-log-likelihood} with $\eta = f(\varepsilon) + 8 \sqrt{f(\varepsilon)}(Hd)^{1/4}$ to complete the proof.
\end{proof}

\subsection{\texorpdfstring{Proof of ~\Cref{prop:bound-finite-relative-cond-number}}{Proof of Proposition~(4.5)}}
\begin{proof}
    For linear MDP, the parameter $\theta = [\theta_1; \theta_2; \ldots; \theta_H]$ and the feature of a trajectory $\tau$ is constructed by concatenating the features of $H$ state, action pairs. Therefore, $\norm{\theta}_2 \le \sqrt{Hd}$ and $\norm{\phi(\tau)}_2 \le \sqrt{H}$ for any trajectory $\tau$. So we substitute $L = \sqrt{H}$,$B=\sqrt{Hd}$, and $M \le LB = H \sqrt{d}$. 

    We will use \texttt{R-LSVI} from \cite{zhang2022corruption} as the corruption robust offline RL oracle \texttt{RobRL}. Note that if $N \ge \Omega(H \cdot \textrm{poly}(d) / \varepsilon)$ we have $f(\varepsilon) \le \tilde{O}\left( H^2 d \sqrt{\alpha \varepsilon}\right)$. Now using the upper bound provided in theorem~\eqref{thm:relative-cond-number-bound} we obtain the following bound. 

\begin{align*}
    &V^\star(\theta^\star) - V^{\tilde{\pi}}(\theta^\star) \\
    &\le O\left(  \kappa \sqrt{\alpha} \left( \sqrt{\varepsilon H} d^{1/4}  + \sqrt{\frac{d}{N} \log \left( \frac{HdN}{\delta}\right)}\right) \right)\\
    &+\tilde{O}(H^2 d \kappa \sqrt{\alpha \varepsilon}) + \tilde{O}\left( H^{5/4} d^{3/4} (\alpha \varepsilon)^{1/4}\right)
\end{align*}
Now observe that if $N \ge \Omega(H \cdot \textrm{poly}(d) / \varepsilon)$ the term $\tilde{O}(\sqrt{d/N})$ can be bounded by $O(\sqrt{\varepsilon})$. Finally, we need a lower bound of $N \ge \tilde{\Omega}\left( \frac{H^{3/2}d^5   }{\varepsilon^3}\right)$ in order to apply theorem~\eqref{thm:relative-cond-number-bound}.
\end{proof}
\begin{lemma}\label{lem:diff-policy-value}
Suppose assumption~\eqref{asn:rel_cond_number} holds, and $\sup_{p \in [0,1]} \abs{\frac{d\Phi^{-1}(p)}{dp}} \le \kappa$. Let $\pi$ be a policy so that $$\min_{\theta \in \Theta(\widehat{\calD}_1)} \left( V^{\tilde{\pi}}(\theta) - \E_{\tau \sim \mu_{\textrm{ref}}} \left[ \phi(\tau)^\top \theta \right]\right)\ge \max_\pi \min_{\theta \in \Theta(\widehat{\calD}_1)} \left( V^{{\pi}}(\theta) - \E_{\tau \sim \mu_{\textrm{ref}}} \left[ \phi(\tau)^\top \theta \right]\right) -\eta $$
then for any target policy $\pi^\dagger$, with probability at least $1-\delta$, we have
$$
V^{\pi^\dagger}(\theta^\star) - V^{\tilde{\pi}}(\theta^\star) \le c \kappa \sqrt{\alpha} \left( \sqrt{\varepsilon H} d^{1/4} + \sqrt{\frac{d}{N} \log \left( \frac{HdN}{\delta}\right)}\right) + \eta.
$$
\end{lemma}
\begin{proof}
    The proof follows a similar approach to the proof of theorem 1 in \cite{zhan2023provable}, except for the fact that we need to account for the approximation error $\eta$ and corrupted dataset. We will write $\calR(\theta) = \E_{\tau \sim \mu_{\textrm{ref}}} \left[ \phi(\tau)^\top \theta \right]$. Moreover, let $\theta^\dagger \in \argmin_{\theta \in \Theta(\widehat{\calD}_1)} V^{\pi^\dagger}(\theta) - \calR(\theta)$.
    \begin{align*}
        V^{\pi^\dagger}(\theta^\star) - V^{\tilde{\pi}}(\theta^\star) &= \left(V^{\pi^\dagger}(\theta^\star) - \calR(\theta^\star) \right) - \left(V^{\tilde{\pi}}(\theta^\star) - \calR(\theta^\star) \right)\\
        &\le \left(V^{\pi^\dagger}(\theta^\star) - \calR(\theta^\star) \right) - \left(V^{\pi^\dagger}(\theta^\dagger) - \calR(\theta^\dagger) \right) + \eta\\
        &= \E_{\stackrel{\tau \sim \mu^{\pi^\dagger}}{ \tau_0 \sim \mu_{\textrm{ref}} }} \left[(\phi(\tau) - \phi(\tau_0))^\top(\theta^\star - \theta^\dagger) \right]\\
        &\le \E_{\stackrel{\tau \sim \mu^{\pi^\dagger}}{ \tau_0 \sim \mu_{\textrm{ref}} }} \left[\abs{(\phi(\tau) - \phi(\tau_0))^\top(\theta^\star - \theta^\dagger) } \right]+ \eta\\
        &\le \sqrt{\E_{\stackrel{\tau \sim \mu^{\pi^\dagger}}{ \tau_0 \sim \mu_{\textrm{ref}} }} \left[(\theta^\star - \theta^\dagger)^\top (\phi(\tau) - \phi(\tau_0)) (\phi(\tau) - \phi(\tau_0))^\top(\theta^\star - \theta^\dagger)  \right]}+ \eta\\
        &\le \sqrt{\alpha}\sqrt{\E_{\stackrel{\tau_0 \sim \mu_0}{ \tau_1 \sim \mu_1 }} \left[(\theta^\star - \theta^\dagger)^\top (\phi(\tau_0) - \phi(\tau_1)) (\phi(\tau_0) - \phi(\tau_1))^\top(\theta^\star - \theta^\dagger)  \right]}+ \eta\\
        &= \sqrt{\alpha}\sqrt{\E_{\stackrel{\tau_0 \sim \mu_0}{ \tau_1 \sim \mu_1 }} \left[\abs{ (\theta^\star - \theta^\dagger)^\top (\phi(\tau_0) - \phi(\tau_1)) }^2 \right]}+ \eta\\
        &= \sqrt{\alpha}\sqrt{\E_{\stackrel{\tau_0 \sim \mu_0}{ \tau_1 \sim \mu_1 }} \left[\abs{ \Phi^{-1}\left( P_{\theta^\star}(o=1\mid \tau_1, \tau_0) \right) - \Phi^{-1}\left( P_{\theta^\dagger}(o=1\mid \tau_1, \tau_0)\right) }^2 \right]}+ \eta\\
        &\le \sqrt{\alpha} \kappa \sqrt{\E_{\stackrel{\tau_0 \sim \mu_0}{ \tau_1 \sim \mu_1 }} \left[\abs{  P_{\theta^\star}(o=1\mid \tau_1, \tau_0)  - P_{\theta^\dagger}(o=1\mid \tau_1, \tau_0) }^2 \right]}+ \eta\\
        &= \frac{\sqrt{\alpha} \kappa   }{\sqrt{2}}\sqrt{\E_{\stackrel{\tau_0 \sim \mu_0}{ \tau_1 \sim \mu_1 }} \left[\norm{  P_{\theta^\star}(\cdot \mid \tau_1, \tau_0)  - P_{\theta^\dagger}(\cdot \mid \tau_1, \tau_0) }^2 \right]}+ \eta
    \end{align*}
    The first inequality follows from the following observation -- $V^{\tilde{\pi}}(\theta^\star) - \calR(\theta^\star) \ge \min_{\theta \in \Theta(\widehat{\calD}_1)} V^{\tilde{\pi}}(\theta) - \calR(\theta) \ge \left(V^{\pi^\dagger}(\theta^\dagger) - \calR(\theta^\dagger) \right) - \eta$. The second inequality uses Jensen's inequality. The third inequality uses the assumption of finite relative condition number~\eqref{asn:rel_cond_number}. Now we can proceed similar to the proof of proposition 14 in \cite{LCSJ22} to establish the following bound (with probability at least $1-\delta$).
    \begin{align*}
        \E_{\stackrel{\tau_0 \sim \mu_0}{ \tau_1 \sim \mu_1 } }\left[ \norm{P_{\theta^\star}(\cdot \mid \tau_0, \tau_1) - P_{\theta^\dagger}(\cdot \mid \tau_0, \tau_1) }_1^2 \right] \le \frac{c}{N}\left( \sum_{n=1}^N \log \left( \frac{P_{\theta^\dagger}(o^n \mid \tilde{\tau}^{0,n}, \tilde{\tau}^{1,n})}{P_{\theta^\star}(o^n \mid \tilde{\tau}^{0,n}, \tilde{\tau}^{1,n})} \right) + \log\left(\frac{\calN(\Theta, 1/N)}{\delta}\right)\right)
    \end{align*}
    Here $\calN(\Theta, 1/N)$ is the number of elements in an $\varepsilon$-net of the set $\Theta$ for $\varepsilon = 1/N$. Since $\norm{\theta}_2 \le H \sqrt{d}$ for each $\theta \in \Theta$ we are guaranteed that $\abs{\calN(\Theta, 1/N)} \le (2H \sqrt{d}N)^d$. Additionally, observe that we are using the clean data $\{\tilde{\tau}^{0,n}, \tilde{\tau}^{1,n}\}_{n=1}^N$ in the bound on the ratio of the log-likelihood.  Now, let $S$ be the set of clean trajectories that have been corrupted by the adversary. Then we can bound the difference in log-likelihood as follows. 
    \begin{align*}
        \frac{1}{N}\sum_{n=1}^N \log \left( \frac{P_{\theta^\dagger}(o^n \mid \tilde{\tau}^{0,n}, \tilde{\tau}^{1,n})}{P_{\theta^\star}(o^n \mid \tilde{\tau}^{0,n}, \tilde{\tau}^{1,n})} \right) &= \frac{1}{N}\sum_{n \notin S} \log \left( \frac{P_{\theta^\dagger}(o^n \mid \tilde{\tau}^{0,n}, \tilde{\tau}^{1,n})}{P_{\theta^\star}(o^n \mid \tilde{\tau}^{0,n}, \tilde{\tau}^{1,n})} \right) + \frac{1}{N}\sum_{n \in S} \log \left( \frac{P_{\theta^\dagger}(o^n \mid \tilde{\tau}^{0,n}, \tilde{\tau}^{1,n})}{P_{\theta^\star}(o^n \mid \tilde{\tau}^{0,n}, \tilde{\tau}^{1,n})} \right)\\
        &\le \frac{1}{N}\sum_{n=1}^N \log \left( \frac{P_{\theta^\dagger}(o^n \mid {\tau}^{0,n}, {\tau}^{1,n})}{P_{\theta^\star}(o^n \mid {\tau}^{0,n}, {\tau}^{1,n})} \right) + \varepsilon \cdot \log \left( \frac{1 + e^{Hd}}{1 + e^{-Hd}}\right)\\
        &\le \frac{1}{N}\sum_{n=1}^N \log \left( \frac{P_{\theta^\star_N}(o^n \mid {\tau}^{0,n}, {\tau}^{1,n})}{P_{\theta^\star}(o^n \mid {\tau}^{0,n}, {\tau}^{1,n})} \right) + 2 \varepsilon H\sqrt{d}\\
        &\le 8\varepsilon H\sqrt{d} + c \cdot \frac{d}{N} \log\left( \frac{Hd}{\delta}\right)
    \end{align*}
    The first inequality uses $\abs{S} \le \varepsilon N$ and $\abs{\phi(\tau)^\top \theta }\le H\sqrt{d}$. The second inequality uses the fact that $\theta^\star_N$ maximizes the log-likelihood over the corrupted dataset, and the final inequality uses lemma~\ref{lem:diff-log-likelihood}.
\end{proof}

\begin{lemma}\label{lem:value-function-lipshitzness}
    For linear MDP, the optimal value function i.e. $V^\star(\theta) = \max_\pi V^\pi(\theta)$ is $\sqrt{Hd}$-Lipschitz in the reward parameter $\theta$.
\end{lemma}
\begin{proof}
    We use the occupancy measure characterization of Markov decision process. Given a probability transition function $P$ let $\calC$ be the set of all feasible occupancy measures with respect to $P$. Then $V^\star(\theta) = \sup_{d \in \calC} \sum_{h=1}^H d_h^\top \Phi \theta_h$.
    \begin{align*}
        V^\star(\theta) - V^\star(\theta') &= \sup_{d \in \calC} \sum_h d_h^\top \Phi \theta_h - \sup_{d \in \calC} \sum_h d_h^\top \Phi \theta'_h \\
        &\le \sup_{d \in \calC} \abs{\sum_{h=1}^H d_h^\top \Phi \theta_h - \sum_{h=1}^H d_h^\top \Phi \theta'_h}\\
        &\le \sup_{d \in \calC} \abs{\sum_{h=1}^H \sum_{s,a} d_h(s,a) \phi(s,a)^\top \left( \theta_h - \theta'_h\right)}\\
        &\le \sup_{d \in \calC} \sum_{h=1}^H \sum_{s,a} d_h(s,a) \norm{\phi(s,a)}_2 \norm{\theta_h - \theta'_h}_2\\
        &\le \sqrt{d} \sum_{h=1}^H \norm{\theta - \theta'}_2 
    \end{align*}
    The last inequality uses $\norm{\phi(s,a)}_2 \le \sqrt{d}$ and $\sum_{s,a} d_h(s,a) = 1$ for any $h$. Now the claim follows from the following observation $\sum_{h=1}^H \norm{\theta - \theta'}_2  \le \sqrt{H}\sqrt{\sum_{h=1}^H \norm{\theta - \theta'}_2^2 } = \sqrt{H} \norm{\theta - \theta'}_2$.
\end{proof}

\begin{lemma}\label{lem:concentration-of-covariance}
Suppose $X_1,\ldots,X_n$ are drawn i.i.d. from a $d$-dimensional distribution with covariance $\Sigma$ and sub-Gaussian norm at most $K$. Then with probability at least $1-\delta$ we have,
$$
\norm{\frac{1}{n}\sum_{i=1}^n X_i X_i^\top - \Sigma} \le c_1 K^2 \norm{\Sigma} \left( \sqrt{\frac{d + \log(1/\delta)}{n} } + \frac{d + \log(1/\delta)}{n}\right).
$$
\end{lemma}
\begin{proof}
    See \cite{Vershynin18} for a proof.
\end{proof}
\begin{lemma}\label{lem:min-max-to-max-min}
    For linear models, $\min_\theta \max_\pi V^\pi(\theta) - \E_{\tau \sim \mu_{\textrm{ref}}}\left[ \phi(\tau)^\top \theta \right]= \max_\pi \min_\theta V^\pi(\theta)- \E_{\tau \sim \mu_{\textrm{ref}}}\left[ \phi(\tau)^\top \theta \right]$.
\end{lemma}
\begin{proof}
We will write $\calR(\theta) = \E_{\tau \sim \mu_{\textrm{ref}}}\left[ \phi(\tau)^\top \theta \right]$. There are two cases to consider. 
    
    \textbf{Case 1}: First, we consider the linear MDP setting. Given a policy $\pi$ let $d$ be the corresponding occupancy measure i.e. $d_h(s,a) = \Pro_\pi(s_h = s, a_h = a)$. Then we can write the value function as $V^\pi(\theta) = \sum_{h,s,a} = d_h(s,a) \phi(s,a)^\top \theta = d^\top \Phi \theta$. This observation implies the following inequality.
    \begin{equation}\label{eq:ineq-le}
    \min_\theta \max_\pi V^\pi(\theta) - \calR(\theta) \le \min_\theta \max_d d^\top \Phi \theta- \calR(\theta) 
    \end{equation}
    On the other hand, given an occupancy measure $d$ one can consider the following policy.
    $$
    \pi^d_h(s,a) = \left\{ \begin{array}{cc}
        \frac{d_h(s,a)}{\sum_b d_h(s,b)} &   \textrm{ if } \sum_b d_h(s,b) > 0\\
        \frac{1}{A} & \textrm{o.w.} 
    \end{array}\right.
    $$
    Moreover, it is known that occupancy measure induced by $\pi^d = (\pi^d_1,\ldots,\pi^d_H)$ is $d$. This implies the following inequality.
    \begin{equation}\label{eq:ineq-ge}
    \min_\theta \max_\pi V^\pi(\theta) - \calR(\theta) \ge \min_\theta \max_d d^\top \Phi \theta- \calR(\theta) 
    \end{equation}
    Therefore, from equations \eqref{eq:ineq-ge} and \eqref{eq:ineq-le} we conclude that
    $$
    \min_\theta \max_\pi V^\pi(\theta) - \calR(\theta)  = \min_\theta \max_d d^\top \Phi \theta- \calR(\theta) 
    $$
    Now observe that the objective $d^\top \Phi \theta- \calR(\theta) $ is linear in both $d$ and $\theta$. Therefore, strong duality holds and we can exchange the order of min and max.
    $$
    \min_\theta \max_\pi V^\pi(\theta) - \calR(\theta) = \min_\theta \max_d d^\top \Phi \theta - \calR(\theta) = \max_d \min_\theta d^\top \Phi \theta - \calR(\theta) 
    $$
    Finally, by an argument very similar to the first part of the proof (correspondence between policy and occupancy measure) we can prove the following identity.
    $$
    \max_d \min_\theta d^\top \Phi \theta - \calR(\theta) = \max_\pi \min_\theta V^\pi(\theta)- \calR(\theta) 
    $$

    \textbf{Case 2}: We now consider the case of trajectory based linear MDP. Let $\calC$ be the set of all valid probability distributions over the trajectories i.e. $\calC = \set{p : \sum_\tau p_\tau = 1, p_\tau \ge 0 \ \forall \tau}$. Given any policy $\pi$, one can consider the probability distribution $p^\pi \in \calC$ induced by $\pi$ so that $V^\pi(\theta) = \sum_\tau p^\pi_\tau \phi(\tau)^\top \theta = p^{\pi^\top} \Phi \theta$. This gives us the following inequality.
    \begin{equation}\label{eq:ineq_le_traj}
    \min_\theta \max_\pi V^\pi(\theta) - \calR(\theta) \le \min_\theta \max_{p \in \calC}p^\top \Phi \theta- \calR(\theta) 
    \end{equation}
    On the other hand, given any probability distribution $p \in \calC$, one can consider the following non-Markovian policy.
    $$
    \pi^p_h(a \mid h) = \left\{ \begin{array}{cc}
        \frac{\sum_{\tau}p_{h,a,\tau}}{\sum_{b,\tau}p_{h,b,\tau}} & \textrm{ if } \sum_{b,\tau}p_{h,b,\tau} > 0 \\
        \frac{1}{A} & \textrm{ o.w. }
    \end{array}\right.
    $$
    We will also write $P_M(\tau')$ to denote the marginal probability of a sub-trajectory $\tau'$ which is defined as $P_M(\tau') = \sum_{\tau''} p_{\tau', \tau''}$.
    Now given any trajectory $\tau = (s_0,a_0,s_1,a_1,s_2,\ldots, s_{H-1},a_{H-1},s_H)$ the probability that the $\tau$ is generated under $\pi^p$ is given as,
    \begin{align*}
    \Pro(\tau) &= \mu(s_0) \pi^p_0(a_0 \mid s_0) \Pro(s_1 \mid s_0, a_0) \mu^p_1(a_1 \mid s_0, a_0, s_1)\\ &\ldots \mu^p_{H-1}(a_{H-1}\mid s_0,\ldots,s_{H-1}) \Pro(s_H \mid s_{H-1}, a_{H-1}) \mu^p_H(a_H \mid s_0,\ldots,s_H)\\
    &= \mu(s_0) \frac{P_M(s_0, a_0) }{P_M(s_0)} \Pro(s_1 \mid s_0, a_0) \frac{P_M(s_0,a_0,s_1, a_1)}{P_M(s_0, a_0, s_1)} \\
    &\ldots \frac{P_M(s_0,\ldots,s_{H-1}, a_{H-1})}{P_M(s_0,\ldots,s_{H-1})} \Pro(s_H \mid s_{H-1}, a_{H-1}) \frac{P_M(s_0,\ldots,s_H,a_H)}{P_M(s_0,\ldots,s_H)}\\
    &= \mu(s_0) \frac{P_M(s_0, a_0, s_1) }{P_M(s_0)} \frac{P_M(s_0,a_0,s_1, a_1, s_2)}{P_M(s_0, a_0, s_1)} \ldots \frac{P_M(s_0,\ldots,s_{H-1}, a_{H-1}, s_H)}{P_M(s_0,\ldots,s_{H-1})} \frac{P_M(s_0,\ldots,s_H,a_H)}{P_M(s_0,\ldots,s_H)}\\
    &= P_M(\tau)
    \end{align*}
    Therefore, policy $\pi^p$ induces the same probability distribution over the trajectories as $p \in \calC$. This implies the following inequality.
   \begin{equation}\label{eq:ineq_ge_traj}
    \min_\theta \max_\pi V^\pi(\theta) - \calR(\theta) \ge \min_\theta \max_{p \in \calC}p^\top \Phi \theta- \calR(\theta) 
    \end{equation}
    Inequalities \eqref{eq:ineq_le_traj} and \eqref{eq:ineq_le_traj} imply the following identity.
    $$
    \min_\theta \max_\pi V^\pi(\theta) - \calR(\theta) = \min_\theta \max_{p \in \calC}p^\top \Phi \theta- \calR(\theta) 
    $$
    The rest of the proof is very similar to case 1 as we can again use strong duality to exchange the order of min and max.
\end{proof}

\section{\texorpdfstring{Missing Proofs from Section \ref{sec:bounded-coverage-ratio}}{Missing Proofs from Section~(5)}}
\subsection{\texorpdfstring{Proof of \Cref{thm:bounded-coverage-feehand}}{Proof of Theorem~(5.1)}}
\begin{proof}
    As shown in the proof of \Cref{thm:relative-cond-number-bound}, $V^\star(\theta) = \max_\pi V^\pi(\theta)$ is a convex function in $\theta$. Let $\calR(\theta) = \E_{\tau \sim \mu_\tref}\left[ \phi(\tau)^\top \theta\right]$. Then $V^\star(\theta) - \calR(\theta)$ is convex in $\theta$.

    Now observe that, algorithm~\eqref{alg:robust_freehand_first_order} performs a projected sub-gradient descent of the function $V^\star(\cdot) - \calR(\cdot)$ with first order oracle calls. Since, \texttt{RobRL} returns a $f(\varepsilon)$ approximate subgradient of the optimal value function $V^\star(\cdot)$, $g_t + \E_{\tau \sim \mu_{\tref}}\left[ \phi(\tau)\right]$ is also an $f(\varepsilon)$ approximate sub-gradient of $V^\star(\theta_t) - \calR(\theta_t)$. Moreover, $\norm{g_t + \E_{\tau \sim \mu_{\tref}}\left[ \phi(\tau)\right]}_2 \le \norm{g_t}_2 + \norm{\E_{\tau \sim \mu_{\tref}}\left[ \phi(\tau)\right]}_2 \le G + \sqrt{H}$, and for any $\theta = (\theta_1,\ldots,\theta_H)$ we have $\norm{\theta}_2 \le \sqrt{Hd}$. Therefore, we can apply \cref{thm:projected-subgradient-first-order} to obtain the following bound.
    $$
    V^\star(\bar{\theta}) - \calR(\bar{\theta}) - \min_\theta \left( V^\star(\theta) - \calR(\theta) \right) \le \frac{\sqrt{Hd(G + \sqrt{H})}}{\sqrt{T}} + f(\varepsilon)
    $$
    If $T \ge \frac{Hd(G + \sqrt{H})}{f(\varepsilon)^2}$, we have
    $$
    V^\star(\bar{\theta}) - \calR(\bar{\theta}) - \min_\theta \left( V^\star(\theta) - \calR(\theta) \right) \le 2 \cdot f(\varepsilon).
    $$
    Since $\tilde{\pi}$ is approximately optimal with respect to the reward parameter $\bar{\theta}$ we are guaranteed that,
    $$
    V^\star(\bar{\theta}) - \calR(\bar{\theta}) - f(\varepsilon) \le V^{\tilde{\pi}}(\bar{\theta}) - \calR(\bar{\theta}) \le V^\star(\bar{\theta}) - \calR(\bar{\theta}) + 2 f(\varepsilon)
    $$
    We can now proceed similar to the proof of \Cref{thm:relative-cond-number-bound}, and establish that $\tilde{\pi}$ approximately optimizes the objective $\max_\pi \min_\theta V^\pi(\theta) - \calR(\theta)$ i.e.
    $$
    \min_\theta \left( V^{\tilde{\pi}}(\theta) - \calR(\theta)\right) \ge \max_\pi \min_\theta \left( V^\pi(\theta) - \calR(\theta)\right) -  2f(\varepsilon)
    $$
    Now we can apply \Cref{lem:diff-policy-value} to complete the proof.
\end{proof}

\subsection{Subgradient Descent with Biased First-Order Oracle}
\textbf{Setting}: Our goal is to minimize a $L$-Lipschitz convex function $f : S \rightarrow [-M,M]$ where $S$ is a convex and bounded set. The function $f$ might not be differentiable, and we have access to a (first-order) noisy oracle, that given a point $x \in S$ returns a sub-gradient vector $g$ such that
$$
f(y) \ge f(x) - \beta + \left \langle g, y - x\right \rangle \forall y \in E.
$$
We will also write $g \in \delta_\beta f(x)$ to denote such a noisy subgradient vector. The next theorem is well-known, but we provide a short proof for completeness.

\begin{theorem}\label{thm:projected-subgradient-first-order}
    Consider the iterates of projected subgradient descent i.e. $\theta_{t+1} = \textrm{Proj}_S\left( \theta_t - \eta g_t\right)$ for $t=0,1,\ldots,T-1$. Suppose $g_t \in \delta_\beta f(\theta_t)$ for all $t$,  $\norm{g_t}_2 \le G$ for all $t$, and $\sup_{\theta \in S} \norm{\theta}_2 \le D$. Then 
    $$
    f(\bar{\theta}) - f(\theta^\star) \le \frac{D\sqrt{G}   }{\sqrt{T}} + \beta
    $$
\end{theorem}
\begin{proof}
    \begin{align*}
        \norm{\theta_{t+1} - \theta^\star}_2^2 &\le \norm{\theta_t - \eta g_t - \theta^\star}_2^2 = \norm{\theta_t - \theta^\star}_2^2 + \eta^2 \norm{g_t}_2^2 - 2\eta \left \langle g_t, \theta_t - \theta^\star \right \rangle
    \end{align*}
    After rearranging and dividing by $2\eta$, we obtain the following inequality.
    \begin{align*}
        \left \langle g_t, \theta_t - \theta^\star \right \rangle \le \frac{1}{2\eta} \left(\norm{\theta_t - \theta^\star}_2^2 - \norm{\theta_{t+1} - \theta^\star}_2^2 \right) + \frac{\eta}{2} \norm{g_t}_2^2
    \end{align*}
    Since $g_t \in \delta_\beta f(\theta_t)$ is a noisy subgradient, using convexity we obtain,
    $$
    f(\theta_t) - f(\theta^\star) \le \left \langle \theta_t - \theta^\star, g_t \right \rangle + \beta.
    $$
    Now using $\bar{\theta} = \frac{1}{T}\sum_{t=1}^T \theta_t$ and convexity of the function $f(\cdot)$ we obtain the following upper bound.
    \begin{align*}
        f(\bar{\theta}) - f(\theta^\star) &\le \frac{1}{T}\sum_{t=1}^T f(\theta_t) - f(\theta^\star) \le \frac{1}{T} \sum_{t=1}^T \left \langle \theta_t - \theta^\star, g_t \right \rangle + \beta \\
        &\le \frac{1}{T} \sum_{t=1}^T \frac{1}{2\eta} \left(\norm{\theta_t - \theta^\star}_2^2 - \norm{\theta_{t+1} - \theta^\star}_2^2 \right) + \frac{\eta}{2T} \sum_{t=1}^T\norm{g_t}_2^2 + \beta\\
        &\le \frac{D^2}{2\eta T} + \frac{\eta G}{2} + \beta
    \end{align*}
    Now choosing $\eta = \frac{D}{\sqrt{GT}}$ we obtain the desired bound.
\end{proof}

\subsection{\texorpdfstring{Proof of \Cref{prop:bounded-coverage-ratio}}{Proof of Proposition~(5.4)}}
\begin{proof}
    For linear MDP, the parameter $\theta = [\theta_1; \theta_2; \ldots; \theta_H]$ and the feature of a trajectory $\tau$ is constructed by concatenating the features of $H$ state, action pairs. Therefore, $\norm{\theta}_2 \le \sqrt{Hd}$ and $\norm{\phi(\tau)}_2 \le \sqrt{H}$ for any trajectory $\tau$. 

    We will use robust offline RL oracle provided by \cref{thm:robust-rl-oracle-first-order}. Note that if $N \ge \tilde{\Omega}(H^2 d^4 \nu^4/\varepsilon^2)$ we have $f(\varepsilon) \le {O}\left( \nu \sqrt{\varepsilon} H^2 d^{3/2}\right)$. Now using the upper bound provided in theorem~\eqref{thm:bounded-coverage-feehand} we obtain the following bound. 

\begin{align*}
    &V^\star(\theta^\star) - V^{\tilde{\pi}}(\theta^\star) \le {O}\left( \nu \sqrt{\varepsilon} H^2 d^{3/2}\right)\\
    &\le O\left(  \kappa \sqrt{\alpha} \left( \sqrt{\varepsilon H} d^{1/4}  + \sqrt{\frac{d}{N} \log \left( \frac{HdN}{\delta}\right)}\right) \right)
\end{align*}
Now observe that if $N \ge \Omega(H \cdot \textrm{poly}(d) / \varepsilon)$ the term $\tilde{O}(\sqrt{d/N})$ can be bounded by $O(\sqrt{\varepsilon})$. Finally, we need a lower bound of $N \ge {\Omega}\left( \frac{H^{3/2}d G   }{f(\varepsilon)^2}\right) = \Omega(d/\varepsilon^2)$ in order to apply theorem~\eqref{thm:bounded-coverage-feehand}.
\end{proof}

\section{Projected Subgradient Descent with Biased Zero-Order Oracle}
\textbf{Setting}: Our goal is to minimize a $L$-Lipschitz convex function $f : S \rightarrow [-M,M]$ where $S$ is a convex and bounded set. The function $f$ might not be differentiable, and we only have access to a noisy oracle $\tilde{f}$ that guarantees $\abs{\tilde{f}(x) - f(x)} \le \varepsilon$ for any $x \in S$. We consider a projected subgradient descent based algorithm where algorithm~\ref{alg:gradient-construction} is used to construct a biased subgradient.

\begin{algorithm}[!h]
\caption{Biased Subgradient Descent}
\label{alg:biased-sgd}
\KwIn{ Stepsize $\eta$, $\theta_0 \in \R^d$, number of iterations $T$.}
\For{$t=0,1,\ldots,T-1$}{
Construct subgradient $g_t = \widetilde{\nabla} f_\mu(\theta_t)$ using algorithm~\eqref{alg:gradient-construction}.\\
$\theta_{t+1} = \textrm{Proj}_S\left(\theta_t - \eta g_t\right)$.
}
$\Bar{\theta} = \frac{1}{T} \sum_{t=1}^T \theta_t$.

\end{algorithm}

\begin{theorem}\label{thm:proj-subgradient-zero-order}
    Suppose algorithm~\eqref{alg:biased-sgd} is run for $T \ge \frac{4DM}{\varepsilon}$ iterations, and we set $K \ge \frac{256 CM^2 d^3}{\varepsilon^2}\ln \left(\frac{16DM}{\varepsilon \delta}\right)$ and $\mu = \frac{\sqrt{\varepsilon}}{\sqrt{8d}}$. Then the output $\Bar{\theta}$ of algorithm~\eqref{alg:biased-sgd} satisfies
    $$
    f(\Bar{\theta}) - f(\theta^\star) \le 5 \sqrt{\varepsilon L} 
    $$
\end{theorem}
\begin{proof}

    \begin{align*}
        \norm{\theta_{t+1} - \theta^\star}_2^2 &\le \norm{\theta_t - \eta g_t - \theta^\star}_2^2 = \norm{\theta_t - \theta^\star}_2^2 + \eta^2 \norm{g_t}_2^2 - 2\eta \left \langle g_t, \theta_t - \theta^\star \right \rangle
    \end{align*}
    After rearranging and dividing by $2\eta$, we obtain the following inequality.
    \begin{align*}
        \left \langle g_t, \theta_t - \theta^\star \right \rangle \le \frac{1}{2\eta} \left(\norm{\theta_t - \theta^\star}_2^2 - \norm{\theta_{t+1} - \theta^\star}_2^2 \right) + \frac{\eta}{2} \norm{g_t}_2^2
    \end{align*}
    Since $g_t = \widetilde{\nabla} f_\mu(\theta_t)$ is a noisy subgradient constructed by algorithm~\eqref{alg:gradient-construction}, using lemma~\eqref{lem:apx-subgradient} we get,
    $$
    f(\theta_t) - f(\theta^\star) \le \left \langle \theta_t - \theta^\star, g_t \right \rangle + b_t
    $$
    where 
    $$
    b_t = \sqrt{\frac{C}{K}} \frac{4M}{\mu} \sqrt{2d\ln(2/\delta)} + \frac{2\varepsilon}{\mu} \textrm{diam}(E) + \mu L \sqrt{d}
    $$
    Summing over $t=0,1,\ldots,T-1$ we obtain the following upper bound.
    \begin{align*}
        \sum_{t=0}^{T-1} f(\theta_t) - f(\theta^\star) &\le \sum_{t=0}^{T-1} \left \langle \theta_t - \theta^\star, g_t \right \rangle + \sum_{t=1}^T b_t\\
        &\le \frac{1}{2\eta} \sum_{t=0}^{T-1} \left(\norm{\theta_t - \theta^\star}_2^2 - \norm{\theta_{t+1} - \theta^\star}_2^2 \right) + \frac{\eta}{2} \sum_{t=0}^{T-1} \norm{g_t}_2^2 + \sum_{t=0}^{T-1} b_t\\
        &\le \frac{1}{2\eta} \left(\norm{\theta_0 - \theta^\star}_2^2 - \norm{\theta_{T} - \theta^\star}_2^2 \right) + \frac{\eta}{2} \sum_{t=0}^{T-1} \norm{g_t}_2^2 + \sum_{t=0}^{T-1} b_t
    \end{align*}
    From the construction of subgradient in algorithm~\eqref{alg:gradient-construction} it is clear that $\norm{g_t}_2 \le \frac{2M}{\mu} \textrm{diam}(E)$. Moreover, diameter of $S$ is at most $D$. This gives us the following result.
    \begin{align*}
        f(\Bar{\theta}) - f(\theta^\star) \le \frac{1}{T} \sum_{t=0}^{T-1} f(\theta_t) - f(\theta^\star) \le \frac{2D^2}{\eta T} + \eta \frac{2M^2}{\mu^2 T} \textrm{diam}^2(E) + \sqrt{\frac{C}{K}} \frac{4M}{\mu} \sqrt{2d\ln(2/\delta)} + \frac{2\varepsilon}{\mu} \textrm{diam}(E) + \mu L \sqrt{d}
    \end{align*}
    We now substitute $\eta = \frac{D\mu}{M \textrm{diam}(E)}$.
    \begin{align*}
         f(\Bar{\theta}) - f(\theta^\star) \le \frac{4DM}{\mu T} \textrm{diam}(E) + \sqrt{\frac{C}{K}} \frac{4M}{\mu} \sqrt{2d\ln(2/\delta)} + \frac{2\varepsilon}{\mu} \textrm{diam}(E) + \mu L \sqrt{d}
    \end{align*}
    We further substitute $\mu = \frac{\sqrt{\varepsilon}   }{\textrm{diam}(E)\sqrt{L}}$ and choose $T \ge \frac{4DM}{\varepsilon}$ and $K \ge \frac{32 C M^2 \textrm{diam}^2(E)}{\varepsilon^2} d \ln(2 T/\delta)$.
\begin{align}\label{eq:almost-final-bound}
    f(\Bar{\theta}) - f(\theta^\star) \le 4 \sqrt{\varepsilon L} + \frac{\sqrt{\varepsilon L d}}{\textrm{diam}(E)}
\end{align}
Now recall that lemma~\eqref{lem:apx-subgradient} requires that the set $E$ be such that $\int_E \exp\left( -\frac{1}{4}\norm{u}_2^2\right) du \ge \frac{1}{2}$. We choose a simple set $E = [-\ell, \ell]^d$ and show that one can pick $\ell = O(1)$. Then we have,
\begin{align*}
    \int_E \exp\left( -\frac{1}{4}\norm{u}_2^2\right) du = \left\{\int_{\ell}^{-\ell} \exp(-1/4 v^2) dv \right\}^d = \left\{2 \int_{\ell/2}^{-\ell/2} \exp(-1/2 t^2) dt \right\}^d = \left\{2 (2 \Phi(\ell/2) - 1) \right\}^d.
\end{align*}
Here $\Phi(t) = \Pro(X \le t)$ with $X$ being a standard Gaussian random variable. Substituting $\Phi(t) \ge 1 - e^{-t^2/2}$ we get the following lower bound.
\begin{align*}
    \int_E \exp\left( -\frac{1}{4}\norm{u}_2^2\right) du  \ge \left\{2 (1 - 2e^{-\ell^2/8}) \right\}^d
\end{align*}
It can be checked that picking $\ell > \sqrt{8 \ln 4}$ satisfies $\int_E \exp\left( -\frac{1}{4}\norm{u}_2^2\right) du  \ge 1/2$. Therefore, we choose $E = [-4,4]^d$. This also implies that $\textrm{diam}(E) = \sqrt{8d}$ and substituting this bound in \cref{eq:almost-final-bound} we obtain the following upper bound.
\begin{align*}
    f(\Bar{\theta}) - f(\theta^\star) \le 4 \sqrt{\varepsilon L} + \frac{\sqrt{\varepsilon L }  }{\sqrt{8}}
\end{align*}
\end{proof}

\subsection{Gradient Construction}
Given a convex function $f: E \rightarrow \R^d$, let $f_\mu$ be defined as its Gaussian approximation.
$$
f_\mu(x) = \frac{1}{\kappa} \int_E f(x + \mu u) e^{-\frac{1}{2} \norm{u}_2^2} du
$$
where $\kappa = \int_E e^{-\frac{1}{2} \norm{u}_2^2} du$. Suppose $f$ is $L$-Lipschitz then the following results are well known~\cite{NS17}.
\begin{enumerate}
    \item For any $x \in E$, $\abs{f_\mu(x) - f(x)} \le \mu L \sqrt{d}$.
    \item $\nabla f_\mu(x) = \frac{1}{\kappa} \int_E \frac{f(x + \mu u) - f(x)}{\mu} e^{-\frac{1}{2}\norm{u}_2^2} u du$.
    \item \label{eq:observation-subgradient}$\nabla f_\mu(x) \in \delta_\alpha f(x)$ for $\alpha = \mu L \sqrt{d}$ i.e. $f(y) \ge f(x) - \mu L \sqrt{d} + \left \langle \nabla f_\mu(x), y - x\right \rangle$ for all $y \in E$.
\end{enumerate}

\begin{algorithm}[!h]
\caption{Gradient Construction}
\label{alg:gradient-construction}
\KwIn{ Noisy oracle $\tilde{f}$, number of iterations $K$, input $x$.}
Generate $u_1,\ldots,u_K$ uniformly at random from the standard normal distribution (restricted to the set $E$).\\
 Let $\widetilde{\nabla} f_\mu(x) = \frac{1}{K} \sum_{k=1}^K \frac{\tilde{f}(x + \mu u_k) - \tilde{f}(x)}{\mu} u_k$.\\
 $\widetilde{\nabla} f_\mu(x)$.

\end{algorithm}

\begin{lemma}\label{lem:apx-subgradient}
    Suppose the set $E$ is chosen so that $\int_E e^{-\frac{1}{4} \norm{u_k}_2^2} \ge \frac{1}{2}$, and $\abs{\tilde{f}(x) - f(x)} \le \varepsilon$ for any $x$. Then the gradient estimate returned by algorithm~\eqref{alg:gradient-construction} satisfies
    $$
    \widetilde{\nabla} f_\mu(x) \in \delta_\alpha f(x)\ \textrm{ for }\ \alpha = \sqrt{\frac{C}{K}} \frac{4M}{\mu} \sqrt{2d\ln(2/\delta)} + \frac{2\varepsilon}{\mu} \textrm{diam}(E) + \mu L \sqrt{d}
    $$
    with probability at least $1-\delta$.
\end{lemma}
\begin{proof}
    Let $\widehat{\nabla} f_\mu(x) = \frac{1}{K} \sum_{k=1}^K \frac{{f}(x + \mu u_k) - {f}(x)}{\mu} u_k$. Then we have,
    \begin{align}
        \norm{\widehat{\nabla} f_\mu(x) - \widetilde{\nabla} f_\mu(x)}_2 &= \frac{1}{K} \norm{\sum_{k=1}^K \frac{\left( \tilde{f}(x + \mu u_k) - {f}(x + \mu u_k) \right) - \left( \tilde{f}(x) - {f}(x) \right)}{\mu} u_k}_2\nonumber\\
        &\le \frac{2\varepsilon}{\mu K} \sum_{k=1}^K \norm{u_k}_2\nonumber\\
        &\le \frac{2\varepsilon}{\mu } \textrm{diam}(E) \label{eq:noisy-to-estimated}
    \end{align}
    We now show that $\widehat{\nabla} f_\mu(x)$ concentrates around $\nabla f_\mu(x)$. Let $V_k = \frac{{f}(x + \mu u_k) - {f}(x)}{\mu} u_k$. We claim that the sub-Gaussian norm of $V_k$ is at most $\frac{4M}{\mu}$. This follows from two observations. First, $\abs{\frac{{f}(x + \mu u_k) - {f}(x)}{\mu}} \le \frac{2M}{\mu}$. Second, we show that the sub-Gaussian norm of the random vector $u_k$ is at most $2$. Since $\norm{u_k}_{\psi_2} = \sup_{v \in S_{d-1}} \norm{u_k^\top v}_{\psi_2}$, consider any $v \in \R^d$ with $\norm{v}_2 = 1$.
    \begin{align*}
        \E\left[ e^{\frac{(u_k^\top v)^2}{4}}\right] = \frac{1}{\kappa} \int_E e^{\frac{(u_k^\top v)^2}{4}} e^{-\frac{1}{2} \norm{u_k}_2^2} du_k \le \frac{1}{\kappa} \int_E e^{-\frac{1}{4} \norm{u_k}_2^2} du_k = \frac{\int_E e^{-\frac{1}{4} \norm{u_k}_2^2}}{\int_E e^{-\frac{1}{2} \norm{u_k}_2^2}} \le \frac{1}{\int_E e^{-\frac{1}{4} \norm{u_k}_2^2}} \le 2
    \end{align*}
    The first inequality uses $\norm{v}_2 = 1$, and the second inequality uses Jensen's inequality. We can now use proposition 2.6.1 from \cite{Vershynin18} to bound the sub-Gaussian norm of the average vector.
    $$
    \norm{\sum_{k=1}^K \frac{f(x + \mu u_k) - f(x)}{\mu} u_k}_{\psi_2} \le \sqrt{C \sum_{k=1}^K \norm{\frac{f(x + \mu u_k) - f(x)}{\mu} u_k}_{\psi_2}^2} \le \sqrt{CK} \frac{4M}{\mu}
    $$
    for some universal constant $C > 0$. Therefore, $\norm{\widehat{\nabla} f_\mu(x)}_{\psi_2} \le \sqrt{\frac{C}{K}} \frac{4M}{\mu}$. This also means that $\norm{\widehat{\nabla} f_\mu(x)}_{\psi_2}$ is $\sqrt{\frac{C}{K}} \frac{4M}{\mu} \sqrt{d}$ norm sub-Gaussian~\cite{JNGK+19} and from the definition of norm sub-Gaussian random vectors (definition 3 from \cite{JNGK+19}) we have the following bound.
    \begin{align}\label{eq:estimated-to-mean}
        \textrm{Pr}\left( \norm{\widehat{\nabla} f_\mu(x) - {\nabla} f_\mu(x)}_2 \ge \sqrt{\frac{C}{K}} \frac{4M}{\mu} \sqrt{2d\ln(2/\delta)} \right) \le \delta
    \end{align}
    Finally, we can combine \cref{eq:noisy-to-estimated}, and \cref{eq:estimated-to-mean} and use \cref{eq:observation-subgradient} to obtain the desired bound.
\end{proof}

\section{A New Corruption Robust Offline RL Method}
We adopt the linear programming based formulation of reinforcement learning~\cite{Manne60}. We will write $\Phi \in \R^{SA \times d}$ to write the feature matrix, and $P_h \in \R^{S \times SA}$ to be the transition probability matrix at time-step $h$, which is defined as $P_h(s, (s',b')) = P_h(s \mid s',b')$. Note that we can write $P_h = \Psi_h \Phi^\top$ where $\bm{\mu}_h\in \R^{S \times d}$ is the $\bm{\mu}_h$ is the $d$-dimensional measure matrix.
\begin{align*}
    \max_q\ &\sum_{h=1}^H q_h^\top \Phi \theta_h\\
    \textrm{s.t.}\ &\sum_a q_1(s,a) = \rho(s)\ \forall s\\
    &E q_{h+1} = \bm{\mu}_h \Phi^\top q_h\ \forall h \in \set{1,2,\ldots,H-1}\\
    &q_h \ge 0 \ \forall h \in [H]
\end{align*}
The matrix $E \in \R^{S \times SA}$ is defined as $E(s, (s',a')) = \one\set{s = s'}$. We make the following substitution $\lambda_h = \Phi^\top q_h$ to obtain the following equivalent LP.
\begin{align}\label{eq:primal-LP}
\begin{split}
    \max_{\{q_h\}_{h=1}^H,\{\lambda_h\}_{h=1}^H}\ &\sum_{h=1}^H \lambda_h^\top  \theta_h\\
    \textrm{s.t.}\ &Eq_1 = \rho \\
    &E q_{h+1} = \bm{\mu}_h \lambda_h\ \forall h \in \set{1,2,\ldots,H-1}\\
    &q_h \ge 0 \ \forall h \in [H]\\
    &\lambda_h = \Phi^\top q_h \ \forall h \in [H]
    \end{split}
\end{align}
The dual problem of the above optimization problem is the following optimization problem.
\begin{align}
    \label{eq:dual-LP}
    \begin{split}
    \min_{\{v_h\}_{h=1}^H, \{w_h\}_{h=1}^H}\ &\rho^\top v_1\\
    \textrm{s.t.}\
    &E^\top v_h \ge \Phi w_{h}\ \forall h \in [H]\\
    &w_h \ge \theta_h + \bm{\mu}^\top_h v_{h+1} \ \forall h \in [H-1]\\
    &w_H \ge \theta_H
    \end{split}
\end{align}
The corresponding Lagrangian is given as $\calL(\bm{q},\bm{\lambda}; \bm{v}, \bm{w})$ where
\begin{align*}
    \calL(\bm{q},\bm{\lambda}; \bm{v}, \bm{w}) &= \rho^\top v_1 + \sum_{h=1}^H \left \langle q_h, -E^\top v_h + \Phi w_h\right \rangle + \sum_{h=1}^{H-1} \left\langle \theta_h + \bm{\mu}^\top_h v_{h+1} - w_h, \lambda_h \right \rangle + \left \langle \theta_H - w_H, \lambda_H \right \rangle\\
    &=\sum_{h=1}^H \lambda_h^\top \theta_h + \left \langle v_1, -Eq_1 + \rho \right \rangle + \sum_{h=2}^H \left \langle v_h, - Eq_{h+1} + \bm{\mu}_h \lambda_h \right \rangle + \sum_{h=1}^H \left \langle w_h,  \Phi^\top q_h - \lambda_h \right \rangle\\
\end{align*}
We aim to solve a saddle point of the Lagrangian through gradient descent-ascent method. Note that each of $\lambda_h$ and $w_h$ is $d$-dimensional. So we will only perform gradient steps over these variables, whereas we will represent high-dimensional (possible infinite) $v_h$ and $q_h$ implicitly. The gradient with respect to $\lambda_h$ is given through the following expression.
\begin{align*}
    \nabla_{\lambda_h} \calL(\bm{q},\bm{\lambda}; \bm{v}, \bm{w}) = \left\{ \begin{array}{cc}
        \theta_h + \bm{\mu}^\top_h v_{h+1} - w_h & \textrm{ if } h \in [H-1] \\
        \theta_h - w_h & \textrm{ if } h = H 
    \end{array}\right.
\end{align*}
Now we introduce a transformation of variables suggested by \cite{GNOP23}. Let $\Lambda_h = \E_{(s,a) \sim \mu_{\tref}^h}\left[\phi(s,a) \phi(s,a)^\top \right]$ be the covariance matrix under the reference policy $\mu_{\tref}$ at time step $h$. Then we can rewrite the gradient as follows.
\begin{align*}
    \nabla_{\lambda_h} \calL(\bm{q},\bm{\lambda}; \bm{v}, \bm{w}) &= \Lambda_h^{-1} \Lambda_h \left(\theta_h + \bm{\mu}^\top_h v_{h+1} - w_h \right) = \Lambda_h^{-1} \E_{(s,a) \sim \mu_{\tref}^h}\left[\phi(s,a) \phi(s,a)^\top\left(\theta_h + \bm{\mu}^\top_h v_{h+1} - w_h \right)\right]\\
    &= \Lambda_h^{-1} \E_{(s,a) \sim \mu_{\tref}^h, s' \sim P_h(\cdot \mid s, a)}\left[\phi(s,a) \left(r_h(s,a) + v_{h+1}(s') - w_h^\top \phi(s,a) \right)\right]
\end{align*}
We can build an estimator of the expectation from samples, however the covariance matrix $\Lambda_h$ might be unknown. Therefore, as proposed by \cite{GNOP23}, we substitute $\beta_h = \Lambda_h^{-1} \lambda_h$ for any $h \in [H]$ in the Lagrangian.
\begin{align}\label{eq:reparametrized-Lagrangian}
\calL(\bm{q},\bm{\beta}; \bm{v}, \bm{w}) &= \rho^\top v_1 + \sum_{h=1}^H \left \langle q_h, -E^\top v_h + \Phi w_h\right \rangle + \sum_{h=1}^{H-1} \left\langle \Lambda_h \left( \theta_h + \bm{\mu}^\top_h v_{h+1} - w_h\right), \beta_h \right \rangle + \left \langle \Lambda_H \left(\theta_H - w_H\right), \beta_H \right \rangle
\end{align}
Gradient with respect to $\beta_h$ is given as follows.
\begin{align*}
    \nabla_{\beta_h} \calL(\bm{q},\bm{\beta}; \bm{v}, \bm{w}) =\left\{
    \begin{array}{cc}
       \E_{(s,a) \sim \mu_{\tref}^h, s' \sim P_h(\cdot \mid s, a)}\left[\phi(s,a) \left(r_h(s,a) + v_{h+1}(s') - w_h^\top \phi(s,a) \right)\right]  & \textrm{ if } h \in [H-1] \\
        \E_{(s,a) \sim \mu_{\tref}^h}\left[\phi(s,a) \left(r_h(s,a)  - w_h^\top \phi(s,a) \right)\right] & \textrm{ if } h=H
    \end{array}\right.
\end{align*}
Therefore, given any data point $(s_h, a_h, s_h', r_h)$ we can define the following estimate of the gradient.
\begin{align*}
    \widetilde{g}_{\beta_h} = \widehat{\nabla}_{\beta_h} \calL(\bm{q},\bm{\beta}; \bm{v}, \bm{w}) =\left\{
    \begin{array}{cc}
       \phi(s_h,a_h) \left(r_h + v_{h+1}(s'_h) - w_h^\top \phi(s_h,a_h) \right)  & \textrm{ if } h \in [H-1] \\
        \phi(s_h,a_h) \left(r_h  - w_h^\top \phi(s_h,a_h) \right) & \textrm{ if } h=H
    \end{array}\right.
\end{align*}
On the other hand, the gradient with respect to $w_h$ is the following.
\begin{align*}
    \nabla_{w_h} \calL(\bm{q},\bm{\beta}; \bm{v}, \bm{w}) =
      \Phi^\top q_h - \Lambda_h \beta_h = \Phi^\top q_h - \E_{(s,a) \sim \mu_\tref}\left[ \phi(s,a) \cdot \beta_h^\top \phi(s,a) \right]
\end{align*}
This leads to the following estimate of the gradient with respect to $w_h$.
\begin{align*}
    \widetilde{g}_{w_h} = \widehat{\nabla}_{w_h} \calL(\bm{q},\bm{\beta}; \bm{v}, \bm{w}) = \Phi^\top q_h -  \phi(s_h,a_h) \cdot \beta_h^\top \phi(s_h,a_h) 
\end{align*}

We will also use the following symbolic representation for policy, value, and occupancy measure.
$$
\pi_h(a \mid s) = \frac{\exp(\phi(s,a)^\top w_h)}{\sum_b \exp(\phi(s,b)^\top w_h)}
$$
$$
v_h(s) = \sum_a \pi_h(a \mid s) \phi(s,a)^\top w_h
$$
and
$$
q_1(s) = \rho(s) \ \textrm{ and } \ q_{h+1}(s') = \bm{\mu}_{h}(s')^\top \Lambda_h \beta_h = \E_{(s,a) \sim \mu^h_\tref} \left[ P_h(s' \mid s,a) \phi(s,a)^\top \beta_h \right] 
$$

\begin{algorithm}[!h]
\caption{Corruption Robust Offline Primal-Dual}\label{alg:offline-primal-dual}

\KwIn{(a) Corrupted dataset $\calD$, (b) corruption parameter $\varepsilon$, (c) Step sizes $\eta_w$, $\eta_b$, and $\alpha$, and (d) Number of iterations $T$.}
Partition dataset $\calD$ uniformly at random into two datasets $\calD_m$ and $\calD_c$, where $\calD_c = \Theta(H \cdot d^2/\varepsilon^2 \log^2(d))$.\\
Partition dataset $\calD_m$ uniformly at random into $2 H T$ groups $\{\calD^{t,h}_1, \calD^{t,h}_2\}_{h\in [H], t \in [T]}$.\\
 Initialize $w^0 = \{w^0_h\}_{h=1}^H$ and $\beta^0 = \{\beta^0_h\}_{h=1}^H$.\\
\For{$t=0,\ldots,T-1$}{

\For{$h=1,\ldots,H$}{
\tcc{Take a gradient step for $w_h$}
Set $\pi_h^t(a \mid s) \propto \exp\left(\alpha  \phi(s,a)^\top w_h^t \right)$.\\
For each $j \in [K]$, set (symbolically) \begin{align*}
q_{h,j}^t(\tilde{s}, \tilde{b}) = \left\{  \begin{array}{cc}
    \pi^t_h(\tilde{b}\mid \tilde{s}) \cdot \one \set{s^{'}_j = \tilde{s}} \phi(s_{h,j},a_{h,j})^\top \beta_{h-1}^t & \textrm{ if } h > 1 \\
     \pi^t_h(\tilde{b}\mid \tilde{s}) \cdot \rho(\tilde{s}) & \textrm{ if } h = 1 
\end{array} \right.     
\end{align*}
 Set $\widetilde{g}_{w_h}^t = \textrm{RobMean}\left(\left\{ \Phi^\top q_{h,j}^t - \phi(s^2_{h,j},a^2_{h,j}) \cdot \left \langle \beta_h^t, \phi(s^2_{h,j}, a^2_{h,j}) \right \rangle \right\}_{j=1}^K \right)$.\\
 $w_h^{t+1} \leftarrow \textrm{Proj}_{\calW}(w_h^t - \eta_w \cdot \widetilde{g}_{w_h}^t)$
}
\For{$h=1,\ldots,H$}{
\tcc{Take a gradient step for $\beta_h$}
Set $\pi_h^t(a \mid s) \propto \exp\left(\alpha  \phi(s,a)^\top w_h^t \right)$.\\
Set $v_{h}^t({s}) = \sum_{a} \pi_h^t(a \mid s) \cdot \phi(s,a)^\top w_h^t$.\\
Set $\widetilde{g}_{\beta_h}^t = \widehat{\nabla}_{\beta_h} \calL(\bm{q},\bm{\beta}; \bm{v}, \bm{w})$ defined as
\begin{align*}
     \widetilde{g}_{\beta_h}^t =\left\{
    \begin{array}{cc}
       \textrm{RobMean} \left( \left\{ \phi(s_{h,j},a_{h,j}) \left(r_{h,j} + v_{h+1}(s'_{h,j}) - \left \langle w_h^t, \phi(s_h,a_h) \right \rangle \right)\right\}_{j=1}^K \right) & \textrm{ if } h \in [H-1] \\
        \textrm{RobMean}\left( \left\{ \phi(s_{h,j},a_{h,j}) \left(r_{h,j}  - \left \langle w_h^t, \phi(s_{h,j},a_{h,j}) \right \rangle  \right) \right\}_{j=1}^K \right) & \textrm{ if } h=H
    \end{array}\right.
\end{align*}
$\beta_h^{t+1} \leftarrow \textrm{Proj}_{\mathcal{B} }(\beta_h^t + \eta_b \cdot \widetilde{g}_{\beta_h}^t)$
}
}
Partition dataset $\calD_c$ uniformly at random into $H$ groups $\{\calD_c^h\}_{h \in [H]}$.\\
\For{$h=1,\ldots,H$}{
 Set $\Bar{w}_h = \frac{1}{T} \sum_{t=1}^T w_h^t$ and $\bar{\beta}_h = \frac{1}{T} \sum_{t=1}^T \beta_h^t$.\\
 Set $\widehat{v}_h = \textrm{RobCovariance}(\calD^h_c) \cdot \bar{\beta}_h$
}
 \Return $\bar{\pi} = (\bar{\pi}_1,\ldots, \bar{\pi}_H)$ and $\widehat{v} = (\widehat{v}_1,\ldots, \widehat{v}_H)$.
\end{algorithm}

Given $w_h, \beta_h$ we define policy $\pi_h$ as
\begin{align*}
    \pi_h(a \mid s) = 
        \frac{\exp(\phi(s,a)^\top w_h)}{\sum_b \exp(\phi(s,b)^\top w_h)} .
\end{align*}
We also define $q^{\pi,\beta}_h$ as
\begin{align*}
    q^{\pi,\beta}_h(s,a) = \left\{ \begin{array}{cc}
       \pi_h(a \mid s) \cdot \rho(s)  & \textrm{ if } h = 1 \\
   \pi_h(a \mid s) \cdot \bm{\mu}_h(s)^\top \Lambda_{h-1} \beta_{h-1}  & \textrm{ o.w. }
    \end{array}\right.
\end{align*}
After substituting $q_h = q^{\pi,\beta}_h$ we obtain the following form of the Lagrangian. 
\begin{align}\label{eq:Lagrangian-form-1}
    \calL(\bm{q}, \bm{\beta}; \bm{v}, \bm{w}) = f(\bm{\pi}, \bm{\beta}, \bm{w}) = \sum_{h=1}^H \left \langle \Lambda_h \theta_h, \beta_h \right \rangle + \sum_{h=1}^H \left \langle w_h, \Phi^\top q^{\pi,\beta}_h - \Lambda_h \beta_h \right \rangle
\end{align}
This also gives us the following expression for derivative with respect to $w_h$.
\begin{align}\label{eq:derivative-wrt-w_h}
    \nabla_{w_h} f(\bm{\pi}, \bm{\beta}, \bm{w}) = \Phi^\top q^{\pi,\beta}_h - \Lambda_h \beta_h
\end{align}

Additionally, if we write $v^{\pi,w}_h(s) = \sum_a \pi_h(a \mid s) \cdot w_h^\top \phi(s,a)$ and $d^\beta_h = E q^{\pi,\beta}_h$ then we obtain the following form of the Lagrangian. 
\begin{align}\label{eq:Lagrangian-form-2}
\calL(\bm{q}, \bm{\beta}; \bm{v}, \bm{w}) = f(\bm{\pi}, \bm{\beta}, \bm{w}) = \sum_{h=1}^H \left \langle \Lambda_h (\theta_h-w_h), \beta_h \right \rangle + \sum_{h=1}^H \left \langle d^\beta_h, v^{\pi,w}_h \right \rangle
\end{align}
And, we can write down the derivative with respect to $\beta_h$ for any $h > 1$ as
\begin{align}\label{eq:derivative-wrt-beta_h}
    \nabla_{\beta_h} f(\bm{\pi}, \bm{\beta}, \bm{w}) = \Lambda_h(\theta_h - w_h) + \sum_{s'} v^{\pi, w}_{h+1}(s') \nabla_{\beta_h} d^\beta_{h+1}(s') = \Lambda_h(\theta_h - w_h) + \sum_{s'} v^{\pi, w}_{h+1}(s') \Lambda_h \bm{\mu}_h(s').
\end{align}
And, for $h=1$ we have,
\begin{align}
    \nabla_{\beta_h} f(\bm{\pi}, \bm{\beta}, \bm{w}) = \Lambda_h(\theta_h - w_h).
\end{align}
Following~\cite{GNOP23} we define the following notion of regret.
\begin{equation}
    \calR(\bm{\beta}^\star, \bm{\pi}^\star, \bm{w}^\star_{1:T}) = \frac{1}{T} \sum_{t=1}^T f( \bm{\beta}^\star, \bm{\pi}^\star, \bm{w}_t) - f(\bm{\beta}_t, \bm{\pi}_t, \bm{w}^\star_t)
\end{equation}
\begin{lemma}\label{lem:regret-upper-bound}
    Suppose $\pi^\star = (\pi^\star_1,\ldots,\pi^\star_H)$ be a policy and $q^{\pi^\star}$ be its state, action occupancy measure. If we set $\beta^\star_h = \Lambda_h^{-1} \Phi^\top q^\star_h$ for each $h=1,\ldots,H$, and $w^\star_{t,h} = w^t_h$ for each $t \in [T]$ and $h \in [H]$, the the policy $\bar{\pi}$ output by algorithm~\eqref{alg:offline-primal-dual} satisfies
    $$
    \E\left[ (q^{\pi^\star} - q^{\bar{\pi}})^\top r\right] \le \calR(\bm{\beta}^\star, \bm{\pi}^\star, \bm{w}^\star_{1:T}) 
    $$
\end{lemma}
\begin{proof}
The proof is very similar to the proof of lemma 4.1 of \cite{GNOP23}.
\end{proof}
\begin{lemma}
    \label{lem:regret-decomposition}
    With the choice of the parameters as in \Cref{lem:regret-upper-bound}, we have the following regret decomposition.
    \begin{align*}
    \calR(\bm{\beta}^\star, \bm{\pi}^\star, \bm{w}^\star_{1:T})  &= \frac{1}{T} \sum_{t=1}^T \sum_{h=1}^H \left \langle w_{t,h} - w^\star_h, \nabla_{w_h} f(\bm{\pi}_t, \bm{\beta}_t, \bm{w}_t) \right \rangle + \frac{1}{T} \sum_{t=1}^T \sum_{h=1}^H \left \langle \beta^\star_{h} - \beta^\star_{t,h}, \nabla_{\beta_h} f(\bm{\pi}_t, \bm{\beta}_t, \bm{w}_t) \right \rangle\\
    &+ \frac{1}{T} \sum_{t=1}^T \sum_{h=1}^H \sum_{s} q^{\pi^\star}_h(s) \sum_a \left( \pi^\star_h(a \mid s) - \pi_{t,h}(a \mid s)\right) \left \langle w_{t,h}, \phi(s,a) \right \rangle 
    \end{align*}
\end{lemma}
\begin{proof}
    The proof is very similar to the proof of lemma 4.2 of \cite{GNOP23}.
\end{proof}

\subsection{\texorpdfstring{Formal Statement and Proof of \Cref{thm:robust-rl-oracle-first-order}}{Formal Statement and Proof of Theorem~(5.3)}}
\begin{theorem}
    Suppose assumptions~\eqref{asn:bounded_coverage}  holds, and $N \ge {\Omega}\left( \frac{H^2 d^4 \nu^4}{\varepsilon^2} (\log^2 d + \log^2 A)\right)$. Then the policy $\bar{\pi}$ output by algorithm~\eqref{alg:offline-primal-dual} is approximately optimal i.e.
    $$
    \max_\pi V^{\pi}(\theta) - \E\left[ V^{\bar{\pi}}(\theta) \right] \le O\left(\nu \sqrt{\varepsilon } H^2 d^{3/2}\right),
    $$
    and the vector $\widehat{v} = (\widehat{v}_1,\ldots, \widehat{v}_H)$ is an approximate sub-gradient to $V^\star(\theta) = \max_\pi V^\pi(\theta)$ i.e.
    $$
    V^\star(\theta') \ge V^\star(\theta) + \sum_{h=1}^H \left \langle \widehat{v}_h, \theta_h\right \rangle -  O\left(\nu \sqrt{\varepsilon } H^2 d^{3/2}\right) \ \forall \theta'.
    $$
\end{theorem}
\begin{proof}
Let $\Lambda_h$ be the feature covariance matrix under the offline policy $\pi_{\tref}$ at time step $h$. Moreover, let $d^\star_h = \E_{(s,a) \sim \pi^\star} [\phi(s,a)]$ and $\beta^\star_h = d^\star_h \Lambda_h^{-1}$. Then by assumption~\eqref{asn:bounded_coverage}, $\norm{\beta^\star}_2 \le \nu$. Therefore, it is sufficient to take diameter of the set $\mathcal{B}$ to be $\nu$. We now bound the diameter of the set $\mathcal{W}$ from the feasiblity condition in the optimization problem~\eqref{eq:dual-LP}. It can be easily seen that given any optimal solution ($\{v_h\}_{h=1}^H, \{w_h\}_{h=1}^H$), we can always choose $w_h = \theta_h + \bm{\mu}_h^\top v_{h+1}$ for any $h \in [H-1]$, and $w_H = \theta_H$. Indeed, if this condition is not satisfied, then we can define the following new set of variables.
$$
\tilde{w}_H = \theta_H \ \textrm{ and } \ \tilde{w}_h = \theta_h + \bm{\mu}_h^\top \tilde{v}_{h+1},\ \tilde{v}_h(s) = \sum_{a} \phi(s,a)^\top \tilde{w}_h \ \textrm{ for } h=H-1,\ldots,1
$$
This new set of variables is feasible to the optimization problem~\eqref{eq:dual-LP} and has objective value bounded above by $\rho^\top v_1$. For linear MDP, the reward at every step is at most $\sqrt{d}$, and hence the value function $v_h(s)$ is at most $H\sqrt{d}$. This implies that for any $h$, $\norm{w_h}_2 \le \norm{\theta_h} + \norm{\bm{\mu}_h^\top \widetilde{v}_{h+1}}_2 \le \sqrt{d} + H \sqrt{d} \norm{\bm{\mu}_h}_2 \le 2Hd$. Therefore, $\norm{w}_2^2 = \sum_{h=1}^H \norm{w_h}_2^2 \le 2 H^2d$, and we can take the diameter of the set $\mathcal{W}$ to be at most $2H\sqrt{d}$.

By lemma~\eqref{lem:regret-upper-bound} and \eqref{lem:regret-decomposition} we can express the suboptimality of value function as follows.
\begin{align*}
    V^{\pi^\star}(\theta) - \E\left[ V^{\bar{\pi}}(\theta)\right] \le & \underbrace{ \frac{1}{T} \sum_{t=1}^T \sum_{h=1}^H \left \langle w_{t,h} - w^\star_h, \nabla_{w_h} f(\bm{\pi}_t, \bm{\beta}_t, \bm{w}_t) \right \rangle}_{:= \reg_1} + \underbrace{\frac{1}{T} \sum_{t=1}^T \sum_{h=1}^H \left \langle \beta^\star_{h} - \beta^\star_{t,h}, \nabla_{\beta_h} f(\bm{\pi}_t, \bm{\beta}_t, \bm{w}_t) \right \rangle}_{:= \reg_2}\\
    &+ \underbrace{\frac{1}{T} \sum_{t=1}^T \sum_{h=1}^H \sum_{s} q^{\pi^\star}_h(s) \sum_a \left( \pi^\star_h(a \mid s) - \pi_{t,h}(a \mid s)\right) \left \langle w_{t,h}, \phi(s,a) \right \rangle }_{:= \reg_3}
\end{align*}
We now apply \Cref{lem:bound-regret-1} with $W = 2 H \sqrt{d}$, $B= \nu$, and $\eta_w = \frac{W}{Bd}\frac{1}{\sqrt{T}} = \frac{H}{\nu \sqrt{dT}}$ to obtain the following bound on the term $\reg_1$.
\begin{equation}
    \reg_1 \le O\left(\nu \sqrt{\varepsilon d} H  \sum_{h=1}^H \norm{\Lambda_h}_2  + \frac{\nu H^2 d^{3/2}}{\sqrt{T}}\right)
\end{equation}
We apply \Cref{lem:bound-regret-2} with $W=2H\sqrt{d}$, $B=\nu$ and $\eta_b = \sqrt{\frac{HB^2}{2T}} \cdot \frac{1}{\sqrt{(d + W^2) Hd^2}} = \frac{\nu}{d^{3/2} \sqrt{2(H^2 + 1)}} \frac{1}{\sqrt{T}}$ to obtain the following bound on the term $\reg_2$.
\begin{equation}
    \reg_2 \le O\left( \sqrt{\varepsilon d} H \sum_{h=1}^H \norm{\Lambda_h}_2 + \frac{H^2 \nu d^{3/2}}{\sqrt{T}}\right)
\end{equation}
For the third term, we apply \Cref{lem:mirror-descent} separately for each $h \in [H]$. In particular, we set $q_t^h = \Phi w_{t,h}$, and $D = \norm{q_t^h}_\infty \le W$.
\begin{align*}
    \reg_3 \le \frac{1}{T}\sum_{h=1}^H\frac{\calH(\pi^\star_h \lVert \pi^h_1)}{\alpha} + \frac{H \alpha W^2}{2}
\end{align*}
We now substitute $W = H\sqrt{d}$, $\calH(\pi^\star_h \lVert \pi^h_1) \le \log A$ and $\alpha = \frac{1}{H}\cdot \sqrt{\frac{2\log A}{dT}}$ to obtain the following bound.
\begin{equation}
    \reg_3 \le O\left( H^2 \sqrt{\frac{d \log A}{T}} \right)
\end{equation}
Using the upper bounds on $\reg_1$, $\reg_2$, and $\reg_3$, we obtain the following upper bound on the suboptimality gap.
\begin{align*}
    V^{\pi^\star}(\theta) - \E\left[ V^{\bar{\pi}}(\theta) \right] \le O\left(\nu \sqrt{\varepsilon d} H  \sum_{h=1}^H \norm{\Lambda_h}_2  + \frac{\nu H^2 d^{3/2}}{\sqrt{T}} + H^2 \sqrt{\frac{d \log A}{T}}\right)
\end{align*}
Now we substitute $\norm{\Lambda_h}_2 \le \mathrm{Trace}(\Lambda_h) \le d$, for any $h \in [H]$. Moreover, we must have $K \ge \Theta((d/\varepsilon) \log d)$ and $N \ge K T H$. If we use $T = \sqrt{N}$ then we need $N \ge \tilde{O}\left( \frac{H^2 d^2}{\varepsilon^2}\right)$. This substitution gives us the following upper bound.
\begin{align*}
    V^{\pi^\star}(\theta) - \E\left[ V^{\bar{\pi}}(\theta) \right] \le O\left(\nu \sqrt{\varepsilon } H^2 d^{3/2}   + \frac{\nu H^2 d^{3/2} + H^2 \sqrt{d \log A}}{N^{1/4}} \right)
\end{align*}
If $N \ge \frac{(\nu d + \sqrt{\log A})^4   }{\nu^4 \varepsilon^2}$ then the second term dominates the first term and we get the following bound.
\begin{align*}
    V^{\pi^\star}(\theta) - \E\left[ V^{\bar{\pi}}(\theta) \right] \le O\left(\nu \sqrt{\varepsilon } H^2 d^{3/2}    \right)
\end{align*}
For any $h \in [H]$, the average of the feature distribution at time-step $h$ is $\E_{(s,a) \sim \bar{\pi}_h}\left[ \phi(s,a) \right] = \frac{1}{T} \sum_{t=1}^T  \E_{(s,a) \sim \pi^t_h}\left[ \phi(s,a) \right] = \frac{1}{T} \sum_{t=1}^T \Phi^\top q^{\pi^t}_h = \frac{1}{T} \sum_{t=1}^T \Lambda_h \beta^t_h = \Lambda_h \bar{\beta}_h$. Algorithm~\eqref{alg:offline-primal-dual} performs a robust covariance estimation of $\Lambda_h$, and then multiplies this estimator to $\bar{\beta}_h$ to obtain the average feature distribution. Give any feature $\phi = \phi(s,a)$, let $X$ be the flattened vector $\phi \phi^\top$. Then each entry of the matrix $XX^\top$ can be expressed as $\phi_i \phi_j \phi_k \phi_\ell$ where $1\le i,j,k,\ell \le m$. This means that $\norm{XX^\top}_F^2 = \sum_{i,j,k,\ell} \phi_i^2 \phi_j^2 \phi_k^2 \phi_\ell^2 = \norm{\phi}_2^8 \le 1$, and $\textrm{cov}(X) \le 2\cdot \Identity$. So we can apply \Cref{lem:robust-covariance-estimation} and conclude that $\norm{\hat{\Lambda}_h - \Lambda_h}_2 \le O(\sqrt{\varepsilon})$ for any $h \in [H]$. Therefore, for any $h \in [H]$, $\norm{\hat{v}_h - \Lambda_h \bar{\beta}_h}_2 \le \norm{\hat{\Lambda}_h - \Lambda_h}_2  \norm{\bar{\beta}_h}_2 \le O\left( \sqrt{\varepsilon} \nu \right)$. This bound also implies that $\norm{(\widehat{v}_1,\ldots,\widehat{v}_H) - \left(\Lambda_1 \bar{\beta}_1,\ldots, \Lambda_H \bar{\beta}_H \right) }_2 \le O\left( \nu \sqrt{H \varepsilon}\right)$.

Now recall that we can write $V^{\bar{\pi}}(\theta) = \sum_{h=1}^H \left \langle \Lambda_h \bar{\beta}_h, \theta_h \right \rangle \ge \sum_{h=1}^H \left \langle \widehat{v}_h, \theta_h \right \rangle - O\left( \nu \sqrt{H \varepsilon}\right) \norm{(\theta_1,\ldots,\theta_H)}_2 \ge \sum_{h=1}^H \left \langle \widehat{v}_h, \theta_h \right \rangle - O\left( \nu H \sqrt{d \varepsilon}\right)$. Since $\bar{\pi}$ is an approximate $O\left( \nu \sqrt{\varepsilon} H^2 d^{3/2}\right)$ optimal policy, and $\widehat{v} = (\widehat{v}_1,\ldots, \widehat{v}_H)$ is an approximate $O\left( \nu \sqrt{\varepsilon} H d^{1/2}\right)$ subgradient of $V^{\bar{\pi}}(\theta)$, we can apply lemma~\eqref{lem:apx-subgradient-additive} to conclude that $\widehat{v} = (\widehat{v}_1,\ldots, \widehat{v}_H)$ is also an approximate  $O\left( \nu \sqrt{\varepsilon} H^2 d^{3/2}\right)$ of the optimal value function with respect to the reward parameter $\theta$.
\end{proof}
We now bound the three terms appearing in lemma~\eqref{lem:regret-decomposition}.

\begin{lemma}\label{lem:bound-regret-1}
Assume $\textrm{diam}(\mathcal{B}) \le B$, $\textrm{diam}(\mathcal{W}) \le W$, and $K \ge \Theta\left( (d/\varepsilon) \log d\right)$. Then we have,
$$
\frac{1}{T} \sum_{t=1}^T \sum_{h=1}^H \left\langle w_{t,h} - w^\star_h, \nabla_{w_h} f(\bm{\pi}^t, \bm{\beta}^t, \bm{w}_t) \right \rangle \le O\left( \sqrt{\varepsilon} W B \sum_{h=1}^H \norm{\Lambda_h}_2   + \frac{HW^2}{ \eta_w T} +  \eta_w B^2 \sum_{h=1}^H \norm{\Lambda_h}_2^2 \right)
$$
with constant probability.
\end{lemma}
\begin{proof}
    Let $\bar{g}^t_{w_h} =\frac{1}{K} \sum_{j=1}^K \Phi^\top q_{h,j}^t - \phi(s^2_{h,j},a^2_{h,j}) \cdot \left \langle \beta_h^t, \phi(s^2_{h,j}, a^2_{h,j}) \right \rangle $. From the definition of $q^t_h$ in algorithm~\eqref{alg:offline-primal-dual}, we have for any $h > 1$,
    \begin{align*}
    \E_{\mu^h_\tref} \left[q^t_{h,j}(\tilde{s},\tilde{b})\right] &= 
    \E_{\mu^h_\tref} \left[\pi^t_h(\tilde{b}\mid \tilde{s}) \cdot  \one \set{s'_{h,j} = \tilde{s}} \phi(s_{h,j},a_{h,j})^\top \beta_{h-1}^t \right]\\
    &= \pi^t_h(\tilde{b}\mid \tilde{s}) \cdot \E_{(s,a) \sim \mu^h_\tref} \left[ P_{h-1}(\tilde{s} \mid s,a) \phi(s,a)^\top \beta^t_{h-1}\right]\\
    &= \pi^t_h(\tilde{b}\mid \tilde{s}) \cdot \bm{\mu}_{h-1}(\tilde{s})^\top \E_{(s,a) \sim \mu^{h-1}_\tref} \left[ \phi(s,a) \phi(s,a)^\top \beta^t_{h-1}\right]\\
    &=\pi^t_h(\tilde{b}\mid \tilde{s}) \cdot \bm{\mu}_{h-1}(\tilde{s})^\top \Lambda_{h-1} \beta^t_{h-1} = q^{\pi^t,\beta_t}_h(\tilde{s},\tilde{b})
    \end{align*}
    Additionally $\E_{\mu^h_\tref} \left[q^t_{h,j}(\tilde{s},\tilde{b})\right] = \pi^t_h(\tilde{b}\mid \tilde{s}) \cdot \rho(\tilde{s}) = q^{\pi^t, \beta^t}_h(\tilde{s}, \tilde{b})$. 
    We now bound on the deviation of the estimator $\widetilde{g}^t_{w_h}$ from $\nabla_{w_h} f(\bm{\pi}_t, \bm{\beta}^t, \bm{w}_t)$.
    \begin{align*}
       \E_{\mu^h_\tref, \calD^{t,h}} \left[ \bar{g}^t_{w_h} \right] &= \frac{1}{K} \sum_{j=1}^K \cdot \E_{\mu^h_\tref, \calD^{t,h}} \left[\Phi^\top q_{h,j}^t - \phi(s_{h,j},a_{h,j}) \cdot \left \langle \beta_h^t, \phi(s_{h,j}, a_{h,j}) \right \rangle \right]\\
       &= \Phi^\top q^{\pi^t, \beta^t}_h - \E_{(s,a) \sim \mu^h_\tref}\left[ \phi(s,a) \phi(s,a)^\top \beta^t_h \right]\\
       &= \Phi^\top q^{\pi^t, \beta^t}_h - \Lambda_h \beta^t_h\\
       &= \nabla_{w_h} f(\bm{\pi}_t, \bm{\beta}^t, \bm{w}_t) \quad \textrm{[By \cref{eq:derivative-wrt-w_h}]}
    \end{align*}
    Let $\phi_{h,j} = \phi(s_{h,j}, a_{h,j})$. Then we have,
    \begin{align*}
        &\E_{\mu^h_\tref}\left[ \norm{\Phi^\top q^t_{h,j} - \phi_{h,j} \cdot \left \langle \phi_{h,j}, \beta^t_h\right \rangle  }_2^2\right] \le 2 \E_{\mu^h_\tref}\left[ \norm{\Phi^\top q^t_{h,j} }_2^2 \right] + 2 \E_{\mu^h_\tref} \left[ \norm{\phi_{h,j} \cdot \left \langle \phi_{h,j}, \beta^t_h\right \rangle }_2^2 \right]\\
        &\le 2 + 2  \cdot \E_{\mu^h_\tref} \left[ (\beta^t_h)^\top \phi_{h,j} \phi_{h,j}^\top \beta^t_h \right] = 2 + 2 \cdot \norm{\beta^t_h}^2_{\Lambda_h} \le 2 \cdot \left(1 + B^2 \norm{\Lambda_h}_2^2\right)
    \end{align*}
    The second inequality uses the fact that the norm of the features is bounded by one, and exactly one entry of $q^t_{h,j}$ is set to one. The above bound also implies that $\E_{\mu^h_\tref}\left[ \norm{\bar{g}^t_{w_h}}_2^2\right] \le 2 \cdot \left(1 + B^2 \norm{\Lambda_h}_2^2\right)$.
   Now, observe that $\eps$-fraction of the dataset $\calD_1^{t,h}$ is corrupted, and we apply robust mean to obtain the estimator $\widetilde{g}^t_{w_h}$. Therefore, we can apply \cref{lem:robust-mean-estimation} with $\sigma^2 = 4 \cdot \left(1 + B^2 \norm{\Lambda_h}_2^2\right)$ to obtain the following bound (as long as $K \ge \Theta((d/\varepsilon) \log d)$).
   \begin{equation}
       \label{eq:bound-robust-mean}
       \norm{\widetilde{q}^t_{w_h} - \nabla_{w_h} f(\bm{\pi}^t, \bm{\beta}^t, \bm{w}_t)}_2 \le O\left( \sqrt{\varepsilon}  B \norm{\Lambda_h}_2\right)
   \end{equation}
   The above bound also implies the following upper bound on the $L_2$-norm $\widetilde{g}^t_{w_h}$.
   \begin{align*}
       \norm{\widetilde{g}^t_{w_h}}_2 &\le O\left( \sqrt{\varepsilon}  B \norm{\Lambda_h}_2\right) + \norm{\nabla_{w_h} f(\bm{\pi}^t, \bm{\beta}^t, \bm{w}_t)}_2\\
       &\le O\left( \sqrt{\varepsilon}  B \norm{\Lambda_h}_2\right) + \norm{\Phi^\top q^{\pi^t, \beta^t}_h - \Lambda_h \beta^t_h}_2\\
       &\le O\left( \sqrt{\varepsilon}  B \norm{\Lambda_h}_2\right) + \norm{\Phi^\top q^{\pi^t, \beta^t}_h}_2 + \norm{ \Lambda_h \beta^t_h}_2\\
       &\le O\left( \sqrt{\varepsilon}  B \norm{\Lambda_h}_2\right) + \sum_{s,a} q^{\pi^t, \beta^t}_h(s,a) \norm{\phi(s,a)}_2 + \norm{ \Lambda_h}_2 \norm{ \beta^t_h}_2\\
       &\le O\left( \sqrt{\varepsilon}  B \norm{\Lambda_h}_2\right) + 1 + B\norm{ \Lambda_h}_2 = O\left( B\norm{ \Lambda_h}_2 \right)
   \end{align*}
   The penultimate inequality uses the fact that $q^{\pi^t, \beta^t}_h$ is a probability distribution over the state, action pairs and the feature norms are bounded by one. 

   Let us write $\widetilde{g}^t_{\mathbf{w}} = (\widetilde{g}^t_{w_1}, \ldots, \widetilde{g}^t_{w_H})$. Then $\norm{\widetilde{g}^t_{\mathbf{w}}}_2^2 \le O\left(  B^2 \sum_{h=1}^H \norm{\Lambda_h}_2^2 \right)$. Furthermore, for any $t$ and $h$, $\norm{w^t_h}_2^2 \le W^2 $. Therefore, $\norm{\mathbf{w}^t}_2^2 \le HW^2$. So we can apply lemma~\eqref{lem:online-sgd} to obtain the following bound.
   \begin{align*}
       &\frac{1}{T} \sum_{t=1}^T \sum_{h=1}^H \left\langle w_{t,h} - w^\star_h, \nabla_{w_h} f(\bm{\pi}^t, \bm{\beta}^t, \bm{w}_t) \right \rangle \\
       &\le  \frac{1}{T} \sum_{t=1}^T \sum_{h=1}^H \left\langle w_{t,h} - w^\star_h, \E\left[ \widetilde{g}^t_{w_h}\right] \right \rangle +  \frac{1}{T} \sum_{t=1}^T \sum_{h=1}^H \norm{ w_{t,h} - w^\star_h}_2 \cdot O\left( \sqrt{\varepsilon}B \norm{\Lambda_h}_2 \right)\\
       &\le O\left( \sqrt{\varepsilon} W B \sum_{h=1}^H \norm{\Lambda_h}_2 \right) + \frac{HW^2}{2 \eta_w T} + O\left( \eta_w B^2 \sum_{h=1}^H \norm{\Lambda_h}_2^2 \right)
   \end{align*}
\end{proof}

\begin{lemma}\label{lem:bound-regret-2}
    Assume $\textrm{diam}(\mathcal{B}) \le B$, $\textrm{diam}(\calW) \le W$, and $K \ge \Theta((d/\varepsilon) \log d)$. Then we have,
    $$
    \frac{1}{T} \sum_{t=1}^T \sum_{h=1}^H \left\langle \beta^\star_h - \beta_{t,h}, \nabla_{\beta_h} f(\bm{\pi}^t, \bm{\beta}^t, \bm{w}_t) \right \rangle  \le O\left( \sqrt{\varepsilon} (\sqrt{d} + W)  \sum_{h=1}^H \norm{\Lambda_h}_2 \right) + \frac{HB^2}{2 \eta_b T} + O\left( \eta_b (d+ W^2) \sum_{h=1}^H \norm{\Lambda_h}_2^2 \right)
    $$
    with constant probability.
\end{lemma}
\begin{proof}
    Recall that algorithm~\eqref{alg:offline-primal-dual} defines $v^t_h(s) = \sum_a \pi^t_h(a \mid s) \cdot \phi(s,a)^\top w^t_h$. Let us define the gradient $\bar{g}^t_{\beta_h}$ as follows.
    \begin{align*}
     \bar{g}_{\beta_h}^t =\left\{
    \begin{array}{cc}
       \frac{1}{K} \sum_{j=1}^K  \phi(s_{h,j},a_{h,j}) \left(r_{h,j} + v_{h+1}(s'_{h,j}) - \left \langle w_h^t, \phi(s_h,a_h) \right \rangle \right) & \textrm{ if } h \in [H-1] \\
        \frac{1}{K} \sum_{j=1}^K \phi(s_{h,j},a_{h,j}) \left(r_{h,j}  - \left \langle w_h^t, \phi(s_{h,j},a_{h,j}) \right \rangle  \right)  & \textrm{ if } h=H
    \end{array}\right.
\end{align*}
We will write $\phi_{h,j} = \phi(s_{h,j}, a_{h,j})$. Then for any $h \in [H-1]$ we have,
\begin{align*}
    &\E_{\mu^h_\tref}\left[ \frac{1}{K} \sum_{j=1}^K \phi_{h,j} \left(\theta_h^\top \phi_{h,j} + v^t_{h+1}(s'_{h,j}) - \left \langle w^t_h, \phi_{h,j} \right \rangle \right)\right]\\
    &= \E_{\mu^h_\tref}\left[ \phi_{h,j} \left(\theta_h^\top \phi_{h,j} + v^t_{h+1}(s'_{h,j}) - \left \langle w^t_h, \phi_{h,j} \right \rangle \right) \right]\\
    &= \E_{(s,a) \sim \mu^h_\tref}\left[ \phi(s,a) \phi(s,a)^\top \theta_h \right] + \E_{(s,a) \sim \mu^h_\tref}\left[ \sum_{s'} P_h(s' \mid s,a) v^t_{h+1}(s') \phi(s,a)\right] - \E_{(s,a) \sim \mu^h_\tref}\left[ \phi(s,a) \phi(s,a)^\top w^t_h\right]\\
    &= \Lambda_h (\theta_h - w^t_h) + \sum_{s'}  \Lambda_h \bm{\mu}_h(s') v^{\pi^t, w}_{h+1}(s')\\
    &= \nabla_{\beta_h} f(\bm{\beta}^t, \bm{\pi}^t, \bm{w}^\star_t) \quad \textrm{[By \cref{eq:derivative-wrt-beta_h}]}
\end{align*}
Moreover,
\begin{align*}
    &\E_{\mu^h_\tref}\left[ \norm{\phi_{h,j} \left(\theta_h^\top \phi_{h,j} + v^t_{h+1}(s'_{h,j}) - \left \langle w^t_h, \phi_{h,j} \right \rangle \right)}_2^2 \right]\\
    &\le 2\E_{\mu^h_\tref}\left[ \norm{\phi_{h,j} \left(\theta_h^\top \phi_{h,j} + v^t_{h+1}(s'_{h,j})\right)}_2^2 \right] + 2\E_{\mu^h_\tref}\left[ \norm{\phi_{h,j} \left \langle w^t_h, \phi_{h,j} \right \rangle }_2^2 \right]\\
    &\le 4 \E_{\mu^h_\tref}\left[ \norm{\phi_{h,j} \cdot \theta_h^\top \phi_{h,j}}_2^2 \right] + 4 \E_{\mu^h_\tref}\left[ \norm{\phi_{h,j} \cdot v^t_{h+1}(s'_{h,j})}_2^2 \right] + 2\E_{\mu^h_\tref}\left[ \norm{\phi_{h,j} \left \langle w^t_h, \phi_{h,j} \right \rangle }_2^2 \right]\\
    &\le 4 \E_{\mu^h_\tref}\left[ \norm{\phi(s,a)}_2^2 \theta_h^\top  \phi(s,a) \phi(s,a)^\top \theta_h \right] + 4 \E_{\mu^h_\tref}\left[ \norm{\sum_{s'} P_h(s' \mid s,a) v^t_{h+1}(s') \cdot \phi(s,a)}_2^2\right]\\
    &+ 2\E_{\mu^h_\tref}\left[ (w_h^t)^\top \phi(s,a) \phi(s,a)^\top w_h^t \right]\\
    &\le 4 \norm{\theta_h^t}_{\Lambda_h}^2 + 2 \norm{w^t_h}_{\Lambda_h}^2 + 4\E_{\mu^h_\tref}\left[ (w_h^t)^\top \phi(s,a) \phi(s,a)^\top w_h^t \right]\\
    &\le 4 \norm{\theta_h^t}_{\Lambda_h}^2 + 6 \norm{w^t_h}_{\Lambda_h}^2 \le \left(4d + 6W^2\right) \norm{\Lambda_h}_2^2
\end{align*}
The fourth inequality uses the definition of $v^t_{h+1}$ and $\norm{\phi(s,a)}_2 \le 1$. The final inequality uses $\norm{\theta^t_h}_2 \le \sqrt{d}$ and $\norm{w^t_h}_2 \le W$. The above bound implies that for any $h \in [H]$, $\E_{\mu^h_\tref}\left[ \norm{\bar{g}^t_{\beta_h}}_2^2\right] \le \left(4d + 6W^2\right) \norm{\Lambda_h}_2^2$. Now, observe that $\varepsilon$-fraction of the dataset $\calD^{t,h}_2$
 is corrupted, and we apply robust mean to obtain the estimator $\tilde{g}^t_{\beta_h}$.
 Therefore, we can apply Lemma ~\eqref{lem:robust-mean-estimation} with $\sigma^2 = \left(4d + 6W^2\right) \norm{\Lambda_h}_2^2$ to obtain the following bound (as long as $K \ge \Theta((d/\varepsilon) \log d)$.
 \begin{align}
     \norm{\tilde{g}^t_{\beta_h} - \nabla_{\beta_h} f(\bm{\pi}^t, \bm{\beta}^t, \bm{w}_t)}_2 \le O(\sqrt{\varepsilon (d + W^2} \norm{\Lambda_h}_2)
 \end{align}
 Furthermore, the above bound also implies the following upper bound on the $L_2$-norm of $\tilde{g}^t_{\beta_h}$.
 \begin{align*}
     \norm{\tilde{g}^t_{\beta_h}}_2 &\le O\left(\sqrt{\varepsilon} (\sqrt{d} + W) \norm{\Lambda_h}_2\right) + \norm{ \nabla_{\beta_h} f(\bm{\pi}^t, \bm{\beta}^t, \bm{w}_t)}_2\\
     &\le O\left(\sqrt{\varepsilon} (\sqrt{d} + W) \norm{\Lambda_h}_2\right) + \norm{\Lambda_h(\theta_h - w_h) + \sum_{s'} v^{\pi, w}_{h+1}(s') \Lambda_h \bm{\mu}_h(s')}_2
 \end{align*}
 From the definition of value function we have $v^{\pi,w}_{h+1}(s') \le \abs{\sum_{b'} \pi^t_{h+1}(b' \mid s') \phi(s',b')^\top w^t_{h+1}} \le \sum_{b'} \pi^t_{h+1}(b' \mid s') \norm{\phi(s',b')}_2 \norm{w^t_{h+1}}_2 \le W$ as feature norms are bounded by one. This result gives us the following upper bound.
  \begin{align*}
     \norm{\tilde{g}^t_{\beta_h}}_2 \le O\left(\sqrt{\varepsilon} (\sqrt{d} + W) \norm{\Lambda_h}_2\right) + \norm{\Lambda_h(\theta_h - w_h)}_2 + \norm{\bm{\mu}_h \Lambda_h }_2 \le O\left( (\sqrt{d} + W) \norm{\Lambda_h}_2\right)
    \end{align*}
    Let us now write $\tilde{g}^t_{\bm{\beta}} = (\tilde{g}^t_{w_1}, \ldots, \tilde{g}^t_{w_H})$. Then $\norm{\tilde{g}^t_{\bm{\beta}}}_2^2 \le O\left( (d + W^2)\sum_{h=1}^H \norm{\Lambda_h}_2^2\right)$. Furthermore, for any $t$ and $h$, $\norm{\beta_h}_2 \le B$. Therefore, $\norm{\bm{\beta}}_2^2 \le HB^2$. So we can apply \Cref{lem:online-sgd} to obtain the following bound.
    \begin{align*}
       &\frac{1}{T} \sum_{t=1}^T \sum_{h=1}^H \left\langle \beta^\star_h - \beta_{t,h}, \nabla_{\beta_h} f(\bm{\pi}^t, \bm{\beta}^t, \bm{w}_t) \right \rangle \\
       &\le  \frac{1}{T} \sum_{t=1}^T \sum_{h=1}^H \left\langle \beta^\star_h - \beta_{t,h}, \E\left[ \widetilde{g}^t_{\beta_h}\right] \right \rangle +  \frac{1}{T} \sum_{t=1}^T \sum_{h=1}^H \norm{ \beta_{t,h} - \beta^\star_h}_2 \cdot O\left( \sqrt{\varepsilon}(\sqrt{d} + W) \norm{\Lambda_h}_2 \right)\\
       &\le O\left( \sqrt{\varepsilon} (\sqrt{d} + W)  \sum_{h=1}^H \norm{\Lambda_h}_2 \right) + \frac{HB^2}{2 \eta_b T} + O\left( \eta_b (d+ W^2) \sum_{h=1}^H \norm{\Lambda_h}_2^2 \right)
   \end{align*}
\end{proof}

\begin{lemma}[Online Stochastic Gradient Descent]\label{lem:online-sgd}
    Let $y_1 \in W$, and $\eta > 0$. Define the sequence $y_2, \ldots, y_{n+1}$ and $h_1,\ldots,h_n$ such that for $k=1,\ldots,n$ 
    $$
    y_{k+1} = \textrm{Proj}_W\left( y_k + \eta \widehat{h}_k\right)
    $$
    and $\widehat{h}_k$ satisfies $\E\left[ \widehat{h}_k \mid \calF_{k-1}\right] = h_k$ and $\E\left[ \norm{\widehat{h}_k}_2^2 \mid \calF_{k-1}\right] \le G^2$. Then for any $y^\star \in W$,
    $$
    \E\left[ \sum_{k=1}^n \left\langle y^\star - y_k, h_k \right \rangle  \right] \le \frac{\norm{y_1 - y^\star}_2^2}{2\eta} + \frac{\eta n G^2}{2}.
    $$
\end{lemma}

\begin{lemma}[Mirror Descent, Lemma D.2 of \cite{GNOP23}]\label{lem:mirror-descent}
    Let $q_1,q_2,\ldots,q_T$ be a sequence of functions from $\calS \times \calA \rightarrow \R$ so that $\norm{q_t}_\infty \le D$. Given an initial policy $\pi_1$, and a learning rate $\alpha > 0$, define a sequence of policies 
    $$\pi_{t+1}(a \mid s) \propto \pi_t(a \mid s) e^{\alpha q_t(s,a)}$$
    for $t=1,2,\ldots,T-1$. Then for any comparator policy $\pi^\star$,
    $$
    \frac{1}{T}\sum_{t=1}^T \sum_{s \in \calS} q^{\pi^\star}(s) \left \langle \pi^\star(\cdot \mid s) - \pi_t(\cdot \mid s), q_t(s,\cdot) \right \rangle \le \frac{\calH(\pi^\star \lVert \pi_1)}{T \alpha} + \frac{\alpha D^2}{2}
    $$
\end{lemma}

\begin{lemma}[\cite{DKKL+17}, Theorem 3.2]\label{lem:robust-mean-estimation}
    Let $P$ be a distribution on $\R^d$ with unknown mean vector $\mu$ and unknown covariance matrix $\Sigma \preccurlyeq \sigma^2 \cdot \Identity$. Let $S$ be an $\varepsilon$-corrupted set of samples from $P$ of size $\Theta((d/\varepsilon) \log d)$. There exists an efficient algorithm that, on input $S$ and $\varepsilon > 0$, with probability $9/10$ outputs $\widehat{\mu}$ with $\norm{\widehat{\mu} - \mu}_2 \le O\left(\sqrt{\varepsilon}\sigma\right)$.
\end{lemma}

\begin{lemma}\label{lem:robust-covariance-estimation}
    Let $P$ be a distribution on $\R^d$ with unknown mean vector $\mu$ and unknown covariance matrix $\Sigma$. Suppose $\textrm{cov}_{X\sim P}(XX^\top) \preccurlyeq \sigma^4 \Identity$. Let $S$ be an $\varepsilon$-corrupted set of samples from $P$ of size $\Theta((d^2/\varepsilon^2) \log^2 d)$. There exists an efficient algorithm that, on input $S$ and $\varepsilon > 0$, with probability $9/10$ outputs $\widehat{\mu}$ with $\norm{\widehat{\Sigma} - \Sigma}_2 \le O\left(\sqrt{\varepsilon}\sigma^2\right)$.
\end{lemma}
\begin{proof}
    Apply robust mean estimation on the set of flattened vectors $\set{ xx^\top : x \in S}$. See also \cite{DK19}, subsection 3.2.
\end{proof}

\begin{lemma}[Approximate Subgradient]\label{lem:apx-subgradient-additive}
    Let $f(x) = \max_{i \in [m]} f_i(x)$ where each $f_i$ is closed and convex. Let $j \in [m]$ be a $\beta_1$-approximate optimizer i.e. $f_j(x) \ge f(x) - \beta_1$. If $v$ is a $\beta_2$-approximate subgradient of $f_j$ at $x$, then $v$ is a $(\beta_1 + \beta_2)$-approximate subgradient of $f$ at $x$.
\end{lemma}
\begin{proof}
    Since $v$ is a $\beta_2$-approximate subgradient of $f_j$ at $x$, for any $y$ we have,
    $$
    f(y) = \max_i f_i(y) \ge f_j(y) \ge f_j(x) - \beta_2 + \left \langle v, y - x\right \rangle \ge f(x) - (\beta_1 + \beta_2) + \left \langle v, y - x\right \rangle.
    $$
\end{proof}

\end{document}